\documentclass{article}

\usepackage{amsmath}
\usepackage{amsthm}
\usepackage{amssymb}
\usepackage{graphicx}
\usepackage{subfigure}
\usepackage[numbers]{natbib}
\usepackage{algorithm}
\usepackage{algorithmic}
\usepackage[hidelinks]{hyperref}
\usepackage{authblk}
\usepackage{fullpage}

\DeclareMathOperator*{\argmax}{argmax}

\newtheorem{theorem}{Theorem}
\newtheorem{corollary}[theorem]{Corollary}
\newtheorem{lemma}[theorem]{Lemma}
\newtheorem{definition}[theorem]{Definition}
\newtheorem{remark}[theorem]{Remark}

\renewcommand{\eqref}[1]{Eq.~(\ref{#1})}

\title{Efficient Mixed-Norm Regularization: Algorithms and Safe Screening Methods}
\author[1]{Jie Wang}
\author[1]{Jun Liu}
\author[1]{Jieping Ye}
\affil[1]{Computer Science and Engineering, Arizona State University,
            USA}
\begin{document}

\maketitle

\begin{abstract}
Sparse learning has recently received increasing attention in many
areas including machine learning, statistics, and applied
mathematics. The mixed-norm regularization based on the
$\ell_1/\ell_q$ norm with $q>1$ is attractive in many applications of
regression and classification in that it facilitates group sparsity
in the model. The resulting optimization problem is, however,
challenging to solve due to the inherent structure of the
$\ell_1/\ell_q$-regularization. Existing work deals with special
cases including $q=2, \infty$, and they can not be easily extended to
the general case. In this paper, we propose an efficient algorithm
based on the accelerated gradient method for solving the
$\ell_1/\ell_q$-regularized problem, which is applicable for all
values of $q$ larger than $1$, thus significantly extending existing
work. One key building block of the proposed algorithm is the
$\ell_1/\ell_q$-regularized Euclidean projection (EP$_{1q}$). Our
theoretical analysis reveals the key properties of EP$_{1q}$ and
illustrates why EP$_{1q}$ for the general $q$ is significantly more
challenging to solve than the special cases. Based on our theoretical
analysis, we develop an efficient algorithm for EP$_{1q}$ by solving
two zero finding problems. To further improve the efficiency of solving large dimensional $\ell_1/\ell_q$ regularized problems, we propose an efficient and effective ``{\it screening}" method which is able to quickly identify the inactive groups, i.e., groups that have 0 components in the solution. This may lead to substantial reduction in the number of groups to be entered to the optimization. An appealing feature of our screening method is that the data set needs to be scanned only once to run the screening. Compared to that of solving the $\ell_1/\ell_q$-regularized problems, the computational cost of our screening test is negligible. The key of the proposed screening method is an accurate sensitivity analysis of the dual optimal solution when the regularization parameter varies. Experimental results demonstrate the
efficiency of the proposed algorithm.
\end{abstract}
%\vspace{-0.3in}
\section{Introduction}

Regularization has played a central role in many machine learning
algorithms. The $\ell_1$-regularization has recently received
increasing attention, due to its sparsity-inducing property,
convenient convexity, strong theoretical guarantees, and great
empirical success in various applications. A well-known application
of the $\ell_1$-regularization is the
Lasso~\cite{TIBSHIRANI:1996:ID24}. Recent studies in areas such as
machine learning, statistics, and applied mathematics have witnessed
growing interests in extending the $\ell_1$-regularization to the
$\ell_1/\ell_q$-regularization~\cite{Bach:2008,Duchi:2009:FOLO,Kowalski:2009,Negahban:2009,Obozinski:2007,Sra2012,Vogt2012,Yuan:2006,Zhao:2009}.
This leads to the following $\ell_1/\ell_q$-regularized minimization
problem:
\begin{equation}\label{eq:optimization:problem}
    \min_{\mathbf W \in \mathbb{R}^p}  f(\mathbf W) \equiv l(\mathbf W) + \lambda \varpi(\mathbf W),
\end{equation}
where $\mathbf W \in \mathbb{R}^p$ denotes the model parameters,
$l(\cdot)$ is a convex loss dependent on the training samples and
their corresponding responses, $\mathbf W=[\mathbf w_1^{\rm T},
\mathbf w_2^{\rm T}, \ldots, \mathbf w_s^{\rm T}]^{\rm T}$ is divided
into $s$ non-overlapping groups, $\mathbf w_i \in \mathbb{R}^{p_i},
i=1, 2. \ldots, s$, $\lambda >0$ is the regularization parameter, and
\begin{equation}\label{eq:varpi}
     \varpi(\mathbf W)=\sum_{i=1}^s \|\mathbf w_i\|_q
\end{equation}
is the $\ell_1/\ell_q$ norm with $\|\cdot\|_q$ denoting the vector
$\ell_q$ norm ($q\geq 1$). One of the commonly used loss function is the least square loss, i.e., $l({\bf W})$ takes the form as
\begin{equation}
    l({\bf W}) = \frac{1}{2}\left\|Y - B{\bf W}\right\|_2^2 = \frac{1}{2}\left\|Y - \sum_{i=1}^sB_i{\bf w}_i\right\|_2^2,
\end{equation}
where $B=[B_1,B_2,\ldots,B_s]\in\mathbb{R}^{m\times p}$ is the data matrix with $m$ samples and $p$ features, $B_i\in\mathbb{R}^{m\times p_i}$, $i=1,2,\ldots,s$, is corresponding to the $i^{th}$ group, and $Y\in\mathbb{R}^m$ denotes the response vector.
The $\ell_1/\ell_q$-regularization belongs
to the composite absolute penalties (CAP)~\cite{Zhao:2009} family.
When $q=1$, the problem (\ref{eq:optimization:problem}) reduces to
the $\ell_1$-regularized problem. When $q>1$, the
$\ell_1/\ell_q$-regularization facilitates group sparsity in the
resulting model, which is desirable in many applications of
regression and classification.

The practical challenge in the use of the
$\ell_1/\ell_q$-regularization lies in the development of efficient
algorithms for solving (\ref{eq:optimization:problem}), due to the
non-smoothness of the $\ell_1/\ell_q$-regularization. According to
the black-box Complexity
Theory~\cite{NEMIROVSKI:1994:ID6,Nesterov:2004}, the optimal
first-order black-box method~\cite{NEMIROVSKI:1994:ID6,Nesterov:2004}
for solving the class of nonsmooth convex problems converges as
$O(\frac{1}{\sqrt{k}})$ ($k$ denotes the number of iterations), which
is slow. Existing algorithms focus on solving the problem
(\ref{eq:optimization:problem}) or its equivalent constrained version
for $q=2, \infty$, and they can not be easily extended to the general
case. In order to systematically study the practical performance of
the $\ell_1/\ell_q$-regularization family, it is of great importance
to develop efficient algorithms for solving
(\ref{eq:optimization:problem}) for any $q$ larger than $1$.

\subsection{First-Order Methods Applicable for (\ref{eq:optimization:problem})}

When treating $f(\cdot)$ as the general non-smooth convex function,
we can apply the subgradient
descent (SD)~\cite{Boyd:2003,NEMIROVSKI:1994:ID6,Nesterov:2004}:
\begin{equation}\label{eq:sd:iteration}
    \mathbf X_{i+1}=\mathbf X_i - \gamma_i \mathbf G_i,
\end{equation}
where $\mathbf G_i \in \partial f(\mathbf X_i)$ is a subgradient of
$f(\cdot)$ at $\mathbf X_i$, and $\gamma_i$ a step size. There are
several different types of step size rules, and more details can be
found in~\cite{Boyd:2003,NEMIROVSKI:1994:ID6}. Subgradient descent is
proven to converge, and it can yield a convergence rate of
$O(1/\sqrt{k})$ for $k$ iterations. However, SD has the following two
disadvantages: 1) SD converges slowly; and 2) the iterates of SD are
very rarely at the points of
non-differentiability~\cite{Duchi:2009:FOLO}, thus it might not
achieve the desirable sparse solution (which is usually at the point
of non-differentiability) within a limited number of iterations.

%\vspace{0.05in}

Coordinate Descent (CD)~\cite{Tseng:2001} and its recent
extension---Coordinate Gradient Descent (CGD) can be applied for
optimizing the non-differentiable composite
function~\cite{Tseng:2009:cd}. Coordinate descent has been applied
for the $\ell_1$-norm regularized least
squares~\cite{FRIEDMAN:2008:ID23}, $\ell_1/\ell_{\infty}$-norm
regularized least squares~\cite{Liu:han:2009:blockwise:cd}, and the
sparse group Lasso~\cite{Friedman:2010}. Coordinate gradient descent
has been applied for the group Lasso logistic
regression~\cite{Meier:2008}.  Convergence results for CD and CGD
have been established, when the non-differentiable part is
separable~\cite{Tseng:2001,Tseng:2009:cd}. However, there is no
global convergence rate for CD and CGD (Note, CGD is reported to
have a \emph{local} linear convergence rate under certain
conditions~\cite[Theorem~4]{Tseng:2009:cd}). In addition, it is not
clear whether CD and CGD are applicable to the
problem~\eqref{eq:optimization:problem} with an arbitrary $q \ge 1$.

Fixed Point Continuation~\cite{HALE:2007:ID23,Shi:2008} was recently
proposed for solving the $\ell_1$-norm regularized optimization
(i.e., $\varpi(\mathbf W) =\|\mathbf W\|_1$). It is based on the
following fixed point iteration:
\begin{equation}\label{eq:fixed:point}
    \mathbf X_{i+1}=\mathcal {P}^{\varpi}_{\lambda \tau}( \mathbf X_i - \tau l'(\mathbf
    X_i)),
\end{equation}
where $\mathcal {P}^{\varpi}_{\lambda \tau}(\mathbf
W)=\mbox{sgn}(\mathbf W) \odot \max( \mathbf W - \lambda \tau, 0)$
is an operator and $\tau >0$ is the step size. The fixed point
iteration~\eqref{eq:fixed:point} can be applied to
solve~\eqref{eq:optimization:problem} for any convex penalty
$\varpi(\mathbf W)$, with the operator $\mathcal
{P}^{\varpi}_{\lambda \tau}(\cdot)$ being defined as:
\begin{equation}\label{eq:P:definition}
    \mathcal {P}^{\varpi}_{\lambda \tau}(\mathbf W)=\arg \min_{\mathbf X}
    \frac{1}{2}\|\mathbf X- \mathbf W\|_2^2 + \lambda \tau
    \varphi(\mathbf X).
\end{equation}
The operator $\mathcal {P}^{\varpi}_{\lambda \tau}(\cdot)$ is called
the proximal
operator~\cite{Hiriart-Urruty:1993,Moreau:1965,Yosida:1964}, and is
guaranteed to be non-expansive. With a properly chosen $\tau$, the
fixed point iteration \eqref{eq:fixed:point} can converge to the
fixed point $\mathbf X^*$ satisfying
\begin{equation}\label{eq:fixed:point:optimal}
    \mathbf X^*=\mathcal {P}^{\varpi}_{\lambda \tau}( \mathbf X^* - \tau l'(\mathbf
    X^*)).
\end{equation}
It follows from \eqref{eq:P:definition} and
\eqref{eq:fixed:point:optimal} that,
\begin{equation}\label{eq:fixed:optimality:condition}
    \mathbf 0 \in \mathbf X^* - (\mathbf X^* - \tau l'(\mathbf X^*))
    + \lambda \tau \partial \varpi(\mathbf X^*),
\end{equation}
which together with $\tau >0$ indicates that $\mathbf X^*$ is the
optimal solution to~\eqref{eq:optimization:problem}.
In~\cite{Beck:2009,Nesterov:2007}, the gradient descent method is
extended to optimize the composite function in the form
of~\eqref{eq:optimization:problem}, and the iteration step is similar
to~\eqref{eq:fixed:point}. The extended gradient descent method for nonsmooth objective functions is
proven to yield the convergence rate of $O(1/k)$ for $k$ iterations.
However, as pointed out in~\cite{Beck:2009,Nesterov:2007}, the scheme
in~\eqref{eq:fixed:point} can be further accelerated for
solving~\eqref{eq:optimization:problem}.

Finally, there are various online learning algorithms that have been
developed for dealing with large scale data, e.g., the truncated
gradient method~\cite{Langford:2009:online}, the forward-looking
subgradient~\cite{Duchi:2009:FOLO}, and the regularized dual
averaging~\cite{Xiao:average:2009} (which is based on the dual
averaging method proposed in~\cite{Nesterov:dual:averaging:2009}).
When applying the aforementioned online learning methods for
solving~\eqref{eq:optimization:problem}, a key building block is the
operator $\mathcal {P}^{\varpi}_{\lambda \tau}(\cdot)$.

\subsection{Screening Methods for (\ref{eq:optimization:problem})}

Although many algorithms have been proposed to solve the mixed norm regularized problems, it remains challenging to solve especially for large-scale problems. To address this issue, ``{\it screening}" has been shown to be a promising approach. The key idea of screening is to first identify the ``{\it inactive}" features/groups, which have 0 coefficients in the solution. Then the inactive features/groups can be discarded from the optimization, leading to a reduced feature matrix and substantial savings in computational cost and memory size.

In \cite{Ghaoui2012}, El Ghaoui {\it et al.} proposed novel screening methods, called ``SAFE", to improve the efficiency for solving a class of $\ell_1$ regularized problems, including Lasso, $\ell_1$ regularized logistic regression and $\ell_1$ regularized support vector machines. Inspired by SAFE, Tibshirani {\it et al.} \cite{tibshirani2012} proposed ``strong rules" for a large class of $\ell_1$ regularized problems, including Lasso, $\ell_1$ regularized logistic regression, $\ell_1/\ell_q$ regularized problems and more general convex problems. Although in most of the cases strong rules are more effective in discarding features than SAFE, it is worthwhile to note that strong rules may mistakenly discard features that have non-zero coefficients in the solution. To overcome this limitation, in \cite{Wang2012} the authors proposed the ``DPP" rules for the group Lasso problem \cite{Yuan:2006}, that is, the $\ell_1/\ell_q$-regularization problem in which $q=2$. The DPP rules are safe in the sense that the features/groups discarded from the optimization are guaranteed to have 0 coefficients in the solution.

The core of strong rules for the $\ell_1/\ell_q$ problems is the assumption that the function $\nabla l({\bf X}^*(\lambda))$, that is, the gradient of the loss function $l(\cdot)$ at the optimal solution ${\bf X}^*(\lambda)$ of problem \eqref{eq:optimization:problem}, is a Lipschitz function of $\lambda$. However, this assumption does not always hold in practice \cite{tibshirani2012}. Therefore, groups which have non-zero coefficients can be discarded mistakenly by strong rules. The key idea of DPP rules is to bound the dual optimal solution of problem \eqref{eq:optimization:problem} within a region $\mathcal{R}$ and compute $\max_{\theta\in\mathcal{R}}\|B_i^T\theta\|_{\bar{q}}$, $i=1,2,\ldots,s$, where $1/q+1/{\bar{q}}=1$ (recall that $B_i$ is the data matrix corresponding to ${\bf w}_i$). The smaller the region $\mathcal{R}$ is, the more inactive groups can be detected. In this paper, we give a more accurate estimation of the region $\mathcal{R}$ and extend the idea to the general case with $q\geq1$.

\subsection{Main Contributions}

The main contributions of this paper include the following two parts.
\begin{enumerate}
\item We develop an efficient algorithm for solving the
    $\ell_1/\ell_q$-regularized problem (\ref{eq:optimization:problem}),
    for any $q \geq 1$. %We assume that $l(\cdot)$ is differentiable (smooth)
%and convex (e.g., the least squares loss and the logistic loss).
    More specifically, we develop the GLEP$_{1q}$
    algorithm\footnote{GLEP$_{1q}$ stands for \textbf{G}roup Sparsity
    \textbf{L}earning via the \textbf{$\ell_1/\ell_q$}-regularized
    \textbf{E}uclidean \textbf{P}rojection.}, which makes use of the
    accelerated gradient method~\cite{Beck:2009,Nesterov:2007} for
    minimizing the composite objective functions. GLEP$_{1q}$ has the
    following two favorable properties: (1) It is applicable to any
    smooth convex loss $l(\cdot)$ (e.g., the least squares loss and the logistic loss) and any $q\geq 1$. Existing algorithms
%(e.g., those proposed
%in~\cite{Argyriou:2008,Bach:2008,Berg:2008:report,DUCHI:2009:icml,Liu:han:2009:blockwise:cd,Liu:han:2009:aistats,Liu:han:2009:report,Meier:2008,Negahban:2008:nips,Obozinski:2007,Obozinski:2008,Quattoni:2009,Yuan:2006,Zhao:2009})
    are mainly focused on $\ell_1/\ell_2$-regularization and/or
    $\ell_1/\ell_{\infty}$-regularization. To the best of our knowledge, this is the first work that provides an efficient algorithm for solving (\ref{eq:optimization:problem}) with any $q \geq 1$; and (2) It achieves a global convergence rate of $O(\frac{1}{k^2})$ ($k$ denotes the number of iterations) for the smooth convex loss $l(\cdot)$. In comparison, although the methods proposed in~\cite{Argyriou:2008,DUCHI:2009:icml,Liu:han:2009:blockwise:cd,Obozinski:2007}
    converge, there is no known convergence rate; and the method proposed in~\cite{Meier:2008} has a \emph{local} linear convergence rate under certain conditions~\cite[Theorem~4]{Tseng:2009:cd}. In addition, these methods are not applicable for an arbitrary $q \ge 1$.

    The main technical contribution of the proposed GLEP$_{1q}$ is the development of
    an efficient algorithm for computing the $\ell_1/\ell_q$-regularized
    Euclidean projection (EP$_{1q}$), which is a key building block in
    the GLEP$_{1q}$ algorithm. More specifically, we analyze the
    key theoretical properties of the solution of EP$_{1q}$,  based on
    which we develop an efficient algorithm for EP$_{1q}$ by solving two
    zero finding problems. In addition, our theoretical analysis reveals
    why EP$_{1q}$ for the general $q$ is significantly more challenging
    than the special cases such as $q=2$. We have conducted experimental
    studies to demonstrate the efficiency of the proposed algorithm.
\item We develop novel {\bf s}creening methods for large-scale {\bf mi}xed-{\bf n}orm (Smin) regularized problems. The proposed screening method is able to quickly identify the inactive groups, i.e., groups that have 0 components in the solution. Consequently, the inactive groups can be removed from the optimization problem and the scale of the resulting problem can be significantly reduced. Several appealing features of our screening method includes: (1) It is ``{\it safe}" in the sense that the groups removed from the optimization are guaranteed to have 0 components in the solution. (2) The data set needs to be scanned only once to run the screening. (3) The computational cost of the proposed screening rule is negligible compared to that of solving the $\ell_1/\ell_q$-regularized problems. (4) Our screening method is independent of solvers for the $\ell_1/\ell_q$-regularized problems and thus it can be integrated with any existing solver to improve the efficiency. Due to the difficulty of the $\ell_1/\ell_q$-regularized problems, existing screening methods are limited to \cite{tibshirani2012,Wang2012}. In comparison, the method proposed in \cite{tibshirani2012} is ``{\it inexact}" in the sense that it may mistakenly remove groups from the optimization which have nonzero coefficients in the solution; and the method proposed in \cite{Wang2012} is designed for the case $q=2$.

    The key of the proposed screening method is an accurate estimation of the region $\mathcal{R}$ which includes the dual optimal solution of problem \eqref{eq:optimization:problem} via the ``variational inequalities" \cite{Guler2010}. After the upper bound $\max_{\theta\in\mathcal{R}}\|B_i^T\theta\|_{\bar{q}}$ is computed, we make use of the KKT condition to determine if the $i^{th}$ group has 0 coefficients in the solution and can be removed from the optimization. Experimental results show that the efficiency of the proposed GLEP$_{1q}$ can be improved by ``{\it three orders of magnitude}" with the screening method, especially for the large dimensional data sets.
\end{enumerate}
\subsection{Related Work}

We briefly review recent studies on $\ell_1/\ell_q$-regularization and the corresponding screening methods,
most of which focus on $\ell_1/\ell_2$-regularization and/or
$\ell_1/\ell_{\infty}$-regularization.

%\vspace{0.05in}

$\ell_1/\ell_2$-Regularization:  The group Lasso was proposed
in~\cite{Yuan:2006} to select the groups of variables for prediction
in the least squares regression. In~\cite{Meier:2008}, the idea of
group lasso was extended for classification by the logistic
regression model, and an algorithm via the coordinate gradient
descent~\cite{Tseng:2009:cd} was developed. In~\cite{Obozinski:2007},
the authors considered joint covariate selection for grouped
classification by the logistic loss, and developed a blockwise
boosting Lasso algorithm with the boosted Lasso~\cite{Zhao:2004}.
In~\cite{Argyriou:2008}, the authors proposed to learn the sparse
representations shared across multiple tasks, and designed an
alternating algorithm. The Spectral projected-gradient (Spg)
algorithm was proposed for solving the $\ell_1/\ell_2$-ball
constrained smooth optimization problem~\cite{BergFriedlander:2008},
equipped with an efficient Euclidean projection that has expected
linear runtime. The $\ell_1/\ell_2$-regularized multi-task learning
was proposed in~\cite{Liu:2009:uai}, and the equivalent smooth
reformulations were solved by the Nesterov's
method~\cite{Nesterov:2004}. %Several recent work studied the
%theoretical properties of the group Lasso. The consistency of the
%group Lasso and multiple kernel learning was studied
%in~\cite{Bach:2008}; the behavior of the block
%$\ell_1/\ell_2$-regularization for multivariate regression was
%studied in~\cite{Obozinski:2008};  and the consistency of the group
%Lasso was studied in~\cite{Liu:han:2009:aistats}.

%\vspace{0.05in}

$\ell_1/\ell_{\infty}$-Regularization: A blockwise coordinate descent
algorithm~\cite{Tseng:2001} was developed for the mutli-task
Lasso~\cite{Liu:han:2009:blockwise:cd}. It was applied to the neural
semantic basis discovery problem. In~\cite{Quattoni:2009}, the
authors considered the multi-task learning via the
$\ell_1/\ell_{\infty}$-regularization, and proposed to solve the
equivalent $\ell_1/\ell_{\infty}$-ball constrained problem by the
projected gradient descent. In~\cite{Negahban:2008:nips}, the authors
considered the multivariate regression via the
$\ell_1/\ell_{\infty}$-regularization, showed that the
high-dimensional scaling of $\ell_1/\ell_{\infty}$-regularization is
qualitatively similar to that of ordinary $\ell_1$-regularization,
and revealed that, when the overlap parameter is large enough
($>2/3$), $\ell_1/\ell_{\infty}$-regularization yields the improved
statistical efficiency over $\ell_1$-regularization.

%\vspace{0.05in}

$\ell_1/\ell_q$-Regularization: In~\cite{DUCHI:2009:icml}, the
authors studied the problem of boosting with structural sparsity, and
developed several boosting algorithms for regularization penalties
including $\ell_1$, $\ell_{\infty}$, $\ell_1/\ell_2$, and
$\ell_1/\ell_{\infty}$. In~\cite{Zhao:2009}, the composite absolute
penalties (CAP) family was introduced, and an algorithm called iCAP
was developed. iCAP employed the least squares loss and the
$\ell_1/\ell_{\infty}$ regularization, and was implemented by the
boosted Lasso~\cite{Zhao:2004}. The multivariate regression with the
$\ell_1/\ell_q$-regularization was studied
in~\cite{Liu:han:2009:report}. In~\cite{Negahban:2009}, a unified
framework was provided for establishing consistency and convergence
rates for the regularized $M$-estimators, and the results for
$\ell_1/\ell_q$ regularization was established.

$\ell_1/\ell_q$-Screening: To the best of our knowledge, existing screening methods for $\ell_1/\ell_q$-Regularization are limited to the methods proposed in \cite{tibshirani2012,Wang2012}. The methods proposed in \cite{tibshirani2012}, i.e., the strong rules, assume that the gradient of the loss function at the optimal solution of problem \eqref{eq:optimization:problem} is a Lipschitz function of $\lambda$. However, there are counterexamples showing that the assumption can be violated in practice. As a result, groups which have non-zero coefficients in the solution can be mistakenly discarded from the optimization by strong rules. In \cite{Wang2012}, the authors considered the group lasso problem \cite{Yuan:2006}, that is, the $\ell_1/\ell_2$-regularized problems and proposed the DPP rules. The key idea of DPP rules is to estimate a region $\mathcal{R}$ which includes the dual optimal solution $\theta^*(\lambda)$ of problem \eqref{eq:optimization:problem} by noting that $\theta^*(\lambda)$ is nonexpansive with respect to $\lambda$.

\subsection{Notation}

Throughout this paper, scalars are denoted by italic letters, and
vectors by bold face letters. Let $\mathbf X, \mathbf Y, \ldots$
denote the $p$-dimensional parameters, $\mathbf x_i, \mathbf y_i,
\ldots$ the $p_i$-dimensional parameters of the $i$-th group, and
$x_i$ the $i$-th component of $\mathbf x$. We denote $\bar q=
\frac{q}{q-1}$, and thus $q$ and $\bar q$ satisfy the following
relationship: $\frac{1}{\bar q} +\frac{1}{q}=1$. We use the following
componentwise operators: $\odot$, $.^{q}$, $|\cdot|$ and $ {\rm sgn}(\cdot)$.
Specifically, $\mathbf z=\mathbf x \odot \mathbf y$ denotes $z_i=x_i
y_i$; ${\bf y} = {\bf x}^q$ denotes $y_i=x_i^q$; $\mathbf y=|\mathbf x|$ denotes $y_i=|x_i|$; and $\mathbf y=
{\rm sgn}(\mathbf x)$ denotes $y_i={\rm sgn} (x_i)$, where ${\rm sgn}
(\cdot)$ is the signum function: ${\rm sgn} (t)=1$ if $t>0$; ${\rm
sgn} (t)=0$ if $t=0$; and ${\rm sgn} (t)=-1$ if $t<0$. We use $\langle{\bf x}, {\bf y}\rangle=\sum_i x_i y_i$ to denote the inner product of ${\bf x}$ and ${\bf y}$.

\section{The Proposed GLEP$_{1q}$ Algorithm}
\label{s:agmeep}

%This section is organized as follows. We first review several
%related methods, and then derive the proposed approach.

In this paper, we consider solving~\eqref{eq:optimization:problem} in
the batch learning setting, and propose to apply the accelerated
gradient method~\cite{Beck:2009,Nesterov:2007} due to its fast
convergence rate. We term our proposed algorithm as ``GLEP$_{1q}$'',
which stands for \textbf{G}roup Sparsity \textbf{L}earning via the
\textbf{$\ell_1/\ell_q$}-regularized \textbf{E}uclidean
\textbf{P}rojection. Note that, one can develop the online learning
algorithm for~\eqref{eq:optimization:problem} using the
aforementioned online learning algorithms, where the
$\ell_1/\ell_q$-regularized Euclidean projection is also a key
building block.

We first construct the following model for approximating the
composite function $\mathcal {M}(\cdot)$ at the point $\mathbf X$:
\begin{equation}\label{eq:model:mL}
\begin{aligned}
    \mathcal {M}_{L, \mathbf X} (\mathbf Y)  =  [\mbox{loss}(\mathbf X) + \langle \mbox{loss}'(\mathbf X), \mathbf Y - \mathbf X \rangle] + \lambda \varpi(\mathbf Y)+ \frac{L}{2} \|\mathbf Y-\mathbf X\|_2^2,
\end{aligned}
\end{equation}
where $L >0$. In the model $\mathcal {M}_{L, \mathbf X} (\mathbf
Y)$, we apply the first-order Taylor expansion at the point $\mathbf
X$ (including all terms in the square bracket) for the smooth loss
function $l(\cdot)$, and directly put the non-smooth penalty
$\varpi(\cdot)$ into the model. The regularization term $\frac{L}{2}
\|\mathbf Y-\mathbf X\|_2^2$ prevents $\mathbf Y$ from walking far
away from $\mathbf X$, thus the model can be a good approximation to
$f(\mathbf Y)$ in the neighborhood of $\mathbf X$.

The accelerated gradient method is based on two sequences $\{\mathbf
X_i\}$ and $\{\mathbf S_i \}$ in which $\{\mathbf X_i\}$ is the
sequence of approximate solutions, and $\{\mathbf S_i \}$ is the
sequence of search points. The search point $\mathbf S_i$ is the
affine combination of $\mathbf X_{i-1}$ and $\mathbf X_i$ as
\begin{equation}\label{eq:sk}
    \mathbf S_i= \mathbf X_i + \beta_i (\mathbf X_i - \mathbf
    X_{i-1}),
\end{equation}
where $\beta_i$ is a properly chosen coefficient. The approximate
solution $\mathbf X_{i+1}$ is computed as the minimizer of $\mathcal
{M}_{L_i, \mathbf S_i} (\mathbf Y)$:
\begin{equation}\label{eq:xkplus1}
    \mathbf X_{i+1}= \arg \min_{\mathbf Y} \mathcal {M}_{L_i, \mathbf S_i} (\mathbf   Y),
\end{equation}
where $L_i$ is determined by line search, e.g., the Armijo-Goldstein
rule so that $L_i$ should be appropriate for $\mathbf S_i$.

\begin{algorithm}
  \caption{GLEP$_{1q}$: \textbf{G}roup Sparsity \textbf{L}earning via the
\textbf{$\ell_1/\ell_q$}-regularized \textbf{E}uclidean
\textbf{P}rojection}
  \label{algorithm:GLEP$_{1q}$}
\begin{algorithmic}[1]
  \REQUIRE $\lambda_1 \geq 0, \lambda_2 \geq 0,  L_0 >0, \mathbf X_0, k$
  \ENSURE  $\mathbf X_{k+1}$
    \STATE Initialize $\mathbf X_1=\mathbf X_0$, $\alpha_{-1}=0$, $\alpha_0=1$, and $L=L_0$.
    \FOR{$i=1$ to $k$}
     \STATE Set $\beta_i= \frac{\alpha_{i-2}-1}{\alpha_{i-1}}$, $\mathbf S_i = \mathbf X_i + \beta_i (\mathbf X_i  - \mathbf X_{i-1})$
     \STATE Find the smallest $L=L_{i-1}, 2L_{i-1}, \ldots $ such that $$f( \mathbf X_{i+1}) \leq \mathcal {M}_{L,\mathbf S_i}( \mathbf X_{i+1} ),$$
     where $\mathbf X_{i+1}= \arg \min_{\mathbf Y} \mathcal {M}_{L, \mathbf S_i} (\mathbf  Y)$
    \STATE Set $L_i=L$ and $\alpha_{i+1}=\frac{1+\sqrt{1+4  \alpha_i^2}}{2}$
    \ENDFOR
\end{algorithmic}
\end{algorithm}

The algorithm for solving~\eqref{eq:optimization:problem} is
presented in Algorithm~\ref{algorithm:GLEP$_{1q}$}. GLEP$_{1q}$
inherits the optimal convergence rate of $O(1/k^2)$ from the
accelerated gradient method. In
Algorithm~\ref{algorithm:GLEP$_{1q}$}, a key subroutine
is~\eqref{eq:xkplus1}, which can be computed as $\mathbf
X_{i+1}=\pi_{1q}(\mathbf S_i- l'(\mathbf S_i) /L_i, \lambda / L_i)$,
where $\pi_{1q}(\cdot)$ is the $\ell_1/\ell_q$-regularized Euclidean
projection (EP$_{1q}$) problem:
\begin{equation}\label{eq:projection:1q}
    \pi_{1q}(\mathbf V, \lambda)= \arg \min_{\mathbf X \in \mathbb{R}^p} \frac{1}{2} \|\mathbf X-\mathbf V\|_2^2 +  \lambda \sum_{i=1}^s \|\mathbf
    x_i\|_q.
\end{equation}
The efficient computation of~\eqref{eq:projection:1q} for any $q > 1$
is the main technical contribution of this paper.
%\section{Efficient $\ell_1/\ell_q$-Regularized Projections}
%\label{s:ep}
%
%In this section, we develop an efficient algorithm for solving the
%following $\ell_1/\ell_q$-regularized Euclidean projection
%problem~\eqref{eq:projection:1q}.
Note that the $s$ groups in \eqref{eq:projection:1q} are independent.
Thus the optimization in (\ref{eq:projection:1q}) decouples into a
set of $s$ independent $\ell_q$-regularized Euclidean projection
problems:
\begin{equation}\label{eq:subproblem}
    \pi_q (\mathbf v)= \arg \min_{\mathbf x \in \mathbb{R}^n} \left(g(\mathbf x)= \frac{1}{2} \|\mathbf x -\mathbf v\|_2^2 + \lambda
    \|\mathbf x\|_q\right),
\end{equation}
where $n=p_i$ for the $i$-th group. Next, we study the key properties
of (\ref{eq:subproblem}).

\subsection{Properties of the Optimal Solution to (\ref{eq:subproblem})}

The function $g(\cdot)$ is strictly convex, and thus it has a unique
minimizer, as summarized below:
\begin{lemma}\label{lemma:unique}
The problem (\ref{eq:subproblem}) has a unique minimizer.
\end{lemma}

Next, we show that the optimal solution to (\ref{eq:subproblem}) is
given by zero under a certain condition, as summarized in the
following theorem:
\begin{theorem}\label{theorem:solution}
$\pi_q (\mathbf v) = \mathbf 0$ if and only if $\lambda \geq
\|\mathbf v\|_{\bar q}$.
\end{theorem}

\begin{proof} Let us first compute the directional derivative of
$g(\mathbf x)$ at the point $\mathbf 0$:
\begin{equation*}
    D g(\mathbf 0)[\mathbf u] = \lim_{ \alpha \downarrow 0 }
    \frac{1}{\alpha}[g(\alpha \mathbf u) - g(\mathbf 0)]
                              = -\langle \mathbf v, \mathbf u \rangle +
                              \lambda \|\mathbf u\|_q,
\end{equation*}
where $\mathbf u$ is a given direction. According to the
H\"{o}lder's inequality, we have
\begin{equation*}
    |\langle \mathbf u, \mathbf v \rangle | \leq \|\mathbf u\|_q
    \|\mathbf v\|_{\bar q}, \forall \mathbf u.
\end{equation*}
Therefore, we have
\begin{equation}\label{eq:direction:zero}
    D g(\mathbf 0)[\mathbf u] \geq 0, \forall \mathbf u,
\end{equation}
if and only if $\lambda \geq \|\mathbf v\|_{\bar q}$. The result
follows, since (\ref{eq:direction:zero}) is the necessary and
sufficient condition for $\mathbf 0$ to be the optimal solution of
(\ref{eq:subproblem}).
\end{proof}

Next, we focus on solving (\ref{eq:subproblem}) for $0<\lambda <
\|\mathbf v\|_{\bar q}$. We first consider solving
(\ref{eq:subproblem}) in the case of $1 < q < \infty$, which is the
main technical contribution of this paper. We begin with a lemma
that summarizes the key properties of the optimal solution to the
problem (\ref{eq:subproblem}):

\begin{lemma}\label{lemma:sign}
Let $1<q<\infty$ and $0<\lambda < \|\mathbf v\|_{\bar q} $. Then,
$\mathbf x^*$ is the optimal solution to the problem
(\ref{eq:subproblem}) if and if only it satisfies:
\begin{equation}\label{eq:x:nonzero}
    \mathbf x^* + \lambda \|\mathbf x^*\|_q^{1-q} {\mathbf
    x^*}^{(q-1)} = \mathbf v,
\end{equation}
where $\mathbf y \equiv \mathbf x^{(q-1)}$ is defined
component-wisely as:
%\begin{equation*}
   $ y_i={\rm sgn}(x_i) |x_i|^{q-1}$.
%\end{equation*}
Moreover, we have
\begin{equation}\label{eq:v:absv}
    \pi_q(\mathbf v)= {\rm sgn} (\mathbf v) \odot \pi_q(|\mathbf
    v|),
\end{equation}
\begin{equation}\label{eq:sign:keep}
    {\rm sgn}(\mathbf x^*)={\rm sgn} (\mathbf v),
\end{equation}
\begin{equation}\label{eq:x:bound}
    0 < |x^*_i| < |v_i|, \forall i \in \{i|v_i \neq 0\}.
\end{equation}
\end{lemma}

\begin{proof} Since $\lambda < \|\mathbf v\|_{\bar q}$, it follows
from Theorem~\ref{theorem:solution} that the optimal solution
$\mathbf x^* \neq \mathbf 0$. $\|\mathbf x\|_q$ is differentiable
when $\mathbf x \neq \mathbf 0$, so is $g(\mathbf x)$. Therefore, the
sufficient and necessary condition for $\mathbf x^*$ to be the
solution of (\ref{eq:subproblem}) is $g'(\mathbf x^*)=0$, i.e.,
(\ref{eq:x:nonzero}). Denote $c^* \equiv \lambda \|\mathbf
x^*\|_q^{1-q} >0$. It follows from (\ref{eq:x:nonzero}) that
(\ref{eq:v:absv}) holds, and
\begin{equation}\label{eq:x:comp}
    \mbox{sgn}(x^*_i) \left( |x^*_i| + c^*  |x^*_i|^{q-1} \right)=v_i,
\end{equation}
from which we can verify (\ref{eq:sign:keep}) and (\ref{eq:x:bound}).
\end{proof}

\begin{figure}
  \centering
  \includegraphics[width=2.1in]{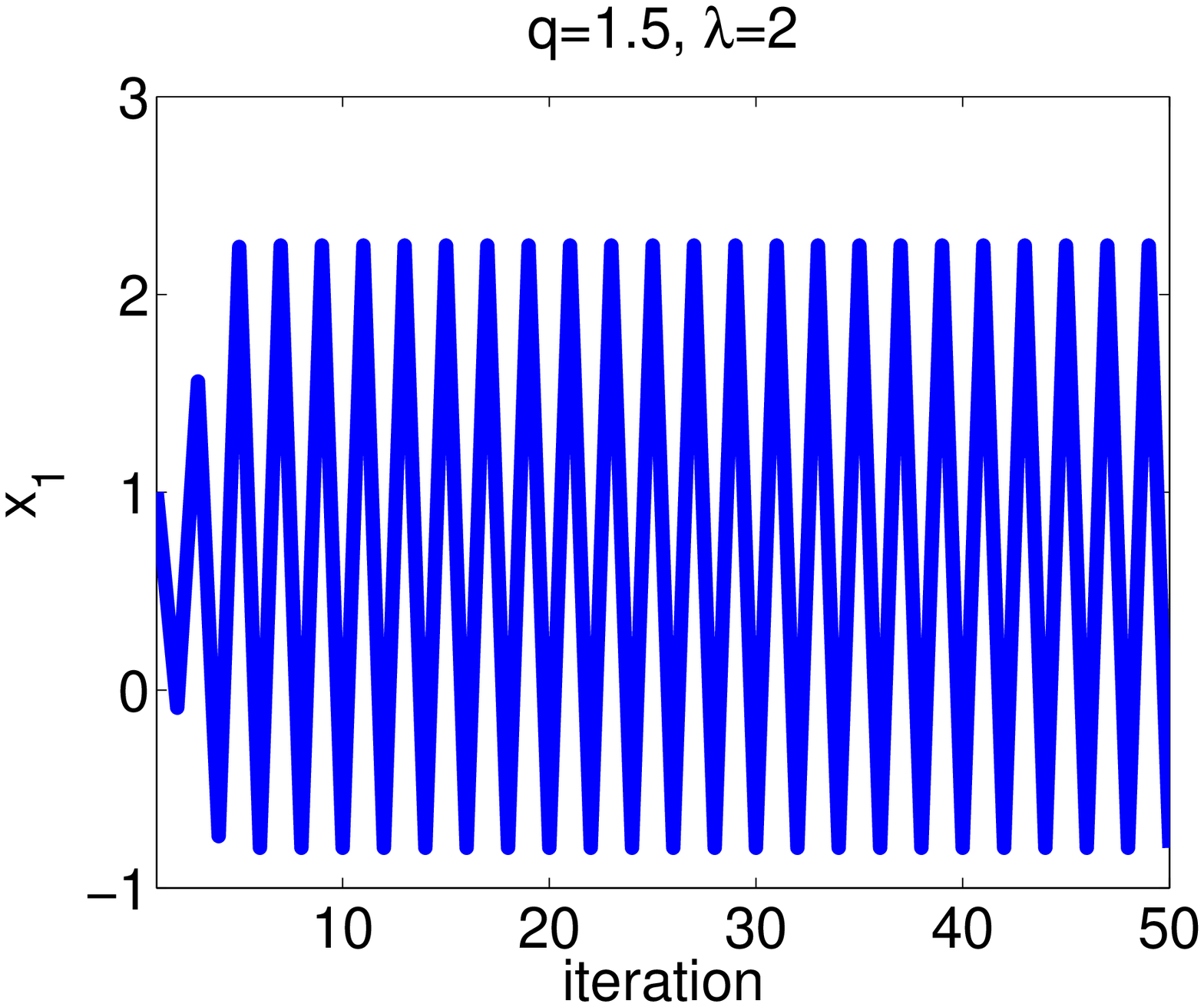} \hspace{0.1in}
  \includegraphics[width=2.1in]{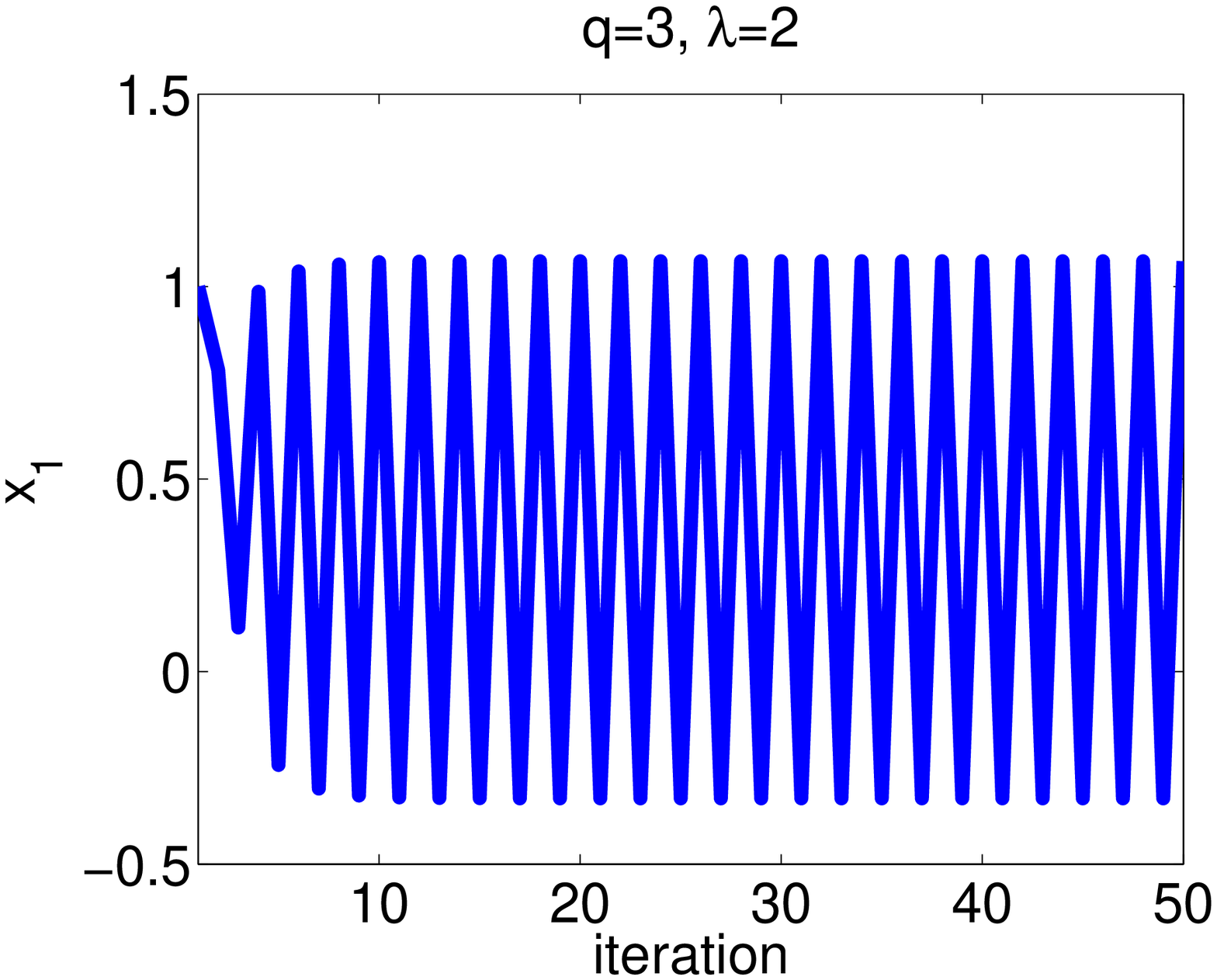}\\[-0.25cm]
  \caption{{\small Illustration of the failure of the fixed point iteration $\mathbf x
= \mathbf v - \lambda \|\mathbf x\|_q^{1-q} {\mathbf x}^{(q-1)}$ for
solving (\ref{eq:subproblem}). We set $\mathbf v=[1, 3]^{\rm T}$ and
the starting point $\mathbf x=[1, 3]^{\rm T}$. The vertical axis
denotes the values of $x_1$ during the
iterations.}}\label{fig:fp:fail}
\end{figure}

%\begin{remark}
It follows from Lemma~\ref{lemma:sign} that i) if $v_i=0$ then
$x_i^*=0$; and ii) $\pi_q(\mathbf v)$ can be easily obtained from
$\pi_q(|\mathbf v|)$. Thus, we can restrict our following discussion
to $\mathbf v> \mathbf 0$, i.e., $v_i>0, \forall i$. It is clear
that, the analysis can be easily extended to the general $\mathbf v$.
The optimality condition in (\ref{eq:x:nonzero}) indicates that
$\mathbf x^*$ might be solved via the fixed point iteration
$$\mathbf x =\eta(\mathbf x) \equiv \mathbf v - \lambda \|\mathbf x\|_q^{1-q}
{\mathbf x}^{(q-1)},$$ which is, however, not guaranteed to converge
(see Figure~\ref{fig:fp:fail} for examples), as $\eta(\cdot)$ is not
necessarily a contraction
mapping~\cite[Proposition~3]{Kowalski:2009}. In addition, $\mathbf
x^*$ cannot be trivially solved by firstly guessing $c=\|\mathbf
x\|_q^{1-q}$ and then finding the root of $\mathbf x + \lambda c
{\mathbf x}^{(q-1)} = \mathbf v$, as when $c$ increases, the values
of $\mathbf x$ obtained from $\mathbf x + \lambda c {\mathbf
x}^{(q-1)} = \mathbf v$ decrease, so that $c=\|\mathbf x\|_q^{1-q}$
increases as well (note, $1-q<0$).

\subsection{Computing the Optimal Solution $\mathbf x^*$ by Zero Finding}

In the following, we show that $\mathbf x^*$ can be obtained by
solving two zero finding problems. Below, we construct our first
auxiliary function $h_c^v(\cdot)$ and reveal its properties:
%\end{remark}

\begin{definition}[Auxiliary Function $h_c^v(\cdot)$ ]\label{def:auxiliary-1}
Let $c>0$, $1 < q < \infty$, and $v >0$. We define the auxiliary
function $h_c^v(\cdot)$ as follows:
\begin{equation}\label{eq:h}
    h_c^v(x)= x + c x^{q-1} - v, 0\leq x \leq v.
\end{equation}
\end{definition}

\begin{lemma}\label{lemma:hf}
Let $c>0$, $1 < q < \infty$, and $v >0$. Then, $h_c^v(\cdot)$ has a
unique root in the interval $(0, v)$.
\end{lemma}

\begin{proof} It is clear that $h_c^v(\cdot)$ is continuous and
strictly increasing in the interval $[0, v]$, $h_c^v(0)=-v<0$, and
$h_c^v(v)= c v^{q-1}
>0$. According to the Intermediate Value Theorem, $h_c^v(\cdot)$ has a
unique root lying in the interval $(0, v)$. This concludes the proof.
\end{proof}

\begin{corollary}\label{corolary:hsf}
Let $\mathbf x, \mathbf v \in \mathbb{R}^n$, $c>0$, $1 < p <
\infty$, and $\mathbf v >\mathbf 0$. Then, the function
\begin{equation}\label{eq:simple:problem}
    \varphi_c^{\mathbf v}(\mathbf x)=\mathbf x + c \mathbf x^{(q-1)}
    -\mathbf v, \mathbf 0 < \mathbf x
    < \mathbf v
\end{equation}
has a unique root.
\end{corollary}

Let $\mathbf x^*$ be the optimal solution satisfying
(\ref{eq:x:nonzero}). Denote $c^* = \lambda \|\mathbf x^*\|_q^{1-q}$.
It follows from Lemma~\ref{lemma:sign} and
Corollary~\ref{corolary:hsf} that $\mathbf x^*$ is the unique root of
$\varphi_{c^*}^{\mathbf v} (\cdot)$ defined in
(\ref{eq:simple:problem}), provided that the optimal $c^*$ is known.
Our methodology for computing $\mathbf x^*$ is to first compute the
optimal $c^*$ and then compute $\mathbf x^*$ by computing the root of
$\varphi_{c^*}^{\mathbf v}(\cdot)$. Next, we show how to compute the
optimal $c^*$ by solving a single variable zero finding problem. We
need our second auxiliary function $\omega(\cdot)$ defined as
follows:

\begin{definition}[Auxiliary Function $\omega(\cdot)$] \label{def:auxiliary-2}
Let $1 < q < \infty $ and $v >0$. We define the auxiliary function
$\omega(\cdot)$ as follows:
\begin{equation}\label{eq:omega}
    c=\omega(x)= (v-x) /  x^{q-1}, 0< x \leq v.
\end{equation}
\end{definition}

\begin{lemma}\label{lemma:omegaf}
In the interval $(0, v]$, $c=\omega(x)$ is i) continuously
differentiable, ii) strictly decreasing, and iii) invertible.
Moreover, in the domain $[0, \infty)$, the inverse function
$x=\omega^{-1}(c)$ is continuously differentiable and strictly
decreasing.
\end{lemma}
\begin{proof} It is easy to verify that, in the interval $(0, v]$, $c=\omega(x)$ is
continuously differentiable with a non-positive gradient, i.e.,
$\omega'(x) < 0$. Therefore, the results follow from the Inverse
Function Theorem.
\end{proof}

It follows from Lemma~\ref{lemma:omegaf} that given the optimal
$c^*$ and $\mathbf v$, the optimal $\mathbf x^*$ can be computed via
the inverse function $\omega^{-1}(\cdot)$, i.e., we can represent
$\mathbf x^*$ as a function of $c^*$. Since $\lambda \|\mathbf
x^*\|_q^{1-q}-c^* =0$ by the definition of $c^*$, the optimal $c^*$
is a root of our third auxiliary function $\phi(\cdot)$ defined as
follows:

\begin{definition}
[Auxiliary Function $\phi(\cdot)$] \label{def:auxiliary-3} Let $1 < q
< \infty$, $0<\lambda < \|\mathbf v\|_{\bar q}$, and $\mathbf v
>\mathbf 0$. We define the auxiliary function $\phi(\cdot)$ as follows:
\begin{equation}\label{eq:phi:c}
    \phi(c)=\lambda \psi(c)-c, c \geq 0,
\end{equation}
where
\begin{equation}\label{eq:psi:c}
    \psi(c)=\left(\sum_{i=1}^n
    (\omega_i^{-1}(c))^q\right)^{\frac{1-q}{q}},
\end{equation}
and $\omega_i^{-1}(c)$ is the inverse function of
\begin{equation}\label{eq:c:x}
    \omega_i(x)=(v_i-x) / x^{q-1}, 0< x \leq v_i.
\end{equation}
\end{definition}

Recall that we assume $0<\lambda < \|\mathbf v\|_{\bar q}$
(otherwise the optimal solution is given by zero from
Theorem~\ref{theorem:solution}). The following lemma summarizes the
key properties of the auxiliary function $\phi(\cdot)$:

\begin{lemma}\label{lemma:lessThanZero}
Let $1 < q < \infty$, $0<\lambda < \|\mathbf v\|_{\bar q}$, $\mathbf
v >\mathbf 0$, and
\begin{equation}\label{eq:epsilon}
    \epsilon=(\|\mathbf v\|_{\bar q} - \lambda) / \|\mathbf v\|_{\bar q}.
\end{equation}
Then, $\phi(\cdot)$ is continuously differentiable in the interval
$[0, \infty)$. Moreover, we have $$\phi(0)=\lambda \| \mathbf
v\|_q^{1-q}
>0,\phi(\overline c) \leq 0,$$ where
\begin{equation}\label{eq:c:tilde}
    \overline c= \max_i c_i,
\end{equation}
\begin{equation}\label{eq:ci:vi}
    c_i=\omega_i(v_i \epsilon), i=1, 2, \ldots, n.
\end{equation}
\end{lemma}

 \begin{proof} From Lemma~\ref{lemma:omegaf}, the function
$\omega_i^{-1}(c)$ is continuously differentiable in $[0, \infty)$.
It is easy to verify that $\omega_i^{-1}(c) >0, \forall c \in [0,
\infty)$. Thus, $\phi(\cdot)$ in~(\ref{eq:phi:c}) is continuously
differentiable in $[0, \infty)$.

It is clear that $\phi(0)=\lambda \|\mathbf v\|_q^{1-q}
>0$. Next, we show $\phi(\overline c) \leq 0$. Since $0<\lambda <
\|\mathbf v\|_{\bar q}$, we have
\begin{equation}\label{epsilon:range}
    0 <\epsilon <1.
\end{equation}
It follows from (\ref{eq:c:x}), (\ref{eq:c:tilde}), (\ref{eq:ci:vi})
and (\ref{epsilon:range}) that $0 < c_i \leq \overline c, \forall
i$. Let $\mathbf x=[x_1, x_2, \ldots, x_n]^{\rm T}$ be the root of
$\varphi_{\overline c}^{\mathbf v} (\cdot)$ (see
Corollary~\ref{corolary:hsf}). Then, $x_i=\omega^{-1}_i(\overline
c)$. Since $\omega^{-1}_i(\cdot)$ is strictly decreasing (see
Lemma~\ref{lemma:omegaf}), $c_i \leq \overline c$, $v_i
\epsilon=\omega_i^{-1}(c_i)$, and $x_i=\omega_i^{-1}(\overline c)$,
we have
\begin{equation}\label{eq:x:upper:bound}
    x_i \leq v_i \epsilon.
\end{equation}
Combining (\ref{eq:c:x}), (\ref{eq:x:upper:bound}), and $\overline c=
\omega_i(x_i)$, we have $ \overline c \geq
v_i(1-\epsilon)/x_i^{q-1}$, since $\omega_i(\cdot)$ is strictly
decreasing. It follows that $x_i \geq
\left(\frac{v_i(1-\epsilon)}{\overline c}\right)^{\frac{1}{q-1}}$.
%\begin{equation*}\label{eq:x:lower:bound}
%    x_i \geq \left(\frac{v_i(1-\epsilon)}{\overline c}\right)^{\frac{1}{q-1}}.
%\end{equation*}
Thus, the following holds:
\begin{equation*}
    \psi(\overline c)=  \left(\sum_{i=1}^n (\omega_i^{-1}(\overline c))^q
    \right)^{\frac{1-q}{q}} = \left(\sum_{i=1}^n x_i^q \right)^{\frac{1-q}{q}}
    \leq \frac{\overline c}{\|\mathbf v\|_{\bar q} (1-\epsilon)},
\end{equation*}
which leads to
\begin{equation*}\label{eq:phi:tilde:c}
    \phi(\overline c) =\lambda \psi(\overline c) -\overline c \leq \overline c
    \left(\frac{\lambda}{\|\mathbf v\|_{\bar q} (1-\epsilon)}-1 \right)=0,
\end{equation*}
where the last equality follows from (\ref{eq:epsilon}).
\end{proof}

\begin{corollary}\label{corolary:lower}
Let $1 < q < \infty$, $0<\lambda < \|\mathbf v\|_{\bar q}$, $\mathbf
v >\mathbf 0$, and $ \underline c= \min_i c_i$, where $c_i$'s are
defined in (\ref{eq:ci:vi}). We have $0<\underline c \leq \overline
c$ and $\phi(\underline c) \geq 0$.
\end{corollary}

Following Lemma~\ref{lemma:lessThanZero} and
Corollary~\ref{corolary:lower}, we can find at least one root of
$\phi(\cdot)$ in the interval $[\underline c, \overline c]$. In the
following theorem, we show that $\phi(\cdot)$ has a unique root:

\begin{theorem}\label{theorem:singleZF}
Let $1 < q < \infty$, $0<\lambda < \|\mathbf v\|_{\bar q}$, and
$\mathbf v >\mathbf 0$. Then, in $[\underline c, \overline c]$,
$\phi(\cdot)$ has a unique root, denoted by $c^*$, and the root of
$\varphi_{c^*}^{\mathbf v}(\cdot)$ is the optimal solution to
(\ref{eq:subproblem}).
\end{theorem}

 \begin{proof} From Lemma~\ref{lemma:lessThanZero} and
Corollary~\ref{corolary:lower}, we have $\phi(\overline c) \leq 0$
and $\phi(\underline c) \geq 0$. If either $\phi(\overline c) = 0$
or $\phi(\underline c) = 0$, $\overline c$ or $\underline c$ is a
root of $\phi(\cdot)$. Otherwise, we have $\phi(\underline c)
\phi(\overline c) <0$. As $\phi(\cdot)$ is continuous in $[0,
\infty)$, we conclude that $\phi(\cdot)$ has a root in $(\underline
c, \overline c)$ according to the Intermediate Value Theorem.

Next, we show that $\phi(\cdot)$ has a unique root in the interval
$[0, \infty)$. We prove this by contradiction. Assume that
$\phi(\cdot)$ has two roots: $0< c_1 <c_2$. From
Corollary~\ref{corolary:hsf}, $\varphi_{c_1}^{\mathbf v}(\cdot)$ and
$\varphi_{c_2}^{\mathbf v}(\cdot)$ have unique roots. Denote $\mathbf
x^1=[x_1^1, x_2^1, \ldots, x_n^1]^{\rm T}$ and $\mathbf x^2=[x_1^2,
x_2^2, \ldots, x_n^2]^{\rm T}$ as the roots of
$\varphi_{c_1}^{\mathbf v}(\cdot)$ and $\varphi_{c_2}^{\mathbf
v}(\cdot)$, respectively. We have $0< x_i^1, x_i^2 < v_i, \forall i$.
It follows from (\ref{eq:phi:c}-\ref{eq:c:x}) that
\begin{eqnarray}
   \mathbf x^1 + \lambda \|\mathbf x^1\|_q^{1-q} {\mathbf x^1}^{(q-1)} -
   \mathbf
   v = \mathbf 0, \nonumber \\
   \mathbf x^2 + \lambda \|\mathbf x^2\|_q^{1-q} {\mathbf x^2}^{(q-1)} -
   \mathbf
   v = \mathbf 0. \nonumber
\end{eqnarray}
According to Lemma~\ref{lemma:sign}, $\mathbf x^1$ and $\mathbf x^2$
are the optimal solution of (\ref{eq:subproblem}). From
Lemma~\ref{lemma:unique}, we have $\mathbf x^1=\mathbf x^2$.
However, since $x_i^1=\omega_i^{-1}(c_1)$,
$x_i^2=\omega_i^{-1}(c_2)$, $\omega_i^{-1}(\cdot)$ is a strictly
decreasing function in $[0, \infty)$ by Lemma~\ref{lemma:omegaf},
and $c_1 < c_2$, we have $x_i^1 > x_i^2, \forall i$. This leads to a
contradiction. Therefore, we conclude that $\phi(\cdot)$ has a
unique root in $[\underline c, \overline c]$.

From the above arguments, it is clear that, the root of
$\varphi_{c^*}^{\mathbf v}(\cdot)$ is the optimal solution to
(\ref{eq:subproblem}).
\end{proof}

\begin{remark} \label{remark:q=2}
When $q=2$, we have $\underline c=\overline c
=\frac{\lambda}{\|\mathbf v\|_2 -\lambda}$. It is easy to verify
that $\phi(\underline c)=\phi(\overline c)=0$ and
\begin{equation}\label{eq:p:2:solution}
    \pi_2(\mathbf v)=\frac{\|\mathbf v\|_2 - \lambda}{\|\mathbf v\|_2} \mathbf
    v.
\end{equation}
Therefore, when $q=2$, we obtain a closed-form solution.
\end{remark}

%\noindent{\bf Bisection for solving the zero finding problem: }

\subsection{Solving the Zero Finding Problem by Bisection}

Let $1 < q < \infty$, $0<\lambda < \|\mathbf v\|_{\bar q}$, $\mathbf
v
>\mathbf 0$, $\overline v= \max_i v_i$, $\underline v =\min_i v_i$,
and $\delta
>0$ be a small constant (e.g., $\delta=10^{-8}$ in our experiments).
When $q>2$, we have
$$\underline c =
\frac{1-\epsilon}{\epsilon^{q-1}\overline v^{q-2}} \quad \mbox{  and
} \quad \overline c = \frac{1-\epsilon}{\epsilon^{q-1}\underline
v^{q-2}}.$$ When $1 < q < 2$, we have
$$\underline c =
\frac{1-\epsilon}{\epsilon^{q-1}\underline v^{q-2}} \quad \mbox{ and
} \quad \overline c = \frac{1-\epsilon}{\epsilon^{q-1}\overline
v^{q-2}}.$$ If either $\phi(\overline c)=0$ or $\phi(\underline
c)=0$, $\overline c$ or $\underline c$ is the unique root of
$\phi(\cdot)$. Otherwise, we can find the unique root of
$\phi(\cdot)$ by bisection in the interval $(\underline c, \overline
c)$, which costs at most
$$N=\log_2 \frac{(1-\epsilon)|\overline v^{q-2}- \underline
v^{q-2}|}{\epsilon^{q-1}\overline v^{q-2} \underline v^{q-2}\delta}$$
iterations for achieving an accuracy of $\delta$. Let $[c_1, c_2]$ be
the current interval of uncertainty, and we have computed
$\omega_i^{-1}(c_1)$ and $\omega_i^{-1}(c_2)$ in the previous
bisection iterations. Setting $c=\frac{c_1+c_2}{2}$, we need to
evaluate $\phi(c)$ by computing $\omega_i^{-1}(c), i=1, 2, \ldots,
n$. It is easy to verify that $\omega_i^{-1}(c)$ is the root of
$h_c^{v_i}(\cdot)$ in the interval $(0, v_i)$. Since
$\omega_i^{-1}(\cdot)$ is a strictly decreasing function (see
Lemma~\ref{lemma:omegaf}), the following holds:
$$\omega_i^{-1}(c_2 )
< \omega_i^{-1}(c) < \omega_i^{-1}(c_1),$$ and thus
$\omega_i^{-1}(c)$ can be solved by bisection using at most
$$\log_2 \frac{\omega_i^{-1}(c_2)-\omega_i^{-1}(c_1)}{\delta} < \log_2
\frac{v_i}{\delta} \leq \log_2 \frac{\overline v}{\delta}$$ iterations
for achieving an accuracy of $\delta$. For given $\mathbf v,
\lambda$, and $\delta$, $N$ and $\overline v$ are constant, and thus
it costs $O(n)$ for finding the root of $\phi(\cdot)$. Once $c^*$,
the root of $\phi(\cdot)$ is found, it costs $O(n)$ flops to compute
$\mathbf x^*$ as the unique root of $\varphi_{c^*}^{\mathbf
v}(\cdot)$. Therefore, the overall time complexity for solving
(\ref{eq:subproblem}) is $O(n)$.

We have shown how to solve (\ref{eq:subproblem}) for $1 < q <
\infty$. For $q=1$, the problem (\ref{eq:subproblem}) is reduced to
the one used in the standard Lasso, and it has the following
closed-form solution~\cite{Beck:2009}:
\begin{equation}\label{eq:x:solution}
    \pi_1(\mathbf v)={\rm sgn}(\mathbf v) \odot \max(|\mathbf
    v|-\lambda,0).
\end{equation}
For $q=\infty$,  the problem (\ref{eq:subproblem}) can computed
via~\eqref{eq:x:solution}, as summarized in the following theorem:
\begin{theorem}\label{theorem:oneAndInf}
Let $q=\infty$, $\bar q=1$, and $0<\lambda < \|\mathbf v\|_{\bar
q}$. Then we have
\begin{equation}\label{eq:x:infty}
    \pi_{\infty}(\mathbf v)={\rm sgn}(\mathbf v) \odot \min( |\mathbf
    v|, t^* ),
\end{equation}
where $t^*$ is the unique root of
\begin{equation}\label{eq:zero:finding:l1}
    h(t)=\sum_{i=1}^n \max( |v_i| - t, 0) - \lambda.
\end{equation}
\end{theorem}

 \begin{proof} Making use of the property that $\|\mathbf
x\|_{\infty} = \max_{\|\mathbf y\|_1 \leq 1} \langle \mathbf y,
\mathbf x  \rangle$,
%\begin{equation*}\label{eq:infty}
%    \|\mathbf x\|_{\infty} = \max_{\|\mathbf y\|_1 \leq 1} \langle \mathbf y, \mathbf x  \rangle,
%\end{equation*}
we can rewrite (\ref{eq:subproblem}) in the case of $q=\infty$ as
\begin{equation}\label{eq:minmax}
    \min_{\mathbf x} \max_{\mathbf y: \|\mathbf y\|_1 \leq \lambda } s(\mathbf x, \mathbf y) \equiv \frac{1}{2} \|\mathbf x -\mathbf v\|_2^2 + \langle \mathbf y, \mathbf x
    \rangle.
\end{equation}
The function $s(\mathbf x, \mathbf y)$ is continuously
differentiable in both $\mathbf x$ and $\mathbf y$, convex in
$\mathbf x$ and concave in $\mathbf y$, and the feasible domains are
solids. According to the well-known von Neumann
Lemma~\cite{NEMIROVSKI:1994:ID6}, the min-max problem
(\ref{eq:minmax}) has a saddle point, and thus the minimization and
maximization can be exchanged. Setting the derivative of $s(\mathbf
x, \mathbf y)$ with respect to $\mathbf x$ to zero, we have
\begin{equation}\label{eq:x:y}
    \mathbf x = \mathbf v -\mathbf y.
\end{equation}
Thus we obtain the following problem:
\begin{equation}\label{eq:l1:proj}
    \min_{\mathbf y: \|\mathbf y\|_1 \leq \lambda} \frac{1}{2} \|\mathbf y -\mathbf
    v\|_2^2,
\end{equation}
which is the problem of the Euclidean projection onto the $\ell_1$
ball~\cite{BergFriedlander:2008,DUCHI:2009:icml,Liu:2009:icml}. It
has been shown that the optimal solution $\mathbf y^*$to
(\ref{eq:l1:proj}) for $\lambda < \|\mathbf v\|_{1}$ can be obtained
by first computing $t^*$ as the unique root of
(\ref{eq:zero:finding:l1}) in linear time, and then computing
$\mathbf y^*$ as
\begin{equation}\label{eq:y:solution}
    \mathbf y^*={\rm sgn}(\mathbf v) \odot \max( |\mathbf v| - t^*, 0).
\end{equation}
It follows from (\ref{eq:x:y}) and (\ref{eq:y:solution}) that
(\ref{eq:x:infty}) holds.
\end{proof}

We conclude this section by summarizing the results for solving the
$\ell_q$-regularized Euclidean projection in
Algorithm~\ref{algorithm:Eq}.

\begin{algorithm}
  \caption{Ep$_{q}$: $\ell_q$-regularized Euclidean projection}
  \label{algorithm:Eq}
\begin{algorithmic}[1]
  \REQUIRE $\lambda>0, q \ge 1, \mathbf v \in \mathbb{R}^n$
  \ENSURE  $\mathbf x^*= \pi_q(\mathbf v)=\arg \min_{\mathbf x \in \mathbb{R}^n}  \frac{1}{2} \|\mathbf x -\mathbf v\|_2^2 + \lambda
    \|\mathbf x\|_q$
    \STATE Compute $\bar q= \frac{q}{q-1}$
    \IF{$\|\mathbf v\|_{\bar q} \le \lambda $}
    \STATE Set $\mathbf x^*=\mathbf 0$, return
    \ENDIF
    \IF{$q=1$}
    \STATE Set $\mathbf x^*={\rm sgn}(\mathbf v) \odot \max(|\mathbf
    v|-\lambda,0)$
    \ELSIF{$q=2$}
    \STATE Set $\mathbf x^*=\frac{\|\mathbf v\|_2 - \lambda}{\|\mathbf v\|_2} \mathbf  v$
    \ELSIF{$q=\infty$}
          \STATE Obtain $t^*$, the unique root of $h(t)$, via the improved bisection method~\cite{Liu:2009:icml}
          \STATE Set $\mathbf x^*={\rm sgn}(\mathbf v) \odot \min( |\mathbf  v|, t^* )$
    \ELSE
          \STATE Compute $c^*$, the unique root of $\phi(c)$, via
                 bisection in the interval $[\underline c, \overline c]$
          \STATE Obtain $\mathbf x^*$ as the unique root of $\varphi_{c^*}^{\mathbf v}(\cdot)$
    \ENDIF
\end{algorithmic}
\end{algorithm}

\section{The Proposed Screening Method (Smin) for the $\ell_1/\ell_q$-regularized Problems}

In this section, we assume the loss function $\ell(\cdot)$ is the least square loss, i.e., we consider the following problem:
\begin{equation}\label{prob:primal}
\min_{{\bf W}\in\mathbb{R}^p}f({\bf W})=\frac{1}{2}\left\|Y-\sum_{i=1}^sB_i{\bf w}_i\right\|_2^2+\lambda\sum_{i=1}^s\|{\bf w}_i\|_q.
\end{equation}
The dual problem of (\ref{prob:primal}) takes the form as:
\begin{align}\label{prob:dual0}
\max_{\theta}\,\,& \frac{1}{2}\|Y\|_2^2-\frac{\lambda^2}{2}\left\|\theta-\frac{Y}{\lambda}\right\|_2^2,\\ \nonumber
\mbox{s.t.}\,\,&\|B_i^{\rm T}\theta\|_{\bar{q}}\leq1, i = 1,2,\ldots,s.
\end{align}
Let ${\bf X}^*(\lambda)$ and $\theta^*(\lambda)$ be the optimal solutions of problems (\ref{prob:primal}) and (\ref{prob:dual0}) respectively. The KKT conditions read as:
\begin{equation}\label{eqn:KKT1}
Y = \sum_{i=1}^sB_i{\bf X}_i^*(\lambda)+\lambda\theta^*(\lambda),
\end{equation}
\begin{equation}\label{eqn:KKT2}
B_i^{\rm T}\theta^*(\lambda)\in
\begin{cases}
\frac{{\bf X}^*_i(\lambda)}{\|{\bf X}^*_i(\lambda)\|_{\bar{q}}},\hspace{28mm}\mbox{if }{\bf X}_i^*(\lambda)\neq0,\\
{\bf U}_i,\,\,{\bf U}_i\in\mathbb{R}^{p_i}, \|{\bf U}_i\|_{\bar{q}}\leq1,\hspace{3mm}\mbox{if }{\bf X}_i^*(\lambda)=0,
\end{cases}
i=1,2,\ldots,s.
\end{equation}
In view of \eqref{eqn:KKT2}, we can see that
\begin{equation}\tag{R}\label{rule0}
\|B_i^{\rm T}\theta^*(\lambda)\|_{\bar{q}}<1\Rightarrow {\bf X}_i^*(\lambda)=0.
\end{equation}
In other words, if $\|B_i^{\rm T}\theta^*(\lambda)\|_{\bar{q}}<1$, then the KKT conditions imply that the coefficients of $B_i$ in the solution ${\bf X}^*(\lambda)$, that is, ${\bf X}^*_i(\lambda)$, are 0 and thus the $i^{th}$ group can be safely removed from the optimization of problem (\ref{prob:primal}). However, since $\theta^*(\lambda)$ is in general unknown, (\ref{rule0}) is not very helpful to discard inactive groups.
To this end, we will estimate a region $\mathcal{R}$ which contains $\theta^*(\lambda)$. Therefore, if $\max_{\theta\in\mathcal{R}}\|B_i^{\rm T}\theta\|_{\bar{q}}<1$, we can also conclude that ${\bf X}^*_i(\lambda)=0$ by (\ref{rule0}). As a result, the rule in (\ref{rule0}) can be relaxed as
\begin{equation}\tag{R$'$}\label{rule}
\varphi(\theta^*(\lambda),B_i):=\max_{\theta\in\mathcal{R}}\|B_i^{\rm T}\theta\|_{\bar{q}}<1\Rightarrow {\bf X}^*_i(\lambda)=0.
\end{equation}
In this paper, (\ref{rule}) serves as the cornerstone for constructing the proposed screening rules. From (\ref{rule}), we can see that screening rules with smaller $\varphi(\theta^*(\lambda),B_i)$ are more effective in  identifying inactive groups. To give a tight estimation of $\varphi(\theta^*(\lambda),B_i)$, we need to restrict the region $\mathcal{R}$ containing $\theta^*(\lambda)$ as small as possible.

In Section \ref{subsection:region}, we give an accurate estimation of the possible region of $\theta^*(\lambda)$ via the variational inequality. We then derive the upper bound $\varphi(\theta^*(\lambda),B_i)$ in Section \ref{subsection:bound} and construct the proposed screening rules, that is, Smin, in Section {\ref{subsection:rules}} based on (\ref{rule}).

\subsection{Estimating the Possible Region for $\theta^*(\lambda)$}\label{subsection:region}
In this section, we briefly discuss the geometric properties of the optimal solution $\theta^*(\lambda)$ of problem (\ref{prob:dual0}), and then give an accurate estimation of the possible region of $\theta^*(\lambda)$ via the variational inequality.

Consider problem
\begin{align}\label{prob:dual}
\min_{\theta}\,\,& g(\theta):=\frac{1}{2}\left\|\theta-\frac{Y}{\lambda}\right\|_2^2,\\ \nonumber
\mbox{s.t.}\,\,&\|B_i^{\rm T}\theta\|_{\bar{q}}\leq1, i = 1,2,\ldots,s.
\end{align}
It is easy to see that problems (\ref{prob:dual0}) and (\ref{prob:dual}) have the same optimal solution. Let $\mathcal{F}=\{\theta:\|B_i^{\rm T}\theta\|_{\bar{q}}\leq1, i = 1,2,\ldots,s.\}$ denote the feasible set of problem (\ref{prob:dual}). We can see that the optimal solution $\theta^*(\lambda)$ of problem (\ref{prob:dual}) is the projection of $\frac{Y}{\lambda}$ onto the feasible set $\mathcal{F}$. The following theorem shows that problem (\ref{prob:primal}) and its dual (\ref{prob:dual0}) admit closed form solutions when $\lambda$ is large enough.

\begin{theorem}\label{thm:thetamx}
Let ${\bf X}^*(\lambda)$ and $\theta^*(\lambda)$ be the optimal solutions of problem (\ref{prob:primal}) and its dual (\ref{prob:dual0}). Then if $\lambda\geq\lambda_{max}:=\max_i \|B_i^{\rm T}Y\|_{\bar{q}}$, we have
\begin{equation}
{\bf X}^*(\lambda)=0\hspace{2mm}\mbox{and }\,\,\theta^*(\lambda)=\frac{Y}{\lambda}.
\end{equation}
\end{theorem}

\begin{proof}
We first consider the cases in which $\lambda>\lambda_{max}$.

When $\lambda>\lambda_{max}$, we can see that $\|B_i^{\rm T}\frac{Y}{\lambda}\|_{\bar{q}}<1$ for all $i = 1,2,\ldots,s$, i.e., $\frac{Y}{\lambda}$ is itself an interior point of $\mathcal{F}$. As a result, we have $\theta^*(\lambda)=\frac{Y}{\lambda}$.  Therefore, the rule in (\ref{rule0}) implies that ${\bf X}^*_i(\lambda)=0$, $i=1,2,\ldots,s$, i.e., the optimal solution ${\bf X}^*(\lambda)$ of problem (\ref{prob:primal}) is 0.

Next let us consider the case in which $\lambda=\lambda_{max}$. Because $\mathcal{F}$ is convex and $\theta^*(\lambda)$ is the projection of $\frac{Y}{\lambda}$ onto $\mathcal{F}$, we can see that $\theta^*(\lambda)$ is nonexpansive with respect to $\lambda$ \cite{Bertsekas2003} and thus $\theta^*(\lambda)$ is a continuous function of $\lambda$. Therefore, it is easy to see that
$$
\theta^*(\lambda_{max})=\lim_{\lambda\downarrow\lambda_{max}}\theta^*(\lambda)
=\frac{Y}{\lambda_{max}}.
$$
Let $h(\lambda):=B{\bf X}^*(\lambda)=\sum_{i=1}^sB_i{\bf X}^*_i(\lambda)$. By \eqref{eqn:KKT1}, we have
$$h(\lambda)=Y-\lambda\theta^*(\lambda),$$
which implies that $h(\lambda)$ is also continuous with respect to $\lambda$. Therefore, we have
$$h(\lambda_{max})=\lim_{\lambda\downarrow\lambda_{max}}h(\lambda)=0.$$
Clearly, we can set $\overline{\bf X}=0$ such that $h(\lambda_{max})=B\overline{X}=0$ can be satisfied. Therefore, both of the KKT conditions in \eqref{eqn:KKT1} and \eqref{eqn:KKT2} are satisfied by $\overline{X}$ and $\theta^*(\lambda_{max})=\frac{Y}{\lambda_{max}}$. Because problem (\ref{prob:primal}) is a convex optimization problem, the satisfaction of the KKT conditions implies that $\overline{\bf X}=0$ is an optimal solution of (\ref{prob:primal}). Therefore, we can choose 0 for ${\bf X}^*(\lambda_{max})$. Moreover, we can see that ${\bf X}^*(\lambda_{max})$ must be zero because otherwise $f({\bf X}^*(\lambda_{max}))=\frac{1}{2}\|Y\|_2^2+\lambda\sum_{i=1}^s\|{\bf X}^*_i(\lambda_{max})\|_{\bar{q}}>\frac{1}{2}\|Y\|_2^2=f(0)$.  Therefore, we have ${\bf X}^*(\lambda_{max})=0$ which completes the proof.
\end{proof}

Suppose we are given two distinct parameters $\lambda'$ and $\lambda''$ and the corresponding optimal solutions of (\ref{prob:dual}) are $\theta^*(\lambda')$ and $\theta^*(\lambda'')$ respectively. Without loss of generality, let us assume $\lambda_{max}\geq\lambda'>\lambda''>0$. Then the variational inequalities \cite{Guler2010} can be written as
\begin{equation}\label{ineqn:vi00}
\left\langle\theta^*(\lambda')-\theta^*(\lambda''),\nabla g(\theta^*(\lambda''))\right\rangle\geq0,
\end{equation}
\begin{equation}\label{ineqn:vi01}
\left\langle\theta^*(\lambda'')-\theta^*(\lambda'),\nabla g(\theta^*(\lambda'))\right\rangle\geq0.
\end{equation}
Because $\nabla g(\theta)=\theta-\frac{Y}{\lambda}$, the variational inequalities in \eqref{ineqn:vi00} and \eqref{ineqn:vi01} can be rewritten as:
\begin{equation}\label{ineqn:vi1}
\left\langle\theta^*(\lambda')-\theta^*(\lambda''),
\theta^*(\lambda'')-\frac{Y}{\lambda''}\right\rangle\geq0,
\end{equation}
\begin{equation}\label{ineqn:vi2}
\left\langle\theta^*(\lambda'')-\theta^*(\lambda'),
\theta^*(\lambda')-\frac{Y}{\lambda'}\right\rangle\geq0.
\end{equation}
When $\lambda'=\lambda_{max}$, Theorem \ref{thm:thetamx} tells that $\theta^*(\lambda')=\theta^*(\lambda_{max})=\frac{Y}{\lambda_{max}}$ and thus the inequality in \eqref{ineqn:vi2} is trivial. Let $B_*:=\argmax_i\|B_i^{\rm T}\theta^*(\lambda_{max})\|_{\bar{q}}$ and $\phi(\theta) := \|B_*^{\rm T}\theta\|_{\bar{q}}$ where $\theta\in\mathbb{R}^{m}$. Then the subdifferential of $\phi(\cdot)$ can be found as:
$$
\partial \phi(\theta):=\{B_*{\bf d}:\|{\bf d}\|_q\leq1,\langle{\bf d}, B_*^{\rm T}\theta\rangle=\|B_*^{\rm T}\theta\|_{\bar{ q}}\}.
$$
To simplify notations, let $\theta_{max}:=\theta^*(\lambda_{max})$. Consider $\partial \phi(\theta_{max})$. It is easy to see that $B_*^{\rm T}\theta_{max}\neq0$ since $\|B_*^{\rm T}\theta_{max}\|_{\bar{q}}=1$.
Let
\begin{equation}
{\bf d}_{max} = {\rm sgn}(B_*^{\rm T}\theta_{max})\odot|B_*^{\rm T}\theta_{max}|^{\bar{q}/q}.
\end{equation}
It is easy to check that $\|{\bf d}_{max}\|_q=1$ and $\langle{\bf d}_{max}, B_*^{\rm T}\theta_{max}\rangle=\|B_*^{\rm T}\theta_{max}\|_{\bar{ q}}=1$, which implies that $B_*{\bf d}_{max}\in\partial \phi(\theta_{max})$. As a result, the hyperplane
\begin{equation}
\mathcal{H}:=\{\theta\in\mathbb{R}^m:\langle B_*{\bf d}_{max},\theta\rangle=\|B_*^{\rm T}\theta\|_{\bar{ q}}=1\}
\end{equation}
is a supporting hyperplane \cite{Hiriart-Urruty:1993} to the set
$$
\mathcal{C}:=\{\theta\in
\mathbb{R}^m:\|B_*^{\rm T}\theta\|_{\bar{q}}\leq1\}.
$$
As a result, $\mathcal{H}$ is also a supporting hyperplane to the set $\mathcal{F}$. [Recall that $\mathcal{F}$ is the feasible set of problem (\ref{prob:dual}) and $\mathcal{C}\subseteq\mathcal{F}$.] Therefore, when $\lambda'=\lambda_{max}$, we have
\begin{equation}\label{ineqn:vis}
\left\langle\theta-\theta^*(\lambda'),
-B_*{\bf d}_{max}\right\rangle\geq0,\,\,\,\forall \theta\in\mathcal{F}.
\end{equation}

Suppose $\theta^*(\lambda')$ is known, for notational convenience, let
\begin{equation}\label{eqn:a}
{\bf a}(\lambda'',\lambda') = \frac{1}{2}\left[\frac{Y}{\lambda''}-\theta^*(\lambda')\right],
\end{equation}
\begin{equation}\label{eqn:b}
{\bf b}(\lambda') =
\begin{cases}
\frac{Y}{\lambda'}-\theta^*(\lambda'),\hspace{3mm}\mbox{if }\lambda'\in(0,\lambda_{max})\\
B_*{\bf d}_{max},\hspace{3mm}\mbox{if }\lambda'=\lambda_{max},
\end{cases}
\end{equation}
\begin{equation}\label{eqn:v}
{\bf v}(\lambda'',\lambda')={\bf a}(\lambda'',\lambda')-\frac{\langle{\bf a}(\lambda'',\lambda'),{\bf b}(\lambda')\rangle}{\|{\bf b}(\lambda')\|_2^2}{\bf b}(\lambda'),
\end{equation}
and
\begin{equation}\label{eqn:o}
{\bf o}(\lambda'',\lambda') = \theta^*(\lambda') + {\bf v}(\lambda'',\lambda').
\end{equation}
The next lemma shows that the inner product between ${\bf a}(\lambda'',\lambda')$ and ${\bf b}(\lambda')$ is nonnegative.
\begin{lemma}\label{lemma:ab}
Suppose $\theta^*(\lambda')$ and $\theta^*(\lambda'')$ are the optimal solutions of problem (\ref{prob:dual}) for two distinct parameters $\lambda_{max}\geq\lambda'>\lambda''>0$. Then
\begin{equation}
\langle{\bf a}(\lambda'',\lambda'),{\bf b}({\lambda'})\rangle\geq0.
\end{equation}
\end{lemma}
\begin{proof}
We first show that the statement holds when $\lambda'=\lambda_{max}$.

By Theorem \ref{thm:thetamx}, we can see that $\theta^*(\lambda_{max})=\frac{Y}{\lambda_{max}}$ and thus
$${\bf a}(\lambda'',\lambda_{max})=\frac{1}{2}
\left(\frac{1}{\lambda''}-\frac{1}{\lambda_{max}}\right)Y.$$
If $Y=0$, the statement is trivial. Let us assume $Y\neq0$. Then
\begin{equation}\label{ineqn:ab1}
\langle{\bf a}(\lambda'',\lambda_{max}),{\bf b}(\lambda_{max})\rangle=
\frac{1}{2}\left(\frac{1}{\lambda''}-\frac{1}{\lambda_{max}}\right)
\langle Y, B_*{\bf d}_{max}\rangle.
\end{equation}
On the other hand, because the zero point is also a feasible point of problem (\ref{prob:dual}), i.e., $0\in\mathcal{F}$, we can see that
\begin{equation}\label{ineqn:ab2}
\langle0-\theta^*(\lambda_{max}),-B_*{\bf d}_{max}\rangle=\langle-\frac{Y}{\lambda_{max}},-B_*{\bf d}_{max}\rangle\geq0.
\end{equation}
In view of \eqref{ineqn:ab1} and \eqref{ineqn:ab2}, we have $\langle{\bf a}(\lambda'',\lambda_{max}),{\bf b}(\lambda_{max})\rangle\geq0$.

Next, let us consider the case with $0<\lambda'<\lambda_{max}$.
In fact, we can see that
\begin{align*}
\langle2{\bf a}(\lambda'',\lambda'),{\bf b}(\lambda')\rangle&=\left\langle\frac{Y}{\lambda''}-\theta^*(\lambda'),
\frac{Y}{\lambda'}-\theta^*(\lambda')\right\rangle\\
&=\left(\frac{\lambda'}{\lambda''}-1\right)\left\langle \frac{Y}{\lambda'}, \frac{Y}{\lambda'}-\theta^*(\lambda')\right\rangle+\left\|\frac{Y}{\lambda'}-\theta^*(\lambda')\right\|_2^2\\
&\geq \left(\frac{\lambda'}{\lambda''}-1\right)\left\langle \theta^*(\lambda'), \frac{Y}{\lambda'}-\theta^*(\lambda')\right\rangle.
\end{align*}
Because $0\in\mathcal{F}$, the variational inequality leads to
$$
\left\langle0-\theta^*(\lambda'),\nabla g(\theta^*(\lambda'))\right\rangle=\left\langle-\theta^*(\lambda'),\theta^*(\lambda')-\frac{Y}{\lambda'}\right\rangle
\geq0.
$$
Therefore, it is easy to see that $$\langle2{\bf a}(\lambda'',\lambda'),{\bf b}(\lambda')\rangle\geq0,$$ which completes the proof.
\end{proof}

Next we show how to bound $\theta^*(\lambda'')$ inside a ball via the above variational inequalities.

\begin{theorem}\label{thm:ball}
Suppose $\theta^*(\lambda')$ and $\theta^*(\lambda'')$ are the optimal solutions of problem (\ref{prob:dual}) for two distinct parameters $\lambda_{max}\geq\lambda'>\lambda''$ and $\theta^*(\lambda')$ is known. Then
\begin{equation}\label{ineqn:ball}
\|\theta^*(\lambda'')-{\bf o}(\lambda'',\lambda')\|_2^2\leq\|{\bf v}(\lambda'',\lambda')\|_2^2.
\end{equation}
\end{theorem}

\begin{proof}
To simplify notations, let $c = \frac{\langle{\bf a}(\lambda'',\lambda'),{\bf b}(\lambda')\rangle}{\|{\bf b}(\lambda')\|_2^2}$. Then, Lemma \ref{lemma:ab} implies that $c\geq0$.

Because $\theta^*(\lambda'')\in\mathcal{F}$, the inequalities in \eqref{ineqn:vi2} and \eqref{ineqn:vis} lead to
\begin{equation}
\left\langle\theta^*(\lambda'')-\theta^*(\lambda'),
{\bf b}(\lambda')\right\rangle\leq0,
\end{equation}
and thus
\begin{equation}\label{ineqn:vi3}
\left\langle\theta^*(\lambda'')-\theta^*(\lambda'),
2c{\bf b}(\lambda')\right\rangle\leq0.
\end{equation}
By adding the inequalities in \eqref{ineqn:vi1} and \eqref{ineqn:vi3}, we obtain
\begin{equation}\label{ineqn:ball0}
\left\langle\theta^*(\lambda'')-\theta^*(\lambda'),
\theta^*(\lambda'')-\frac{Y}{\lambda''}+2c{\bf b}(\lambda')\right\rangle\leq0.
\end{equation}
By noting that
\begin{align*}
\theta^*(\lambda'')-\frac{Y}{\lambda''}+c{\bf b}(\lambda')& = \theta^*(\lambda'')-\theta^*(\lambda')-\left(\frac{Y}{\lambda''}
-\theta^*(\lambda')-2c{\bf b}(\lambda')\right)\\
&=\theta^*(\lambda'')-\theta^*(\lambda')-2\left({\bf a}(\lambda'',\lambda')-c{\bf b}(\lambda')\right)\\
&=\theta^*(\lambda'')-\theta^*(\lambda')-2{\bf v}(\lambda'',\lambda'),
\end{align*}
the inequality in \eqref{ineqn:ball0} becomes
$$\left\langle\theta^*(\lambda'')-\theta^*(\lambda'),
\theta^*(\lambda'')-\theta^*(\lambda')-2{\bf v}(\lambda'',\lambda')\right\rangle\leq0,$$
which is equivalent to \eqref{ineqn:ball}.

\end{proof}

Theorem \ref{thm:ball} bounds $\theta^*(\lambda'')$ inside a ball. For notational convenience, let us denote
\begin{equation}\label{def:ball}
\mathcal{R}(\lambda'',\lambda')=\{\theta:\|\theta-{\bf o}(\lambda'',\lambda')\|_2\leq\|{\bf v}(\lambda'',\lambda')\|_2\}.
\end{equation}

\subsection{Estimating the Upper Bound}\label{subsection:bound}

Given two distinct parameters $\lambda'$ and $\lambda''$, and assume the knowledge of $\theta^*(\lambda')$, we estimate a possible region $\mathcal{R}(\lambda'',\lambda')$ for $\theta^*(\lambda'')$ in Section \ref{subsection:region}. To apply the screening rule in (\ref{rule}) to identify the inactive groups, we need to estimate an upper bound of $\varphi(\theta^*(\lambda''),B_i)=\max_{\theta\in\mathcal{R}({\lambda'',\lambda'})}\|B_i^{\rm T}\theta\|_{\bar{q}}$. If $\varphi(\theta^*(\lambda''),B_i)<1$, then ${\bf X}^*_i(\lambda'')=0$ and the $i^{th}$ group can be safely removed from the optimization of problem (\ref{prob:primal}).

We need the following technical lemma for the estimation of an upper bound.
\begin{lemma}\label{lemma:matrixbd}
Let $A\in\mathbb{R}^{n\times t}$ be a matrix and ${\bf u}\in\mathbb{R}^t$, and $q\in[1,\infty]$. Then the following holds:
\begin{equation}
\|A{\bf u}\|_q\leq T_{A}^q\|{\bf u}\|_2,
\end{equation}
where
\begin{equation}
T_{A}^q=
\begin{cases}
\left(\sum_{k=1}^n\left(\sum_{j=1}^t a_{k,j}^2\right)^{q/2}\right)^{1/q},\hspace{5mm}\mbox{if }q\in[1,\infty),\\
\max_{k\in\{1,\ldots,n\}}
\left(\sum_{j=1}^t a_{k,j}^2\right)^{1/2},\hspace{3mm}\mbox{if }q=\infty,
\end{cases}
\end{equation}
and $a_{k,j}$ is the $(k,j)^{th}$ entry of $A$.
\end{lemma}

\begin{proof}
Let $A = [{\bf a}_1,{\bf a}_2, \ldots,{\bf a}_n]^{\rm T}$, i.e., ${\bf a}_k$ is the $k^{th}$ row of $A$. When $q>1$, we can see that
\begin{align}
\|A{\bf u}\|_q=\left(\sum_{k=1}^n|\langle{\bf a}_k,{\bf u}\rangle|^q\right)^{1/q}\leq\left(\sum_{k=1}^n\|{\bf a}_k\|_2^q\|{\bf u}\|_2^q\right)^{1/q}=\left(\sum_{k=1}^n\|{\bf a}_k\|_2^q\right)^{1/q}\|{\bf u}\|_2.
\end{align}
By a similar argument, we can prove the statement with $q=\infty$.
\end{proof}

%For notational convenience, let $T_{A}^q:=\left(\sum_{k=1}^n\left(\sum_{j=1}^t a_{k,j}^2\right)^{q/2}\right)^{1/q}$.

%\begin{lemma}\label{lemma:lp_norm}
%\cite{Steele2004} Let ${\bf z}\in\mathbb{R}^t$ and $p>q>0$. We have
%\begin{equation}
%\|{\bf z}\|_p\leq\|{\bf z}\|_q\leq t^{(1/q-1/p)}\|{\bf z}\|_p.
%\end{equation}
%\end{lemma}

The following theorem gives an upper bound of $\varphi(\theta^*(\lambda''),B_i)$.
\begin{theorem}\label{thm:bound}
Given two distinct parameters $\lambda_{max}\geq\lambda'>\lambda''>0$. Let $\mathcal{R}(\lambda'',\lambda')$ be defined as in \eqref{def:ball}, then
\begin{equation}\label{ineqn:bound}
\varphi(\theta^*(\lambda''),B_i)=\max_{\theta\in\mathcal{R}({\lambda'',\lambda'})}\|B_i^{\rm T}\theta\|_{\bar{q}}\leq T_{B_i^{\rm T}}^{\bar{q}}\|{\bf v}(\lambda'',\lambda')\|_2+\|B_i^{\rm T}{\bf o(\lambda'',\lambda')}\|_{\bar{q}},
\end{equation}
where ${\bf o}(\lambda'',\lambda')$ and ${\bf v}(\lambda'',\lambda')$ are defined in \eqref{eqn:o} and \eqref{eqn:v} respectively.
\end{theorem}

\begin{proof}
Recall that
\begin{align*}
\mathcal{R}(\lambda'',\lambda')=\{\theta:\|\theta-{\bf o}(\lambda'',\lambda')\|_2\leq\|{\bf v}(\lambda'',\lambda')\|_2\}.
\end{align*}
Then for $\theta\in\mathcal{R}(\lambda'',\lambda')$, we have
\begin{align}\label{ineqn:key_bound}
\|B_i^{\rm T}\theta\|_{\bar{q}}&\leq\|B_i^{\rm T}(\theta-{\bf o}(\lambda'',\lambda'))\|_{\bar{q}}+\|B_i^{\rm T}{\bf o}(\lambda'',\lambda')\|_{\bar{q}}\\ \nonumber
&\leq T_{B_i^{\rm T}}^{\bar{q}}\|\theta-{\bf o}(\lambda'',\lambda')\|_2+\|B_i^{\rm T}{\bf o}(\lambda'',\lambda')\|_{\bar{q}}\\ \nonumber
&\leq T_{B_i^{\rm T}}^{\bar{q}}\|{\bf v}(\lambda'',\lambda')\|_2+\|B_i^{\rm T}{\bf o}(\lambda'',\lambda')\|_{\bar{q}}.
\end{align}
The second inequality in \eqref{ineqn:key_bound} follows from Lemma \ref{lemma:matrixbd} and the proof is completed.
\end{proof}

\subsection{The Proposed Screening Method (Smin) for $\ell_1/\ell_q$-Regularized Problems}\label{subsection:rules}

Using (\ref{rule}), we are now ready to construct screening rules for the $\ell_1/\ell_q$-regularized problems.

\begin{theorem}\label{thm:rule}
For problem (\ref{prob:primal}), assume $\theta^*(\lambda')$ is known for a specific parameter $\lambda'\in(0,\lambda_{max}]$. Let $\lambda''\in(0,\lambda')$ . Then ${\bf X}^*_i(\lambda'')=0$ if
\begin{equation}\label{ineqn:rule}
\|B_i^{\rm T}{\bf o}(\lambda'',\lambda')\|_{\bar{q}}<1-T_{B_i^{\rm T}}^{\bar{q}}\|{\bf v}(\lambda'',\lambda')\|_2.
\end{equation}
\end{theorem}

\begin{proof}
From (\ref{rule}), we know that
\begin{equation}
\varphi(\theta^*(\lambda''),B_i)<1\Rightarrow {\bf X}^*_i(\lambda'')=0.
\end{equation}
By Theorem \ref{thm:bound}, we have
\begin{equation}\label{ineqn:rule0}
\varphi(\theta^*(\lambda''),B_i)\leq T_{B_i^{\rm T}}^{\bar{q}}\|{\bf v}(\lambda'',\lambda')\|_2+\|B_i^{\rm T}{\bf o(\lambda'',\lambda')}\|_{\bar{q}}.
\end{equation}
In view of \eqref{ineqn:rule}, \eqref{ineqn:rule0} results in
$$
\varphi(\theta^*(\lambda''),B_i)<1,
$$
which completes the proof.
\end{proof}

By setting $\lambda'=\lambda_{max}$ in Theorem \ref{thm:rule}, we immediately obtain the following basic screening rule.

\begin{corollary}$({\rm{Smin}_{b}})$
Consider problem (\ref{prob:primal}), let $\lambda_{max}=\max_i\|B_i^{\rm T}Y\|_{\bar{q}}$. If $\lambda\geq\lambda_{max}$, ${\bf X}^*_i(\lambda)=0$ for all $i=1,\ldots,s$. Otherwise, we have ${\bf X}^*_i(\lambda)=0$ if the following holds:
\begin{equation}
\|B_i^{\rm T}{\bf o}(\lambda,\lambda_{max})\|_{\bar{q}}<1-T_{B_i^{\rm T}}^{\bar {q}}\|{\bf v}(\lambda'',\lambda_{max})\|_2,
\end{equation}
where $\theta^*(\lambda_{max})=\frac{Y}{\lambda_{max}}$, ${\bf o}(\lambda'',\lambda_{max})$ and ${\bf v}(\lambda'',\lambda_{max})$ are defined in \eqref{eqn:o} and \eqref{eqn:v} respectively.
\end{corollary}

%\begin{proof}
%The conclusion of the cases in which $\lambda\geq\lambda_{max}$ has been proved in Theorem \ref{thm:thetamx}. Therefore, we only focus on the cases in which $\lambda\in(0,\lambda_{max})$.
%
%Let $\lambda'=\lambda_{max}$ and $\lambda''=\lambda$. Theorem \ref{thm:thetamx} tells that $\theta^*(\lambda')=\theta^*(\lambda_{max})=\frac{Y}{\lambda_{max}}$. Therefore, as computed in Theorem \ref{thm:bound}, we have
%${\bf b}=0$ and thus ${\bf v}(\lambda;\lambda_{max})={\bf a}=\frac{1}{2}(\frac{Y}{\lambda}-\frac{Y}{\lambda_{max}})$. The statement follows by noting that ${\bf o}(\lambda;\lambda_{max})=\theta^*(\lambda_{max})+{\bf v}(\lambda;\lambda_{max})$.
%\end{proof}

In practical applications, the optimal parameter value of $\lambda$ is unknown and needs to be estimated. Commonly used approaches such as cross validation and stability selection involve solving the $\ell_1/\ell_q$-regularized problems over a grid of tuning parameters $\lambda_1>\lambda_2>\ldots>\lambda_K$ to determine an appropriate value for $\lambda$. As a result, the computation is very time consuming. To address this challenge, we propose the sequential version of the proposed Smin.
\begin{corollary}$({\rm{Smin}_{s}})$
For the problem in (\ref{prob:primal}), suppose we are given a sequence of parameter values $\lambda_{max}=\lambda_0>\lambda_1>\ldots>\lambda_{\mathcal{K}}$. For any integer $0\leq k<\mathcal{K}$, we have ${\bf X}^*_i(\lambda_{k+1})=0$ if ${\bf X}^*(\lambda_k)$ is known and the following holds:
\begin{equation}
\|B_i^{\rm T}{\bf o}(\lambda_{k+1},\lambda_k)\|_{\bar{q}}<1-T_{B_i^{\rm T}}^{\bar{q}}\|{\bf v}(\lambda_{k+1},\lambda_k)\|_2,
\end{equation}
where $\theta^*(\lambda_k)=\frac{Y-\sum_{i=1}^sB_i{\bf X}_i^*(\lambda_k)}{\lambda_k}$, ${\bf o}(\lambda_{k+1},\lambda_k)$ and ${\bf v}(\lambda_{k+1},\lambda_k)$ are defined in \eqref{eqn:o} and \eqref{eqn:v} respectively.

\end{corollary}
\begin{proof}
The statement easily follows by setting $\lambda''=\lambda_{k+1}$ and $\lambda'=\lambda_{k}$ and applying \eqref{eqn:KKT1} and Theorem \ref{thm:rule}.
\end{proof}

\section{Experiments}\label{s:experiment}

In Sections \ref{subsection:exp1} and \ref{subsection:exp2}, we conduct experiments to evaluate the efficiency of the
proposed algorithm, that is, GLEP$_1q$, using both synthetic and real-world data sets. We evaluate the proposed Smin$_s$ on large-scale data sets and compare the performance of Smin$_s$ with DPP and strong rules which achieve state-of-the-art performance for the $\ell_1/\ell_q$-regularized problems in Section \ref{subsection:screening}. We set
the regularization parameter as $\lambda= r \times
\lambda_{\max}$, where $0 < r \leq 1$ is the ratio, and
$\lambda_{\max}$ is the maximal value above which the
$\ell_1/\ell_q$-norm regularized problem
(\ref{eq:optimization:problem}) obtains a zero solution (see
Theorem~\ref{thm:thetamx}). We try the following values for
$q$: $1,1.25, 1.5, 1.75, 2, 2.33, 3, 5$, and $\infty$. The source
codes, included in the SLEP package~\cite{Liu:2009:SLEP:manual}, are
available
online\footnote{\url{http://www.public.asu.edu/~jye02/Software/SLEP/}}.

\subsection{Simulation Studies}\label{subsection:exp1}

We use the synthetic data to study the effectiveness of the
$\ell_1/\ell_q$-norm regularization for reconstructing the jointly
sparse matrix under different values of $q>1$. Let $A \in
\mathbb{R}^{m \times d}$ be a measurement matrix with entries being
generated randomly from the standard normal distribution, $X^* \in
\mathbb{R}^{d \times k}$ be the jointly sparse matrix with the first
$\tilde d < d$ rows being nonzero and the remaining rows exactly
zero, $Y = A X^* + Z$ be the response matrix, and $Z \in
\mathbb{R}^{m \times k}$ be the noise matrix whose entries are drawn
randomly from the normal distribution with mean zero and standard
deviation $\sigma=0.1$. We treat each row of $X^*$ as a group, and
estimate $X^*$ from $A$ and $Y$ by solving the following
$\ell_1/\ell_q$-norm regularized problem: $$
    X= \arg \min_{W} \frac{1}{2}\|AW-Y\|_F^2 + \lambda \sum_{i=1}^d \|W^i\|_q,$$
where $W^i$ denotes the $i$-th row of $W$. We set $m=100$, $d=200$,
and $\tilde d=k=50$. We try two different settings for $X^*$, by
drawing its nonzero entries randomly from 1) the uniform
distribution in the interval $[0,1]$ and 2) the standard normal
distribution.

\begin{figure}
  \centering
  \includegraphics[width=2.1in]{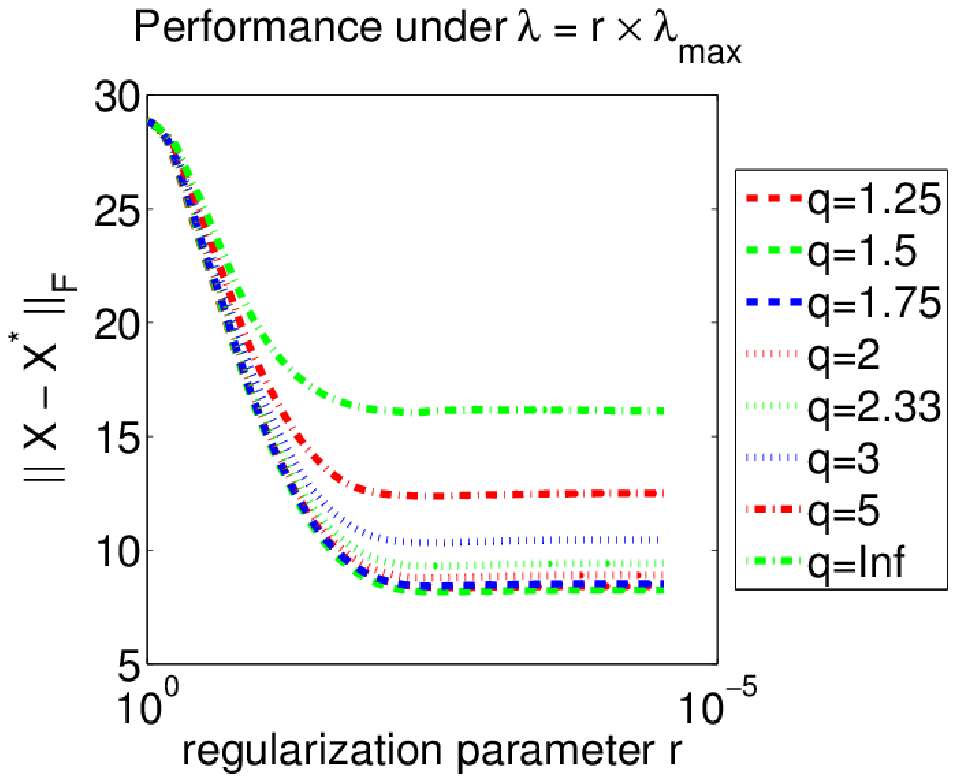}
  \hspace{0.2in}
  \includegraphics[width=2.1in]{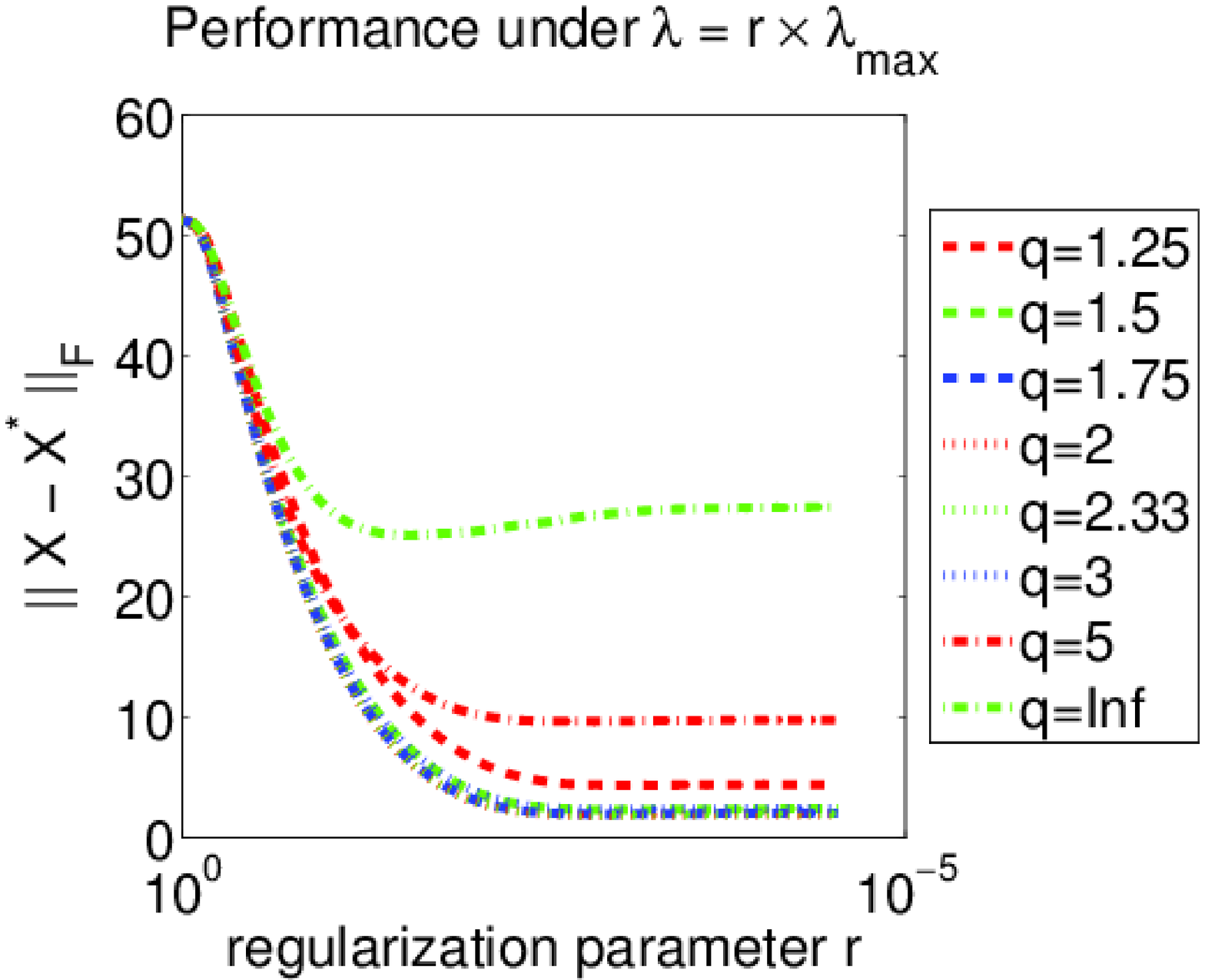}\\

  \includegraphics[width=2.1in]{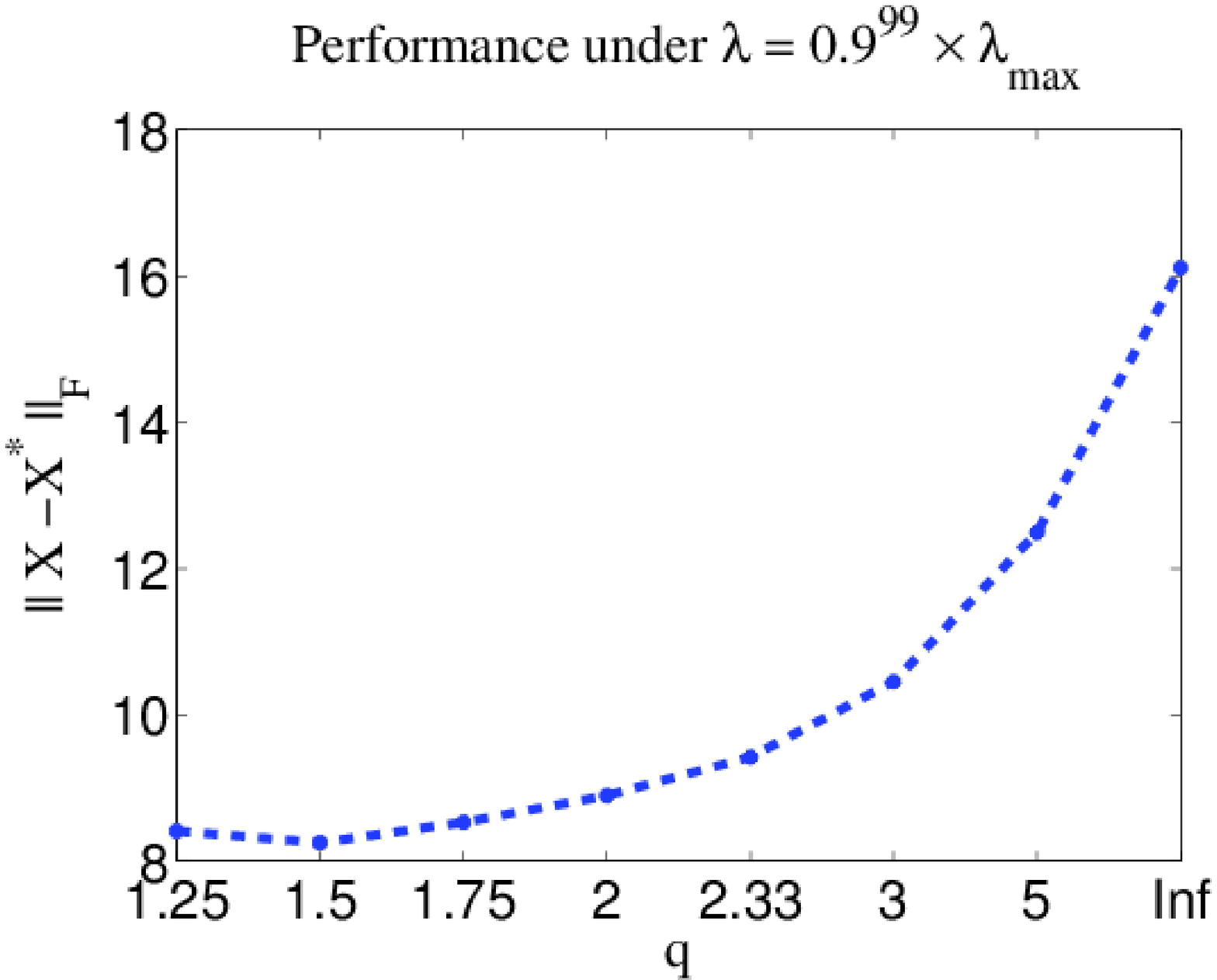}
  \hspace{0.2in}
  \includegraphics[width=2.1in]{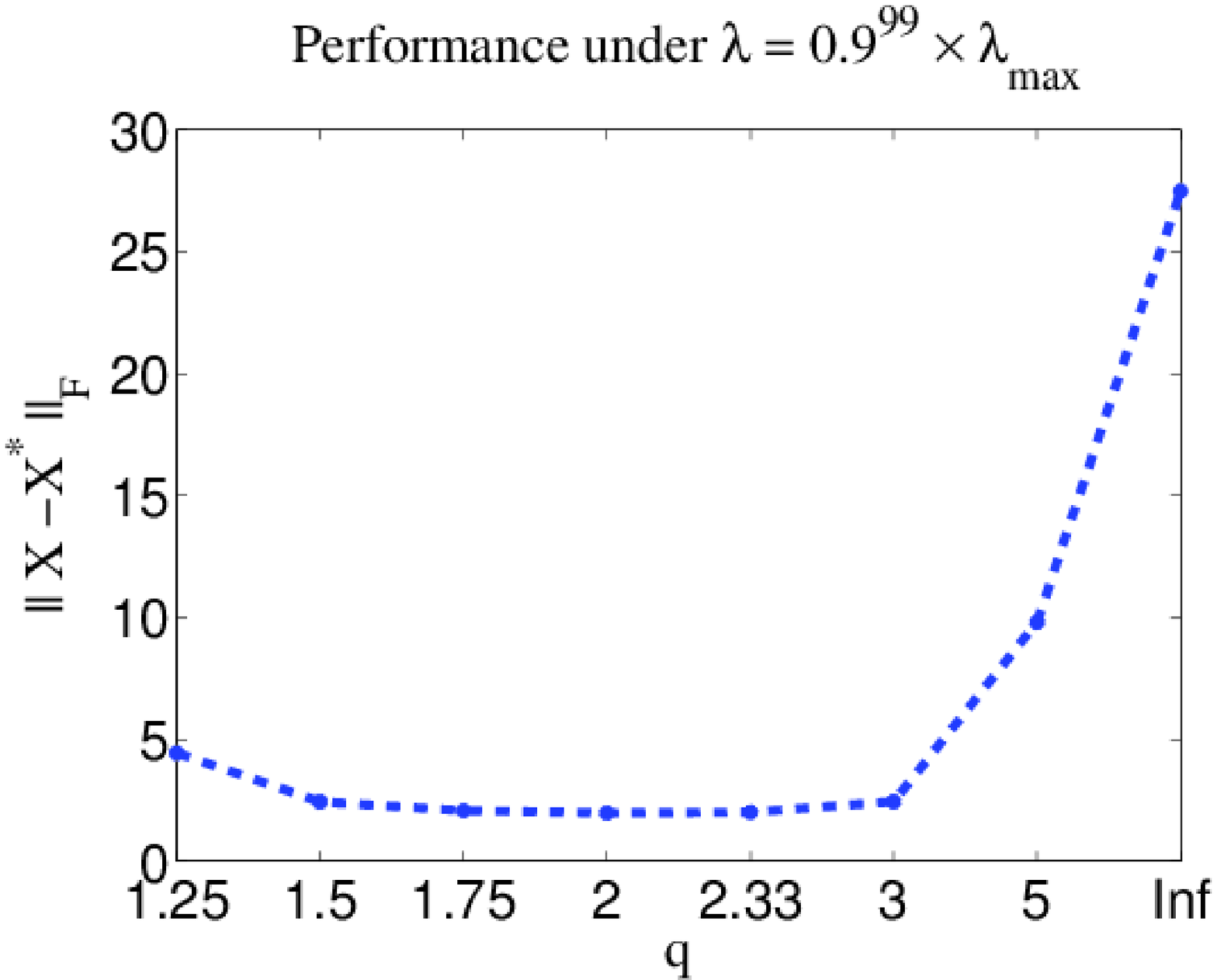}\\

  \includegraphics[width=2.1in]{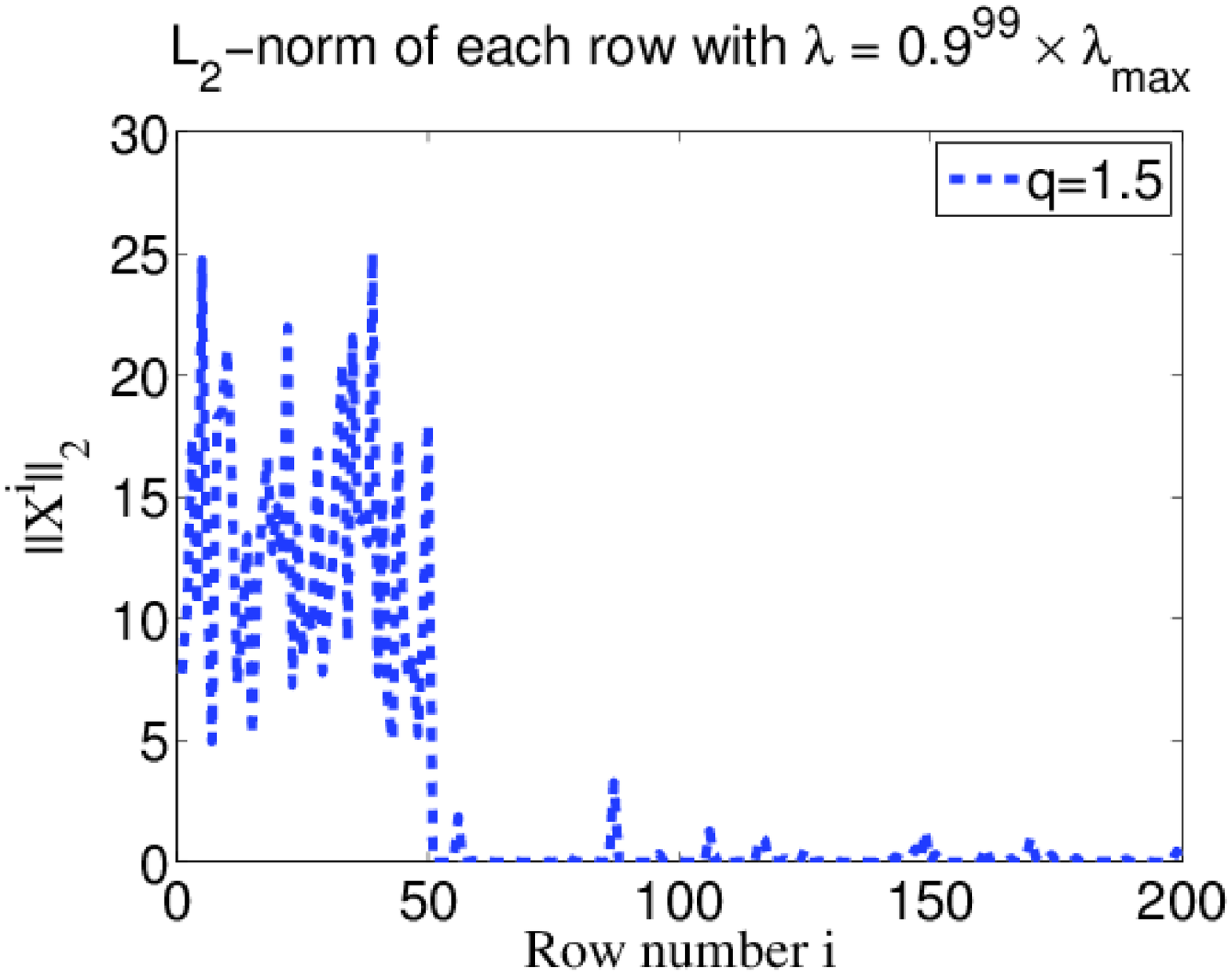}
  \hspace{0.2in}
  \includegraphics[width=2.1in]{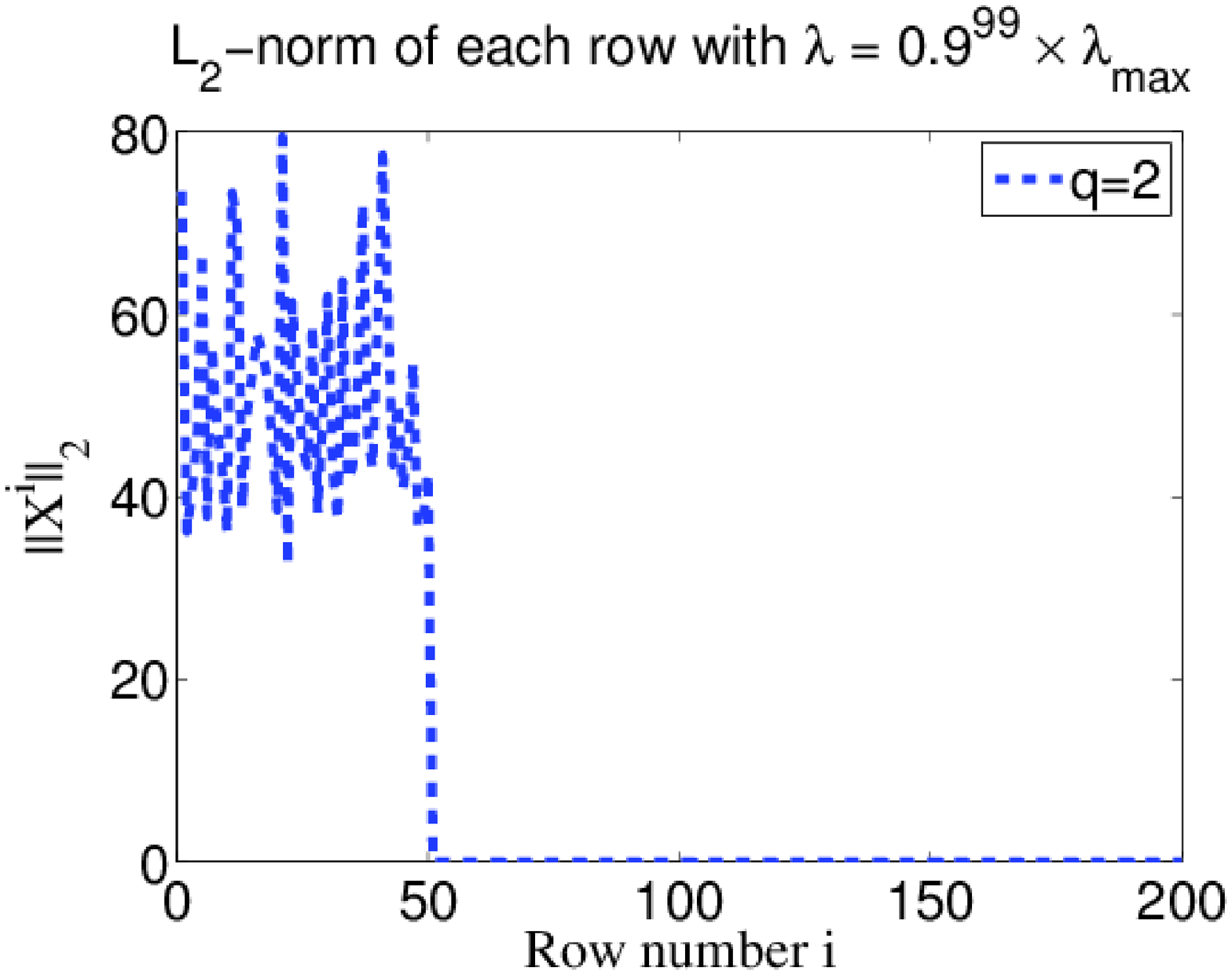}\\
\caption{ {\small Performance of the $\ell_1/\ell_q$-norm
regularization for reconstructing the jointly sparse $X^*$. The
nonzero entries of $X^*$ are  drawn randomly from the uniform
distribution for the plots in the first column, and from the normal
distribution for the plots in the second column. Plots in the first two
rows show $\|X-X^*\|_F$, the Frobenius norm difference between the
solution and the truth; and plots in the third row show the
$\ell_2$-norm of each row of the solution
$X$.}}\label{fig:synthetic}
\end{figure}

We compute the solutions corresponding to a sequence of decreasing
values of $\lambda= r \times \lambda_{\max}$, where $r=0.9^{i-1}$,
for $i=1, 2, \ldots, 100$. In addition, we use the solution
corresponding to the $0.9^i \times \lambda_{\max}$ as the ``warm"
start for $0.9^{i+1} \times \lambda_{\max}$. We report the results
in Figure~\ref{fig:synthetic}, from which we can observe: 1) the
distance between the solution $X$ and the truth $X^*$ usually
decreases with decreasing values of $\lambda$; 2) for the uniform
distribution (see the plots in the first row), $q=1.5$ performs the
best; 3) for the normal distribution (see the plots in the second
row), $q=1.5, 1.75, 2$ and 3 achieve comparable performance and
perform better than $q=1.25$, 5 and $\infty$; 4) with a properly
chosen threshold, the support of $X^*$ can be exactly recovered by
the $\ell_1/\ell_q$-norm regularization with an appropriate value of
$q$, e.g., $q=1.5$ for the uniform distribution, and $q=2$ for the
normal distribution; and 5) the recovery of $X^*$ with nonzero
entries drawn from the normal distribution is more accurate than that with
entries generated from the uniform distribution.

The existing theoretical
results~\cite{Liu:han:2009:report,Negahban:2009} can not tell which
$q$ is the best; and we believe that the optimal $q$ depends on the
distribution of $X^*$, as indicated from the above results.
Therefore, it is necessary to conduct the distribution-specific
theoretical studies (note that the previous studies usually make no
assumption on $X^*$). %The previous theoretical studies are mainly for
%a large value of the parameter $\lambda$; however, our experimental
%results show that the solution corresponding to a smaller $\lambda$
%(with the ``warm" start) can outperform the larger ones in certain
%cases. Thus, we should perform theoretical studies for small values
%of $\lambda$.
The proposed GLEP$_{1q}$ algorithm shall help verify the theoretical
results to be established.

\subsection{Performance on the Letter Data Set}\label{subsection:exp2}

We apply the proposed GLEP$_{1q}$ algorithm for multi-task learning
on the Letter data set~\cite{Obozinski:2007}, which consists of
45,679 samples from 8 default tasks of two-class classification
problems for the handwritten letters: c/e, g/y, m/n, a/g, i/j, a/o,
f/t, h/n. The writings were collected from over 180 different
writers, with the letters being represented by $8\times 16$ binary
pixel images. We use the least squares loss for $l(\cdot)$.

\begin{figure}[h]
  \centering
  \includegraphics[width=2.1in]{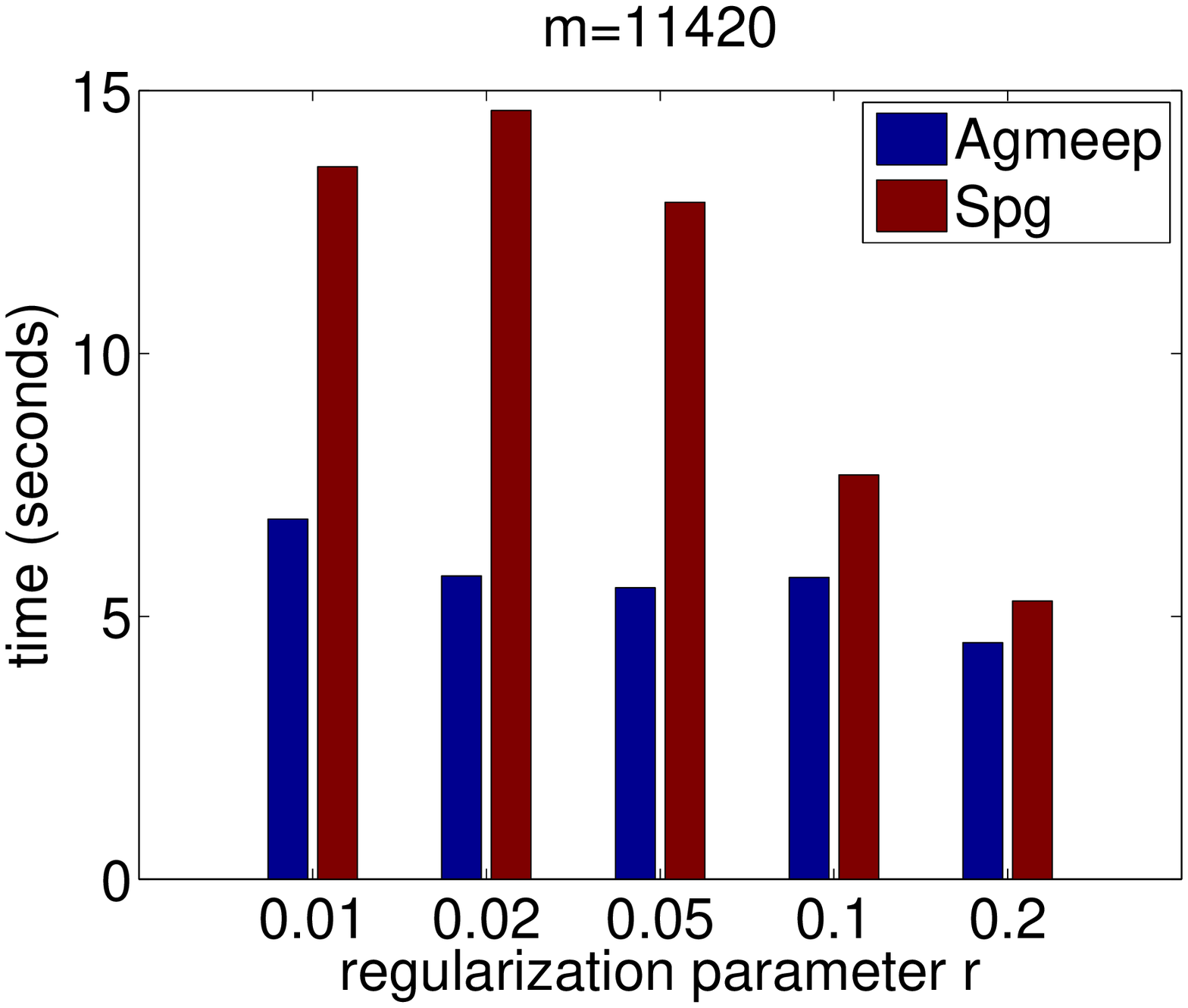}
  \includegraphics[width=2.1in]{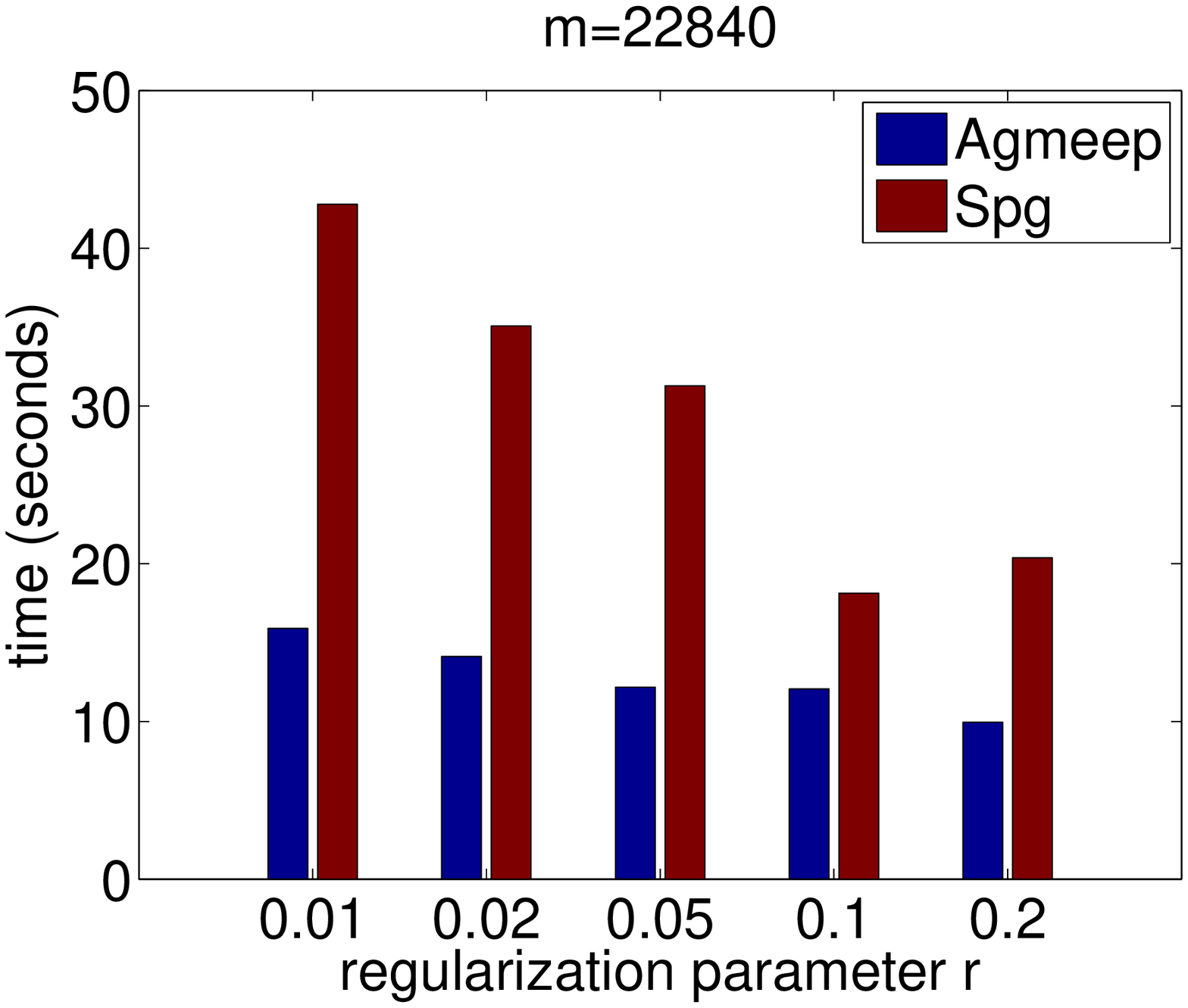}\\
  \includegraphics[width=2.1in]{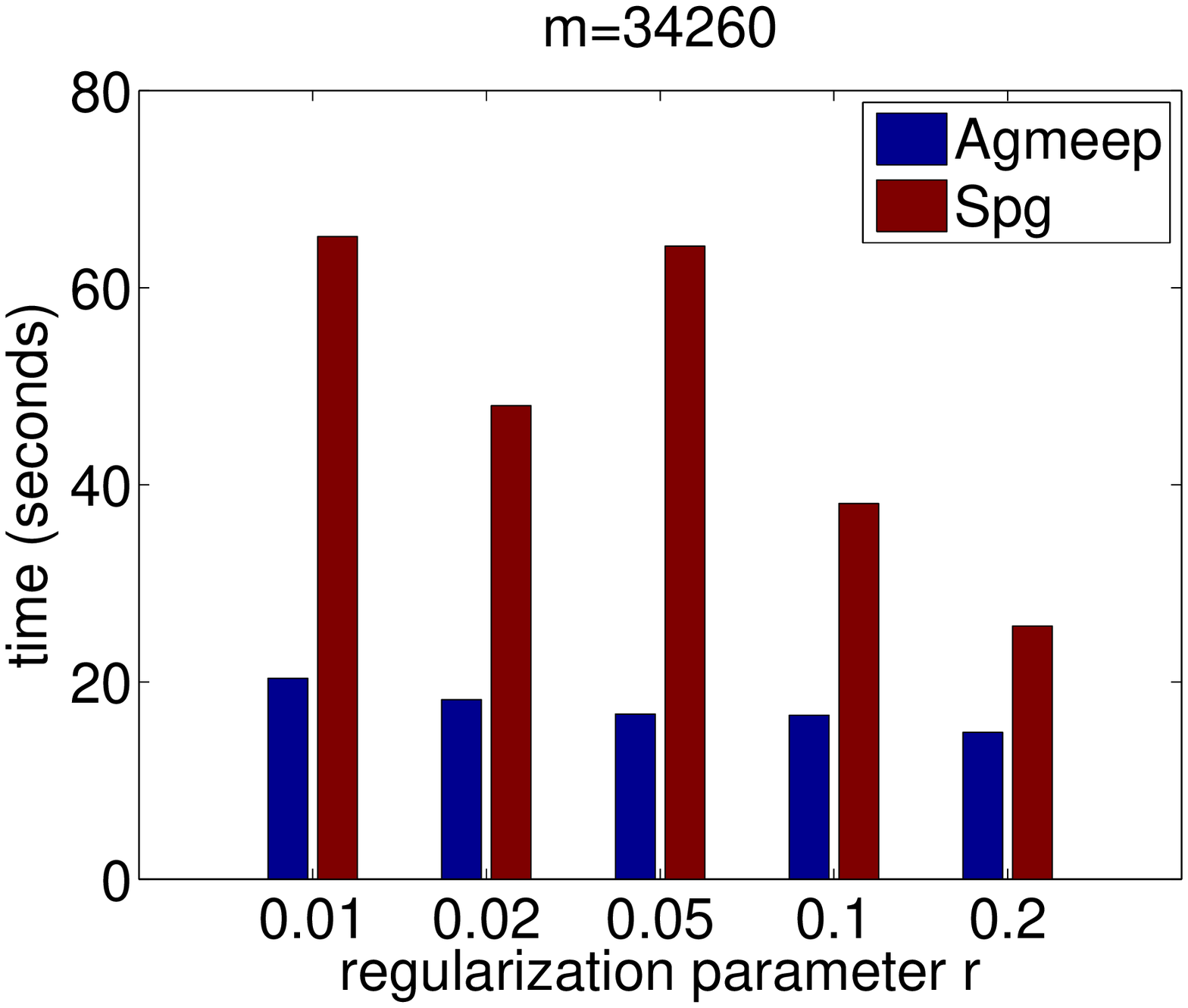}
  \includegraphics[width=2.1in]{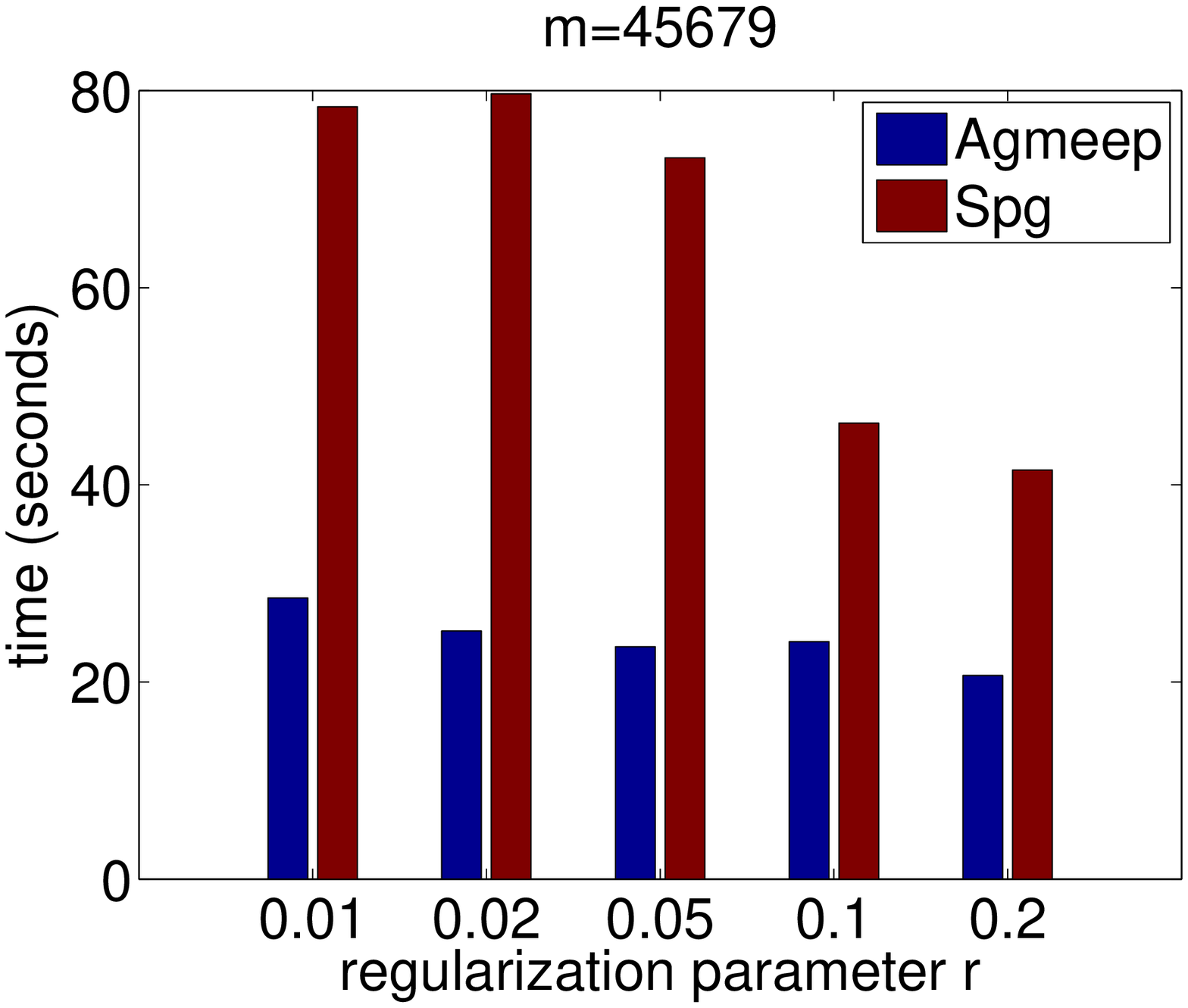}
  \caption{{\small Computational time (seconds) comparison between GLEP$_{1q}$ ($q=2$) and Spg under different values
  of $\lambda=r \times \lambda_{\max}$ and $m$.}}\label{fig:comp}
\end{figure}

\subsubsection{Efficiency Comparison with Spg} We compare GLEP$_{1q}$
with the Spg algorithm proposed in~\cite{BergFriedlander:2008}. Spg
is a specialized solver for the $\ell_1/\ell_2$-ball constrained
optimization problem, and has been shown to outperform existing
algorithms based on blockwise coordinate descent and projected
gradient. In Figure~\ref{fig:comp}, we report the computational time
under different values of $m$ (the number of samples) and $\lambda=r
\times \lambda_{\max}$ ($q=2$). It is clear from the plots that
GLEP$_{1q}$ is much more efficient than Spg, which may attribute to:
1) GLEP$_{1q}$ has a better convergence rate than Spg; and 2) when
$q=2$, the EP$_{1q}$ in GLEP$_{1q}$ can be computed analytically (see
Remark~\ref{remark:q=2}), while this is not the case in Spg.

\subsubsection{Efficiency under Different Values of $q$ } We
report the computational time (seconds) of GLEP$_{1q}$ under
different values of $q$, $\lambda=r \times \lambda_{\max}$ and $m$
(the number of samples) in Figure~\ref{fig:time}. We can observe from
this figure that the computational time of GLEP$_{1q}$ under
different values of $q$ (for fixed $r$ and $m$) is comparable.
Together with the result on the comparison with Spg for $q=2$, this
experiment shows the promise of GLEP$_{1q}$ for solving large-scale
problems for any $q \geq 1$.

%, and on the other hand, this partially verifies the
%efficiency of the EP$_{1q}$ proposed in Section~\ref{s:ep}.

\begin{figure}
  \centering
  \includegraphics[width=2.1in]{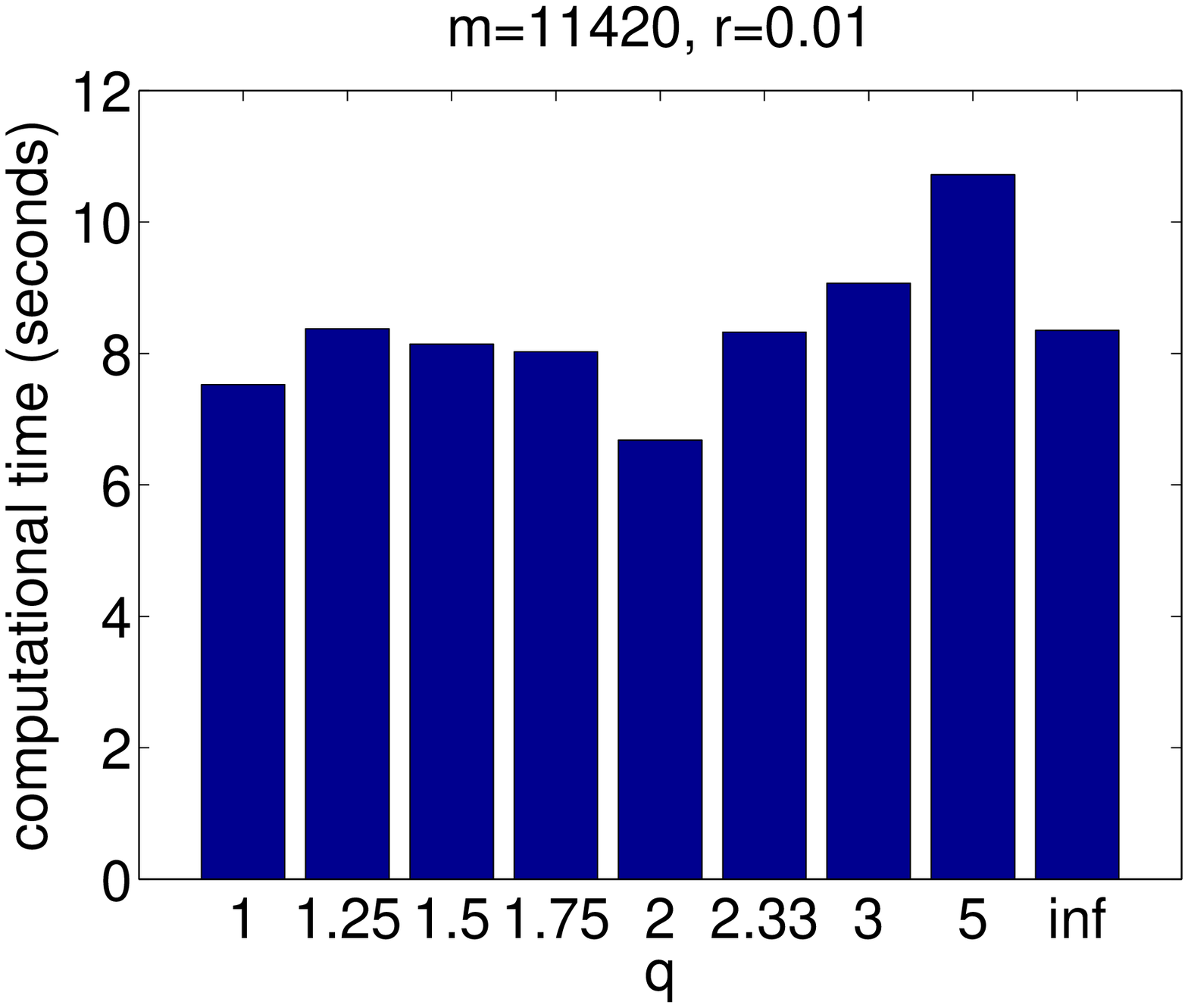}
  \includegraphics[width=2.1in]{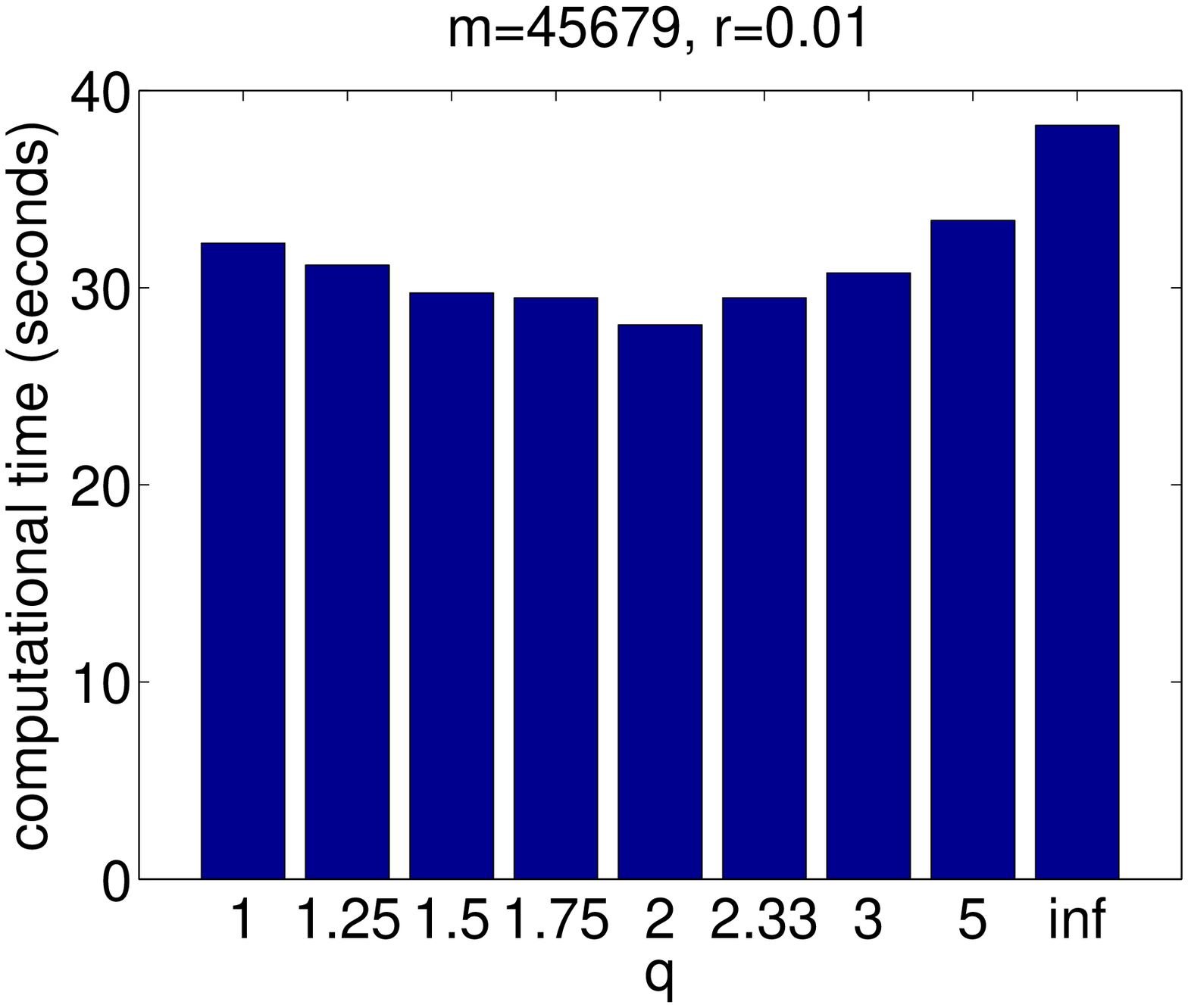}\\
  \includegraphics[width=2.1in]{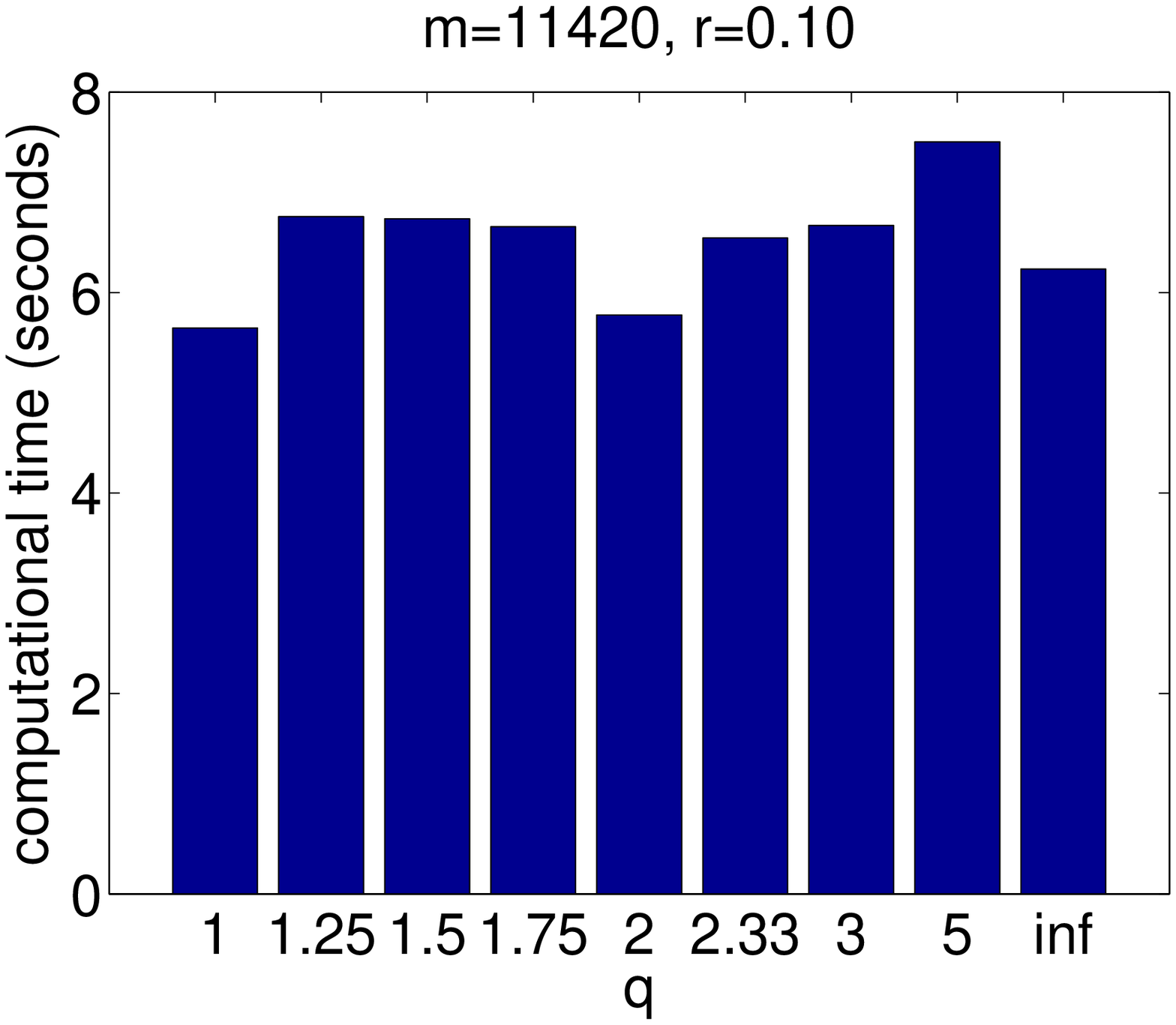}
  \includegraphics[width=2.1in]{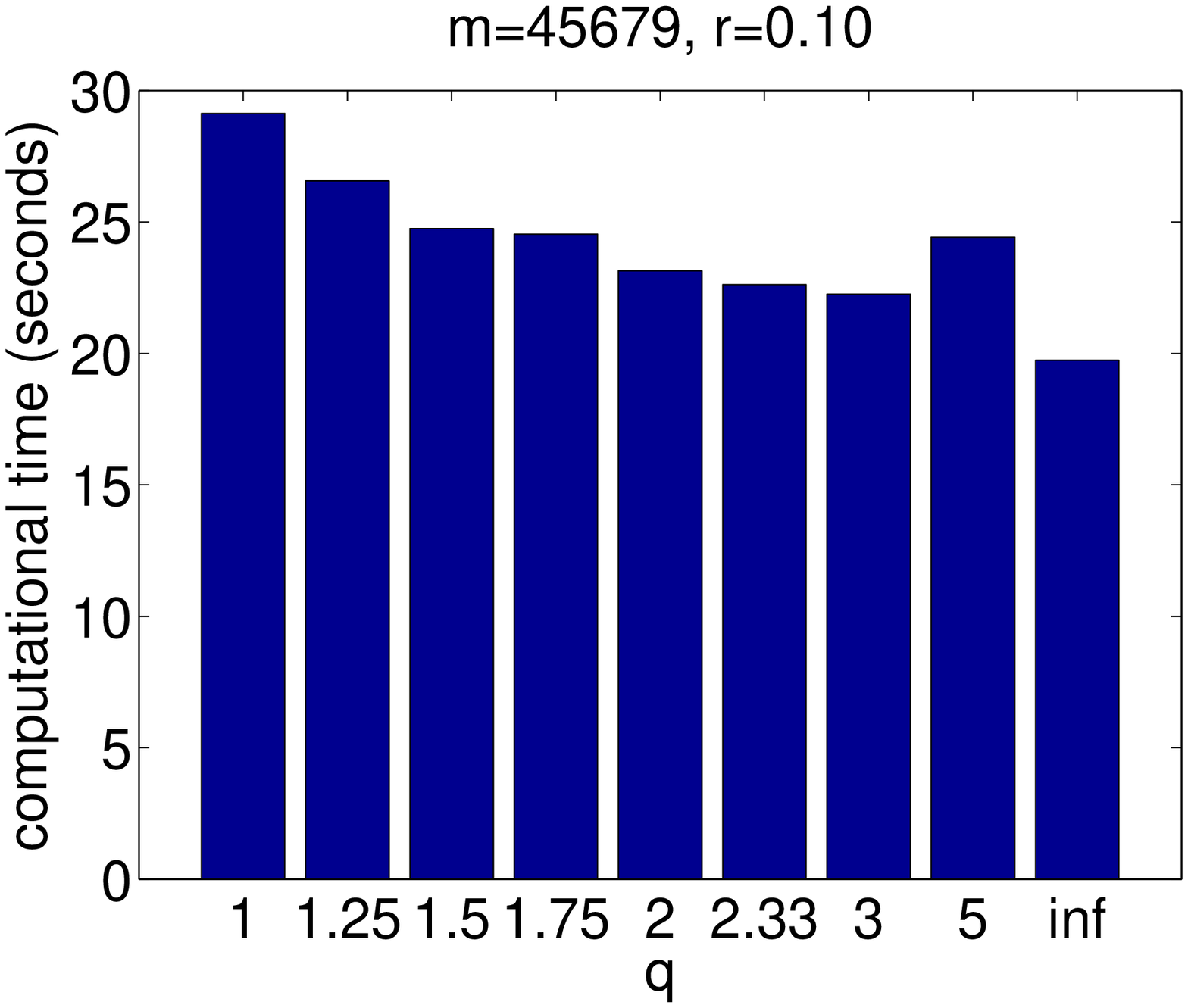}\\
  \caption{{\small Computation time (seconds) of GLEP$_{1q}$ under different values of
  $m$, $q$ and $r$.}} \label{fig:time}
\end{figure}

\begin{figure}
  \centering
  \includegraphics[width=2.1in]{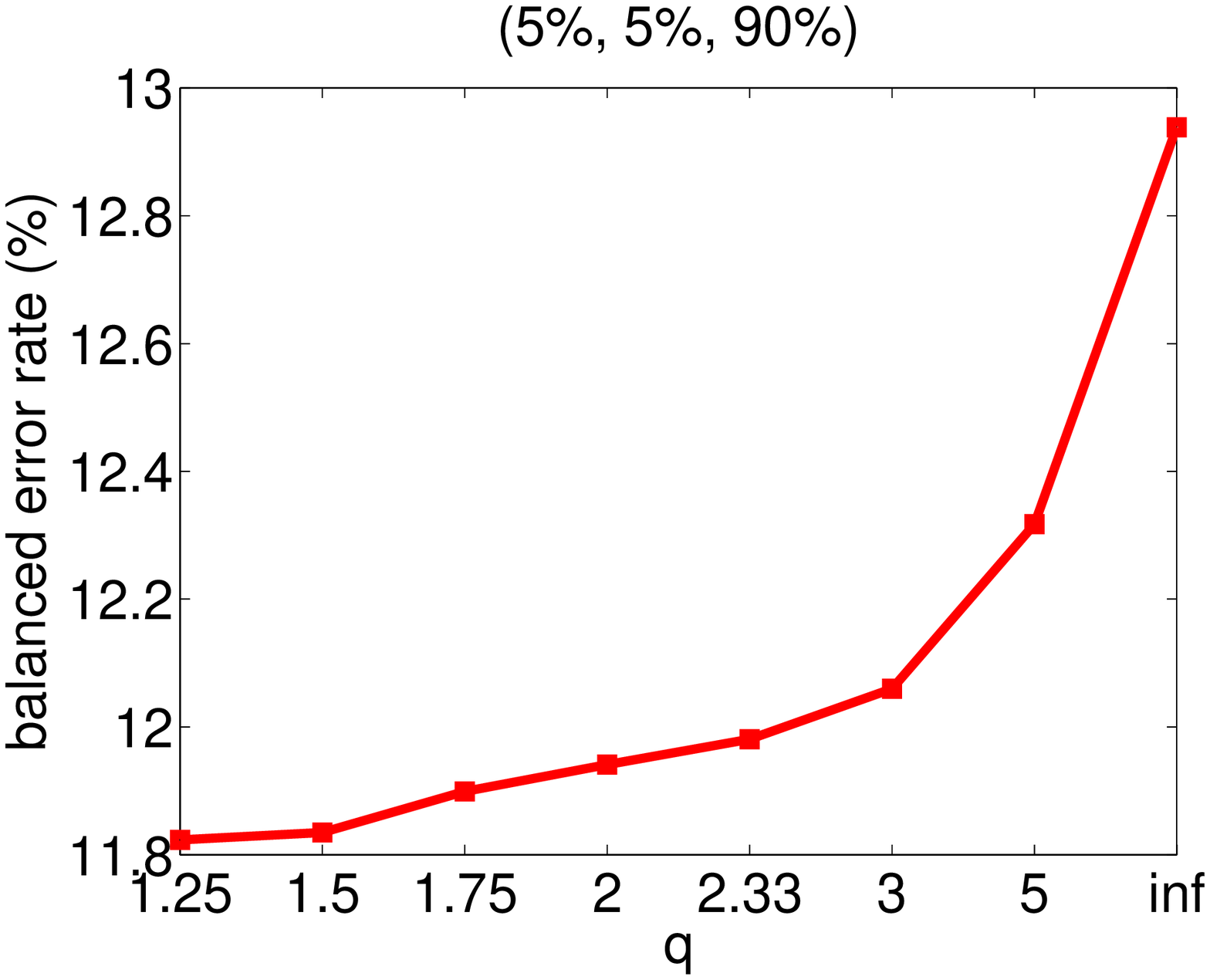}
  \includegraphics[width=2.1in]{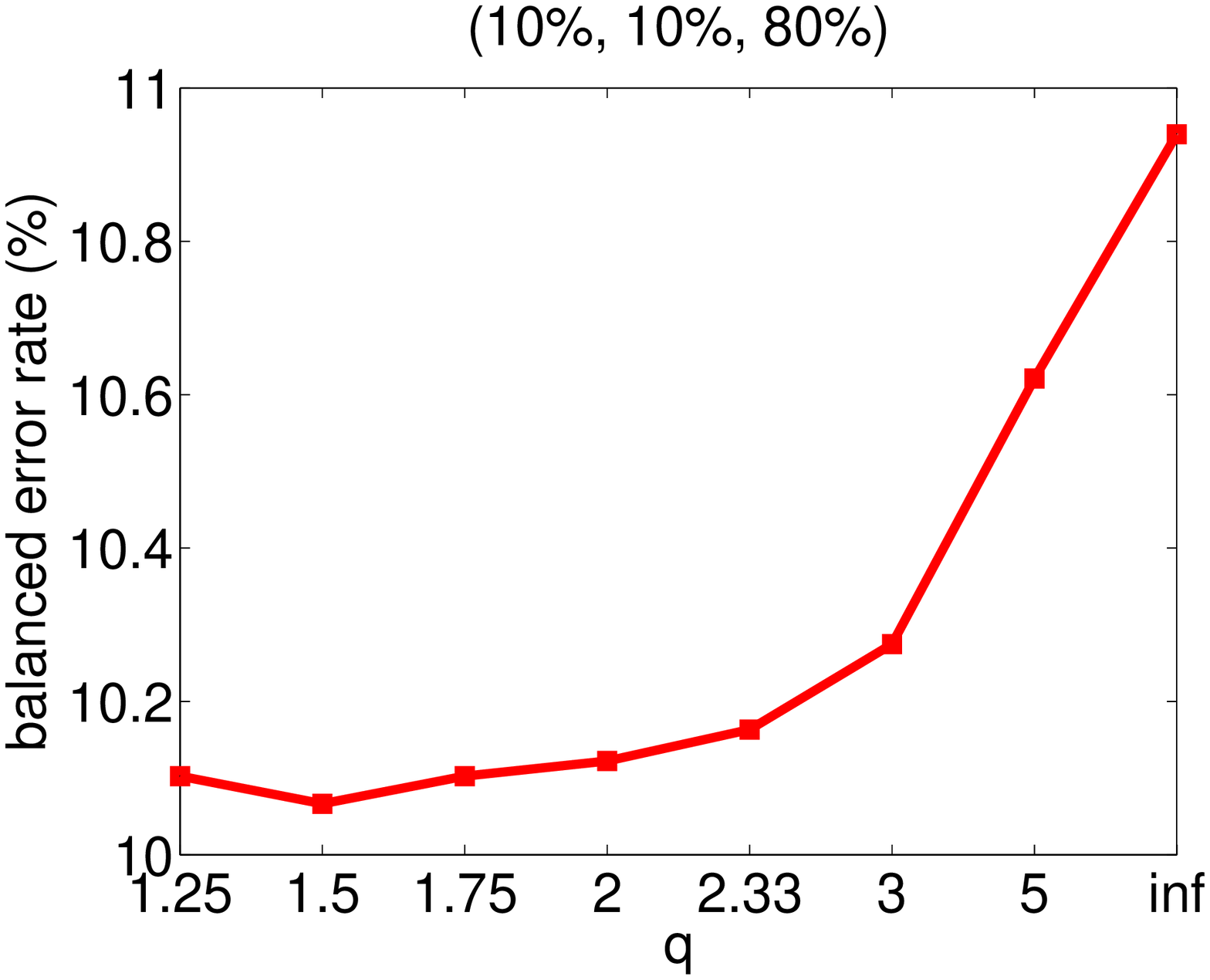}\\
  \caption{{\small
The balanced error rate achieved by the $\ell_1/\ell_q$
regularization under different values of $q$. The title of each plot
indicates the percentages of samples used for training, validation,
and testing.}} \label{fig:letter:performance}
\end{figure}

\subsubsection{Performance under Different Values of $q$ } We
randomly divide the Letter data into three non-overlapping sets:
training, validation, and testing. We train the model using the
training set, and tune the regularization parameter $\lambda=r \times
\lambda_{\max}$ on the validation set, where $r$ is chosen from $\{
10^{-1}, 5\times 10^{-2}, 2 \times 10^{-2}, 1 \times 10^{-2}, 5
\times 10^{-3}, 2 \times 10^{-3}, 1 \times 10^{-3}\}$. On the testing
set, we compute the balanced error rate~\cite{Guyon:2004}. We report
the results averaged over 10 runs in
Figure~\ref{fig:letter:performance}. The title of each plot indicates
the percentages of samples used for training, validation, and
testing. The results show that, on this data set,  a smaller value of
$q$ achieves better performance.

\subsection{Performance of The Proposed Screening Method}\label{subsection:screening}

In this section, we evaluate the proposed screening method (Smin) for solving problem (\ref{prob:primal}) with different values of $q$, $n_g$ and $p$, which correspond to the mixed-norm, the number of groups and the feature dimension, respectively. We compare the performance of Smin$_s$, the sequential DPP rule, and the sequential strong rule in identifying inactive groups of problem (\ref{prob:primal}) along a sequence of $91$ tuning parameter values equally spaced on the scale of $r=\frac{\lambda}{\lambda_{max}}$ from $1.0$ to $0.1$. To measure the performance of the screening methods, we report the rejection ratio, i.e., the ratio between the number of groups discarded by the screening methods and the number of groups with 0 coefficients in the true solution. It is worthwhile to mention that Smin$_s$ and DPP are ``{\it safe}", that is, the discarded groups are guaranteed to have 0 coefficients in the solution. However, strong rules may mistakenly discard groups which have non-zero coefficients in the solution (notice that, DPP can be easily extended to the general $\ell_1/\ell_q$-regularized problems with $q\neq2$ via the techniques developed in this paper).

More importantly, we show that the efficiency of the proposed GLEP$_{1q}$ can be improved by ``{\it three orders of magnitude}" with Smin$_s$.
Specifically, in each experiment, we first apply GLEP$_{1q}$ with warm-start to solve problem (\ref{prob:primal}) along the given sequence of $91$ parameters and report the total running time. Then, for each parameter, we first use the screening methods to discard the inactive groups and then apply GLEP$_{1q}$ to solve problem (\ref{prob:primal}) with the reduced data matrix. We repeat this process until all of the problems have been solved. The total running time (including screening) is reported to demonstrate the effectiveness of the proposed Smin$_s$ in accelerating the computation. Finally, we also report the total running time of the screening methods in discarding the inactive groups for problem (\ref{prob:primal}) with the given sequence of parameters. The entries of the response vector $Y$ and the data matrix $B$  are generated i.i.d. from a standard Gaussian distribution, and the correlation between $Y$ and the columns of $B$ is uniformly distributed on the interval $[-0.8,0.8]$. For each experiment, we repeat the computation $20$ times and report the average results.

\subsubsection{Efficiency with Different Values of $q$}\label{sssection:diff_q}

In this experiment, we evaluate the performance of the proposed Smin$_s$ and demonstrate its effectiveness in accelerating the computation of GLEP$_{1q}$ with different values of $q$. The size of the data matrix $B$ is fixed to be $1000\times 10000$, and we randomly divide $B$ into $1000$ non-overlapping groups.

\begin{figure}[t]
\centering{
\subfigure[$q=1$] { \label{fig:q1e1}
\includegraphics[width=0.31\columnwidth]{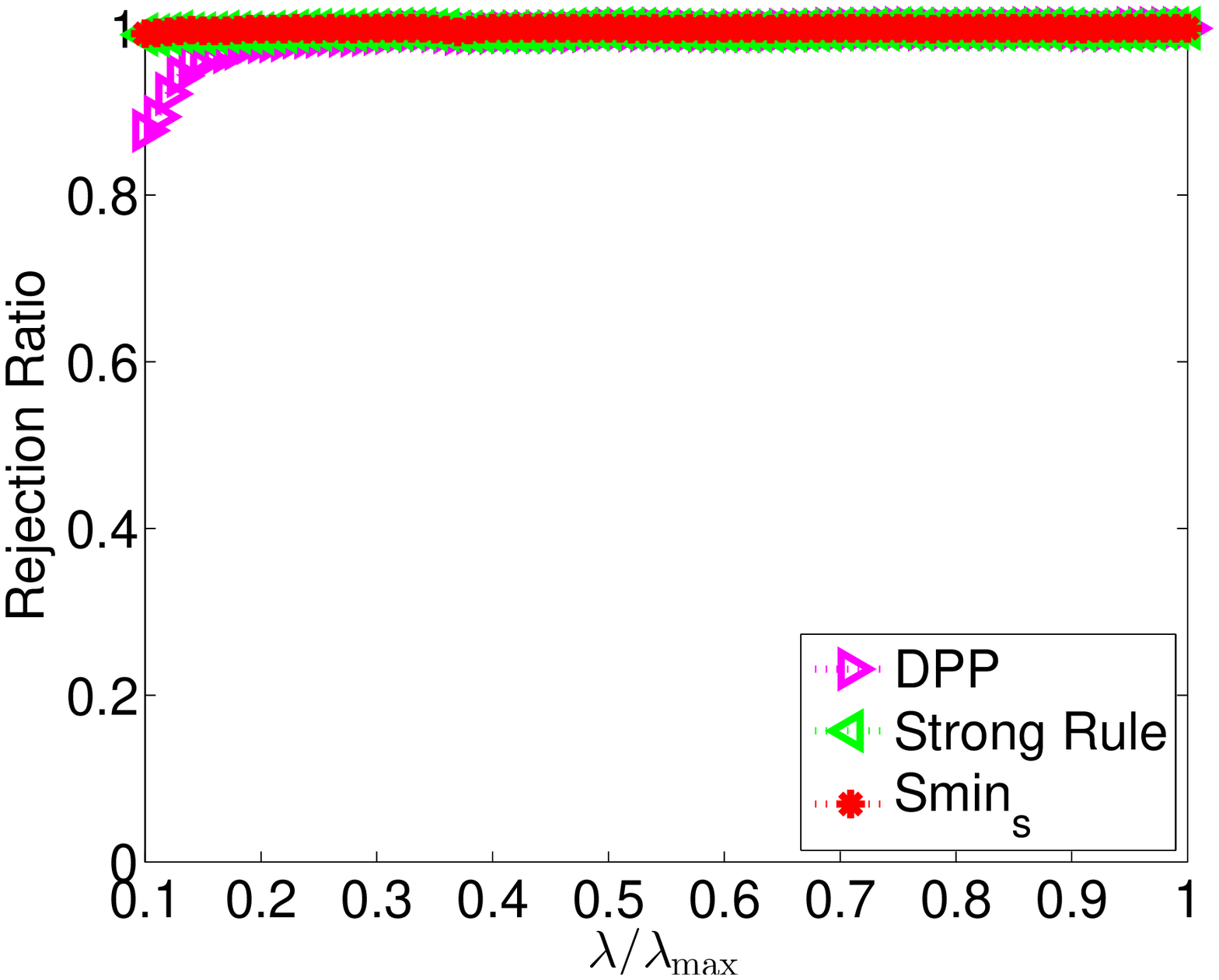}
}
\subfigure[$q=1.25$] { \label{fig:q125e1}
\includegraphics[width=0.31\columnwidth]{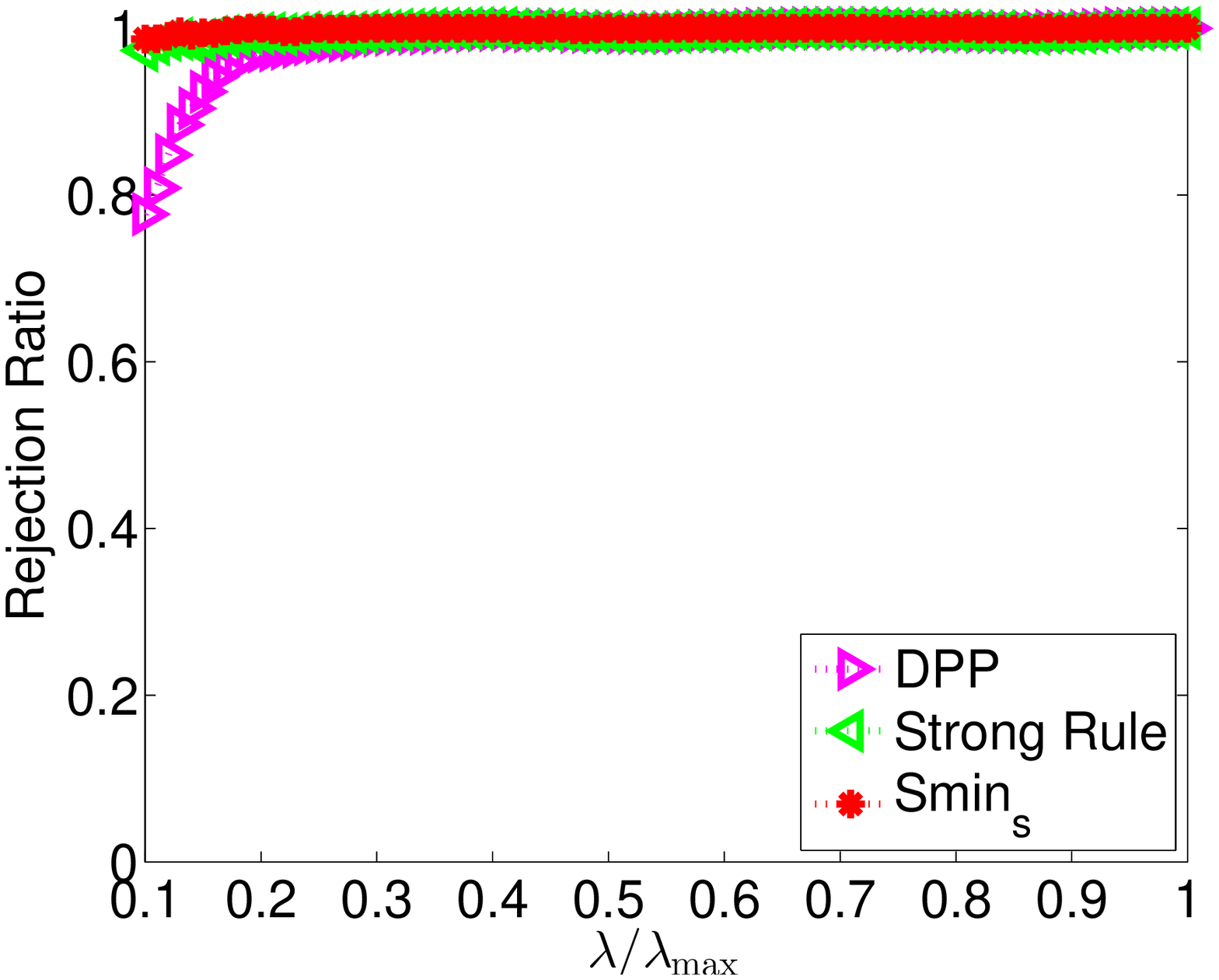}
}
\subfigure[$q=1.5$] { \label{fig:q15e1}
\includegraphics[width=0.31\columnwidth]{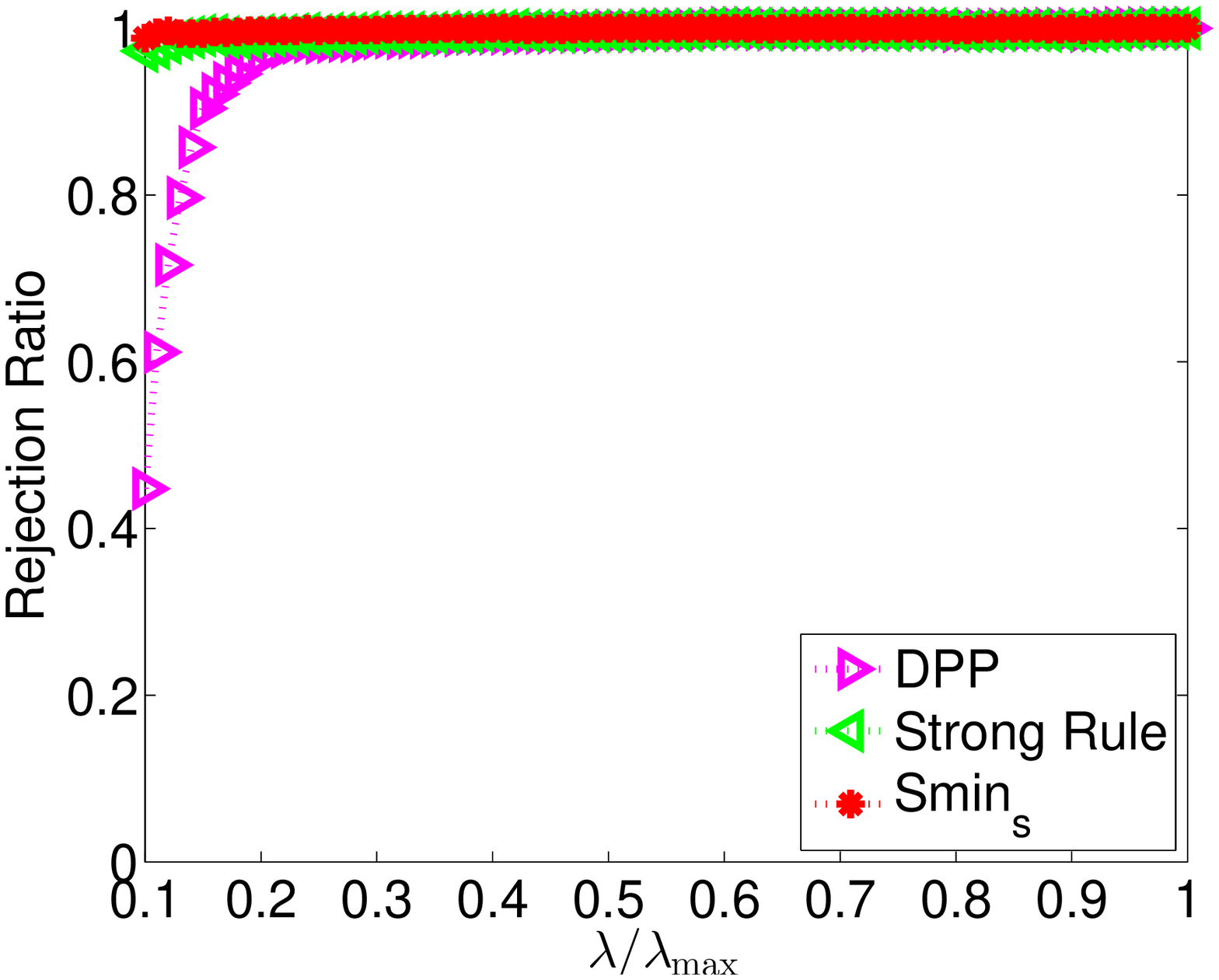}
}\\
\subfigure[$q=1.75$] { \label{fig:q175e1}
\includegraphics[width=0.31\columnwidth]{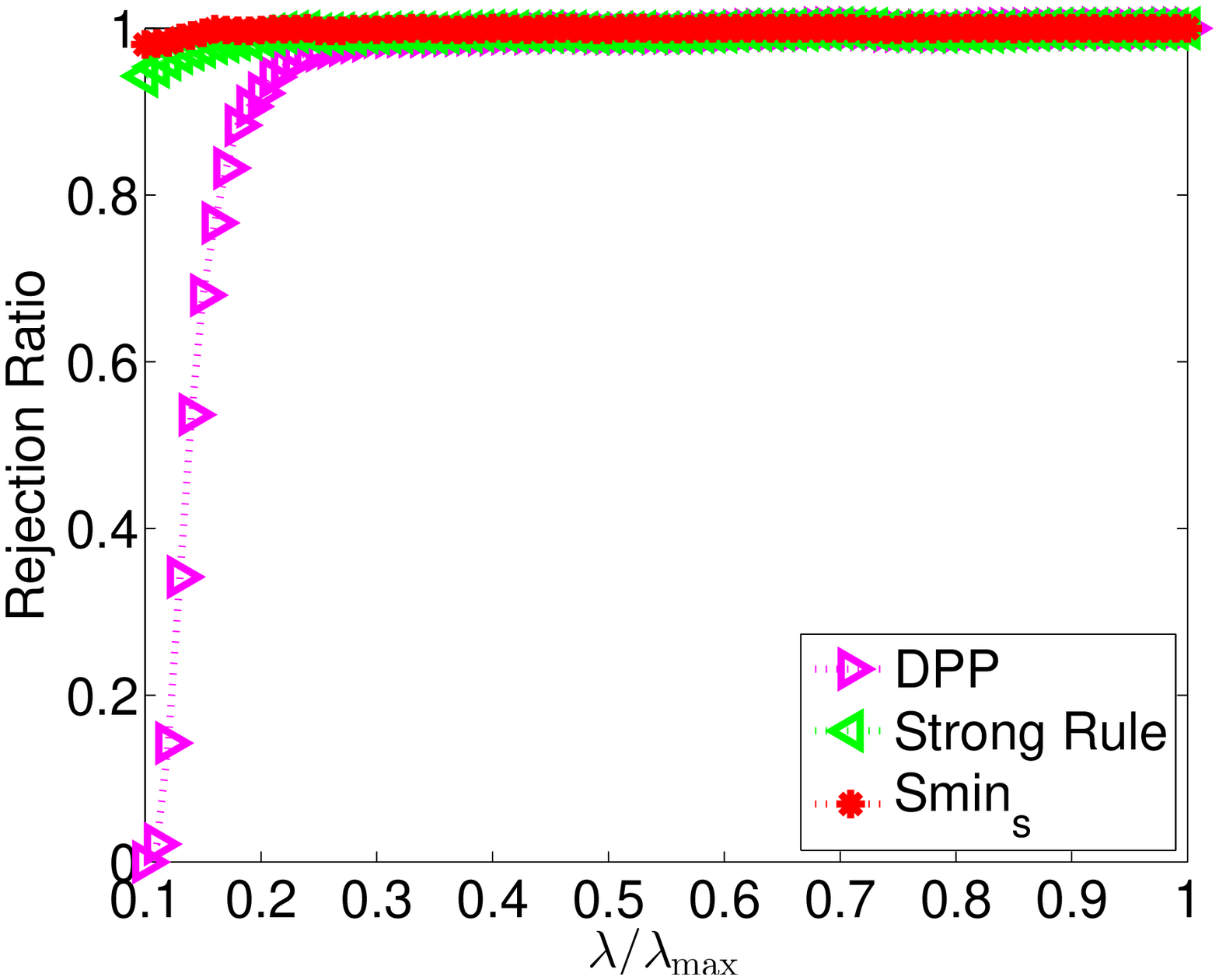}
}
\subfigure[$q=2$] { \label{fig:q2e1}
\includegraphics[width=0.31\columnwidth]{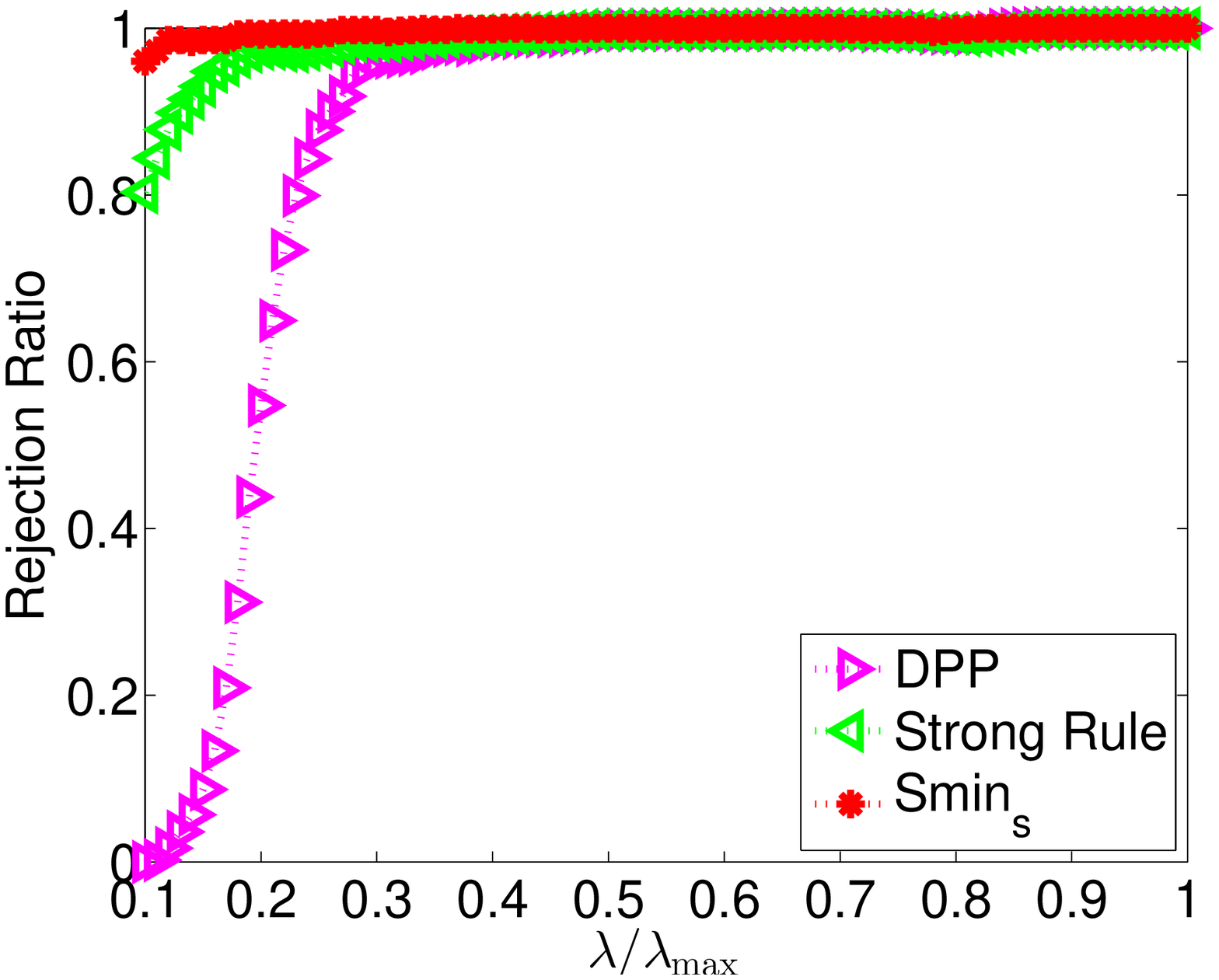}
}
\subfigure[$q=2.33$] { \label{fig:q233e1}
\includegraphics[width=0.31\columnwidth]{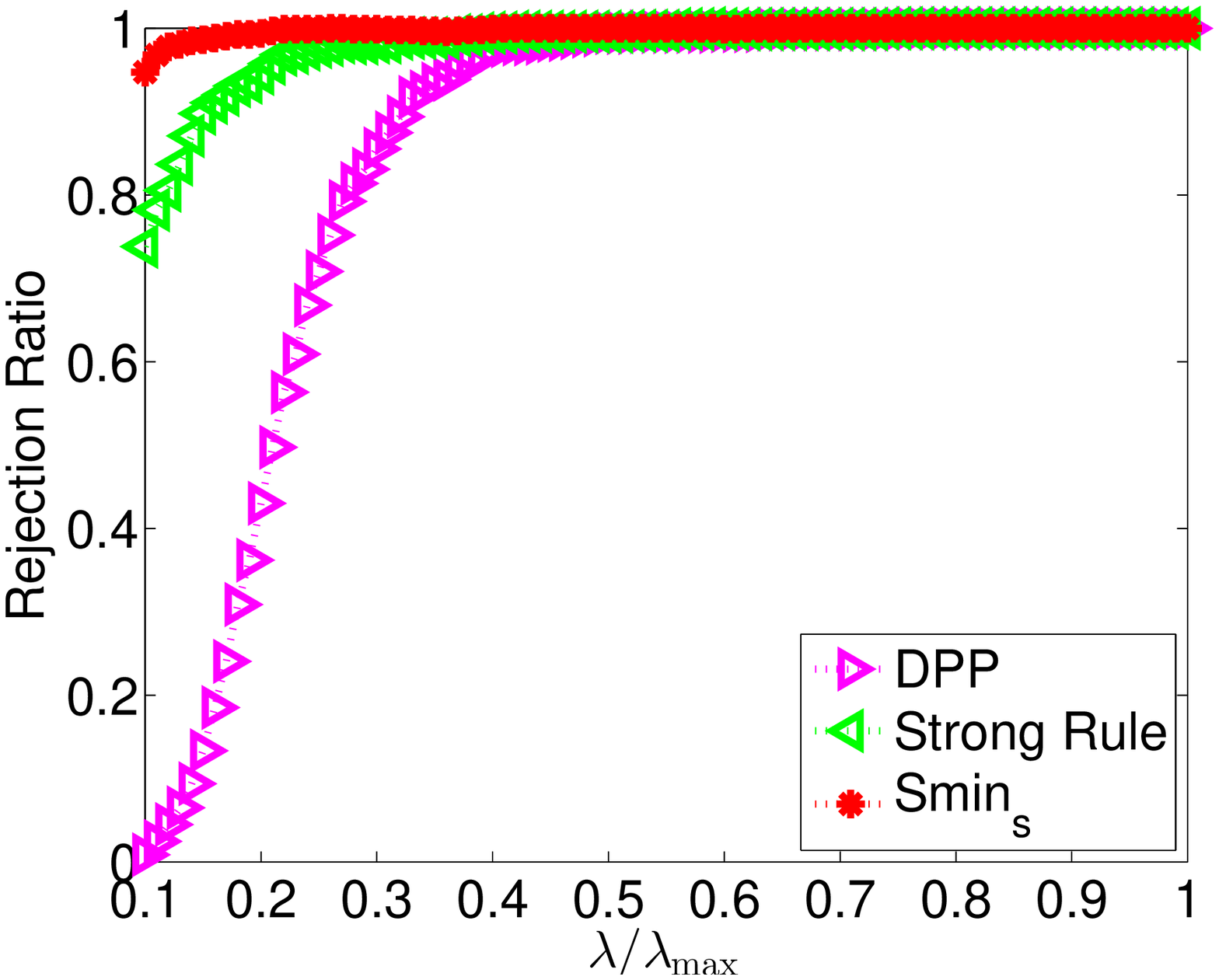}
}\\
\subfigure[$q=3$] { \label{fig:q3e1}
\includegraphics[width=0.31\columnwidth]{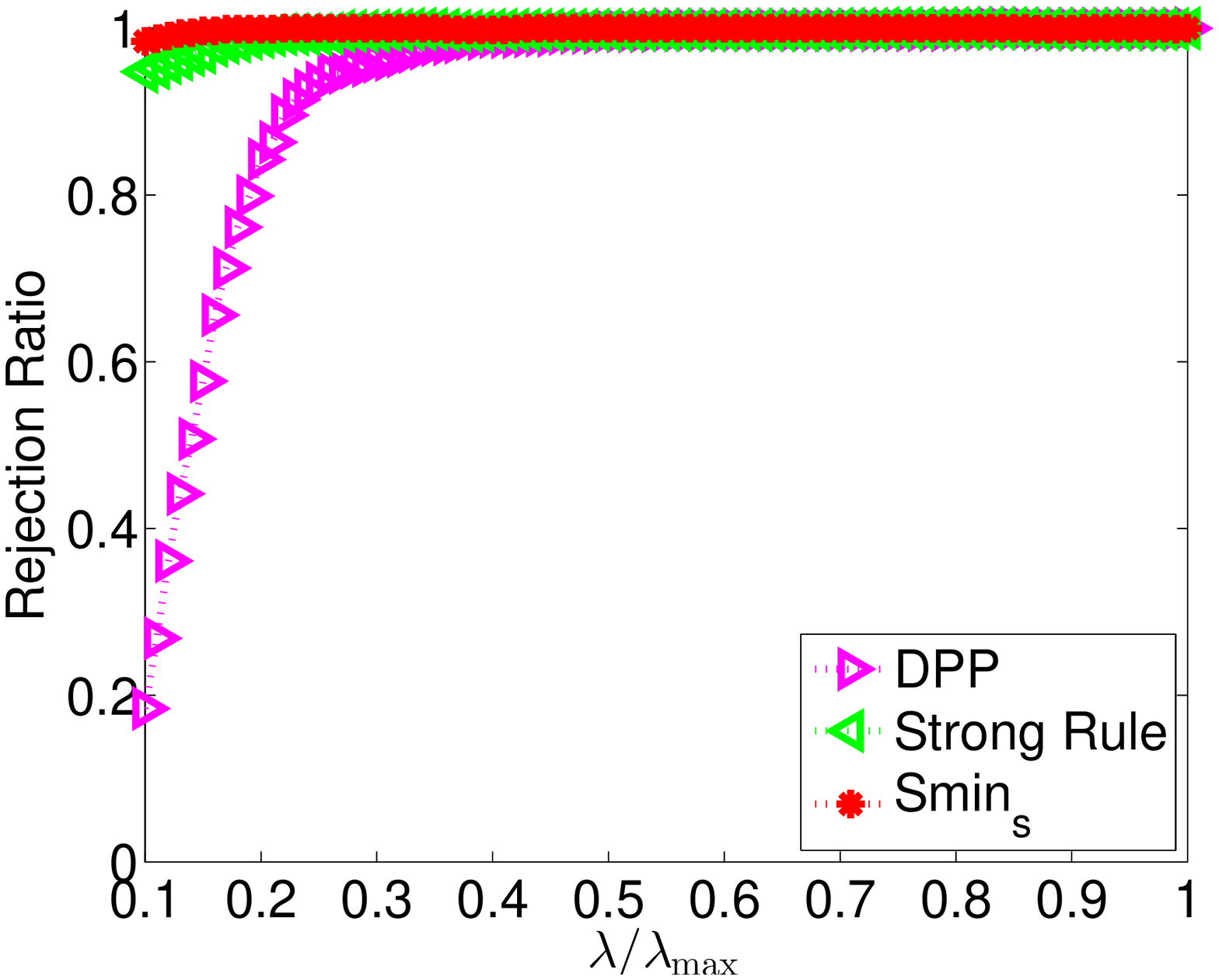}
}
\subfigure[$q=5$] { \label{fig:q5e1}
\includegraphics[width=0.31\columnwidth]{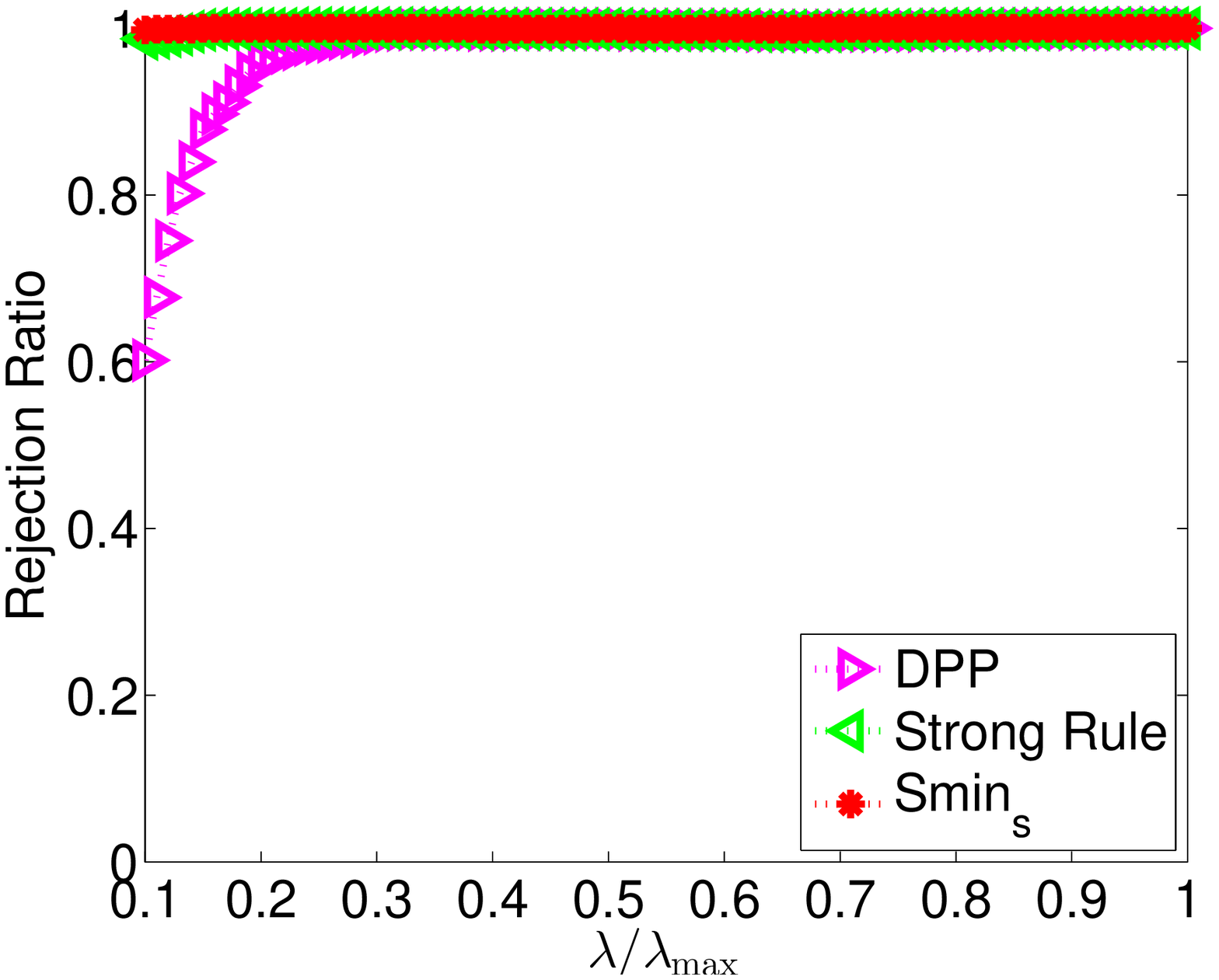}
}
\subfigure[$q=\infty$] { \label{fig:qinfe1}
\includegraphics[width=0.31\columnwidth]{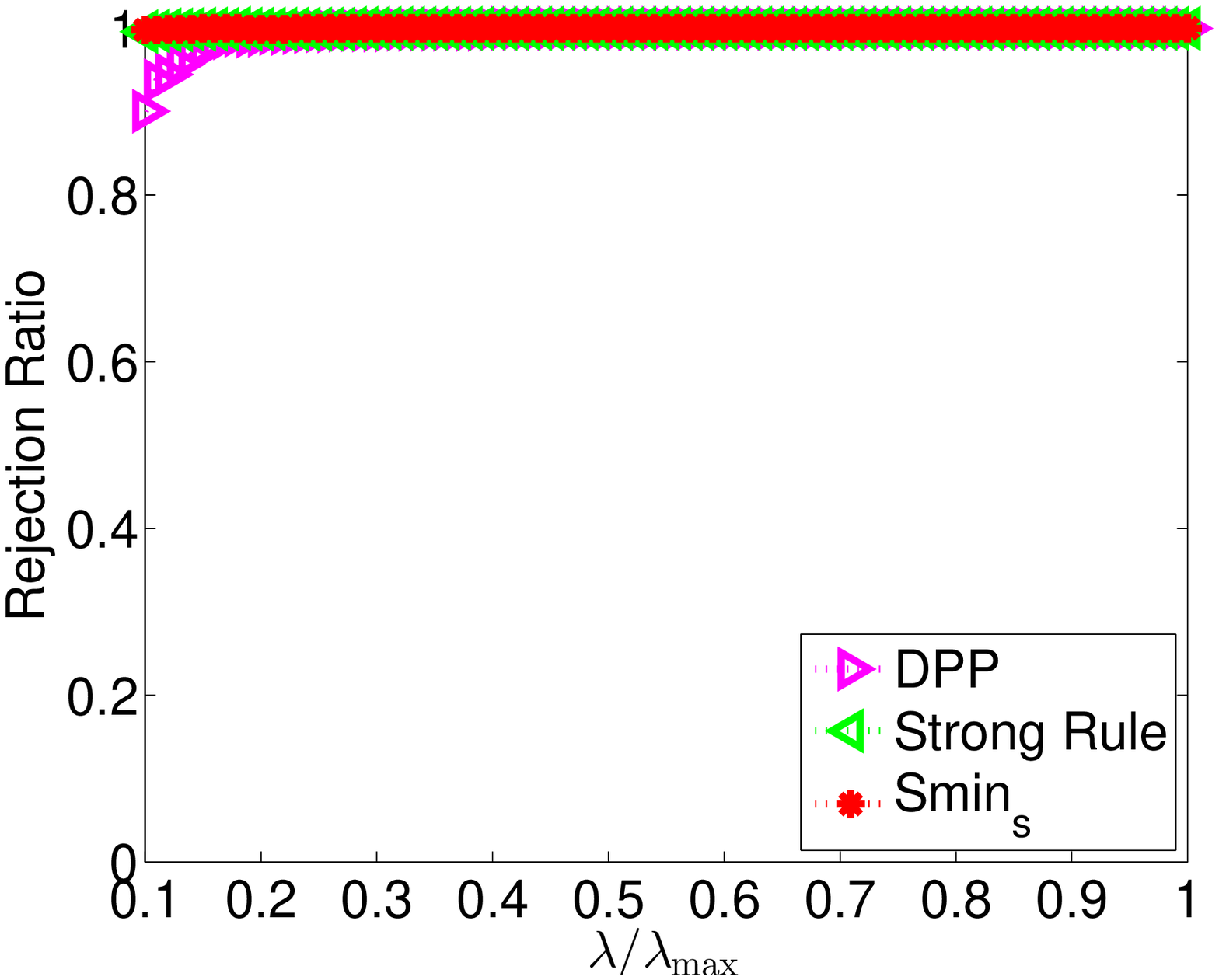}
}
}%\\[-0.4cm]
\caption{Comparison of Smin$_s$, DPP and strong rules with different values of $q$: $B\in\mathbb{R}^{1000\times10000}$ and $n_g=1000$.}
\label{fig:rej_ratio_e1}
\end{figure}

Figure \ref{fig:rej_ratio_e1} presents the rejection ratios of DPP, strong rule and Smin$_s$ with $9$ different values of $q$. We can see that the performance of Smin$_s$ is very robust to the variation of $q$. The rejection ratio of Smin$_s$ is almost $100\%$ along the sequence of $91$ parameters for all different $q$, which implies that almost all of the inactive groups are discarded by Smin$_s$. As a result, the number of groups that need to be entered into the optimization is significantly reduced. The performance of the strong rule is less robust to different values of $q$ than Smin$_s$, and DPP is more sensitive than the other two. Figure \ref{fig:rej_ratio_e1} indicates that the performance of Smin$_s$ significantly outperforms that of DPP and the strong rule.

\begin{table}
\begin{center}
\begin{footnotesize}
\begin{tabular}{ |l||c||c|c|c||c|c|c| }
  \hline
  $q$ & GLEP$_{1q}$ & GLEP$_{1q}$+DPP & GLEP$_{1q}$+SR & GLEP$_{1q}$+Smin$_s$ & DPP & Strong Rule & Smin$_s$  \\
  \hline\hline
  $1$  & 284.89 & 3.99 & 3.63 & 2.35 & 1.07 & 1.98 & 1.02 \\\hline
  $1.25$  & 338.76 & 5.50 & 5.18 & 3.65 & 1.05 & 2.03 & 1.06 \\\hline
  $1.5$  & 380.02 & 11.58 & 5.58 & 4.19 & 1.12 & 2.04 & 1.08 \\\hline
  $1.75$  & 502.07 & 50.27 & 7.53 & 5.83 & 1.21 & 1.97 & 1.08 \\\hline
 $2$  & 484.11 & 121.90 & 10.64 & 4.94 & 1.54 & 1.88 & 1.01 \\ \hline
 $2.33$  & 809.42 & 266.60 & 89.52 & 47.72 & 1.72 & 2.33 & 1.08\\\hline
 $3$  & 959.83 & 220.81 & 105.61 & 92.14 & 1.52 & 2.05 & 1.12 \\\hline
 $5$  & 1157.69 & 185.55 & 131.85 & 128.77 & 1.24 & 2.04 & 1.06 \\\hline
 $\infty$  & 221.83 & 9.35 & 3.62 & 2.69 & 1.01 & 1.81 & 0.99 \\
  \hline
\end{tabular}
\end{footnotesize}
\end{center}
\vspace{0.1in}
\caption{Running time (in seconds) for solving the mixed-norm regularized problems along a sequence of $91$ tuning parameter values equally spaced on the scale of $r=\frac{\lambda}{\lambda_{max}}$ from $1.0$ to $0.1$ by (a): GLEP$_{1q}$ (reported in the second column) without screening; (b): GLEP$_{1q}$ combined with different screening methods (reported in the third to the fifth columns).
The last three columns report the total running time (in seconds) for the screening methods. The data matrix $B$ is of size $1000\times10000$ and $q=2$.
}
\label{table:time_diff_q}
\end{table}

Table \ref{table:time_diff_q} shows the total running time for solving problem (\ref{prob:primal}) along the sequence of $91$ values of $r=\lambda/\lambda_{max}$ by GLEP$_{1q}$ and GLEP$_{1q}$ with screening methods. We also report the running time for the screening methods. We can see that the running times of GLEP$_{1q}$ combined with Smin$_s$ are only about $1\%$ of the running times of GLEP$_{1q}$ without screening when $q=1,1.25,1.5,1.75,2,\infty$. In other words, Smin$_s$ improves the efficiency of GLEP$_{1q}$ by about $100$ times. When $q=2.33,3,5$, the efficiency of GLEP$_{1q}$ is improved by more than $10$ times. The last three columns of Table \ref{table:time_diff_q} also indicate that Smin$_s$ is the most efficient screening method among DPP and the strong rule in terms of the running time. Moreover, we can see that the running time of the strong rule is about twice as long as the other two methods. The reason is because the strong rule needs to check its screening results \cite{tibshirani2012} by verifying the KKT conditions each time.

%\begin{table}
%\begin{center}
%\begin{footnotesize}
%\begin{tabular}{ |l|c|c|c|c|c|c|c| }
%  \hline
%  & GLEP$_{1q}$ & GLEP$_{1q}$+DPP & GLEP$_{1q}$+SR & GLEP$_{1q}$+Smin$_s$ & DPP & Strong Rule & Smin$_s$  \\
%  \hline
% $q=2$  & 1930.84 & 563.14 & 49.61 & 13.53 & 6.80 & 8.76 & 4.56 \\
%  \hline
%\end{tabular}
%\end{footnotesize}
%\end{center}
%\caption{Running time (in seconds) for solving the mixed-norm regularized problems by (a): GLEP$_{1q}$ (reported in the first column); (b): GLEP$_{1q}$ combined with different screening methods (reported in the second to the fourth columns), that is, we first use the screening methods to discard the inactive groups which have 0 coefficients in the solution and then apply GLEP$_{1q}$ to solve the reduced size problems. The last three columns report the running time (in seconds) for the screening methods themselves. Entries of the response vector $Y\in\mathbb{R}^{5000}$ and the data matrix $B\in \Re^{5000\times 10000}$ are i.i.d. standard Gaussian. The correlation between $Y$ and the columns of $B$ is uniformly distributed on the interval [-0.8,0.8].}
%\label{table:time_diff_q1}
%\end{table}

\subsubsection{Efficiency with Different Numbers of Groups}\label{sssection:diff_ng}

In this experiment, we evaluate the performance of Smin$_s$ with different numbers of groups. The data matrix $B$ is fixed to be $1000\times10000$ and we set $q=2$. Recall that $n_g$ denotes the number of groups. Therefore, let $s_g$ be the average group size. For example, if $n_g$ is $100$, then $s_g=p/n_g=100$.

\begin{figure}[t]
\centering{
\subfigure[$n_g=100$] { \label{fig:ng100e2}
\includegraphics[width=0.31\columnwidth]{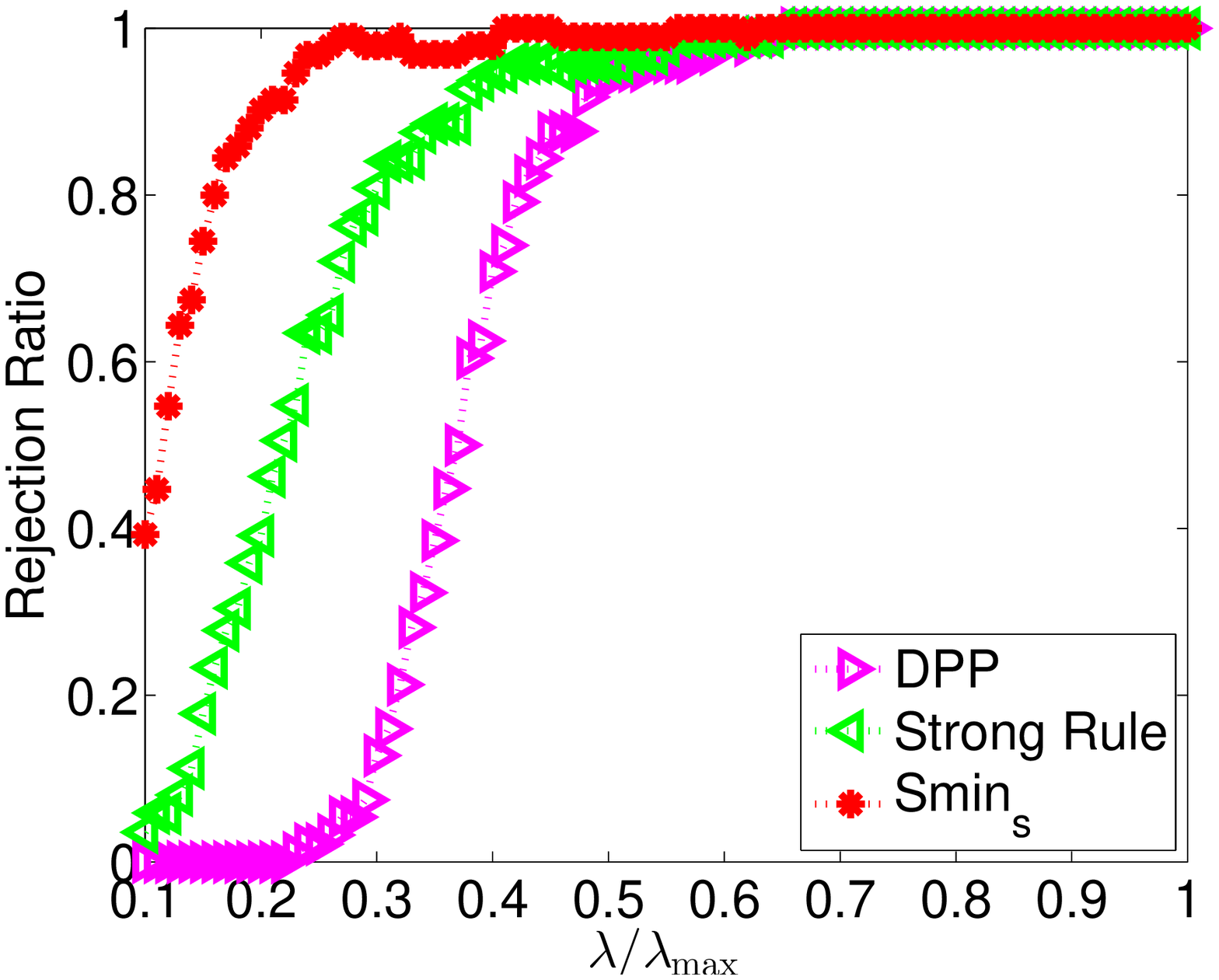}
}
\subfigure[$n_g=400$] { \label{fig:ng400e2}
\includegraphics[width=0.31\columnwidth]{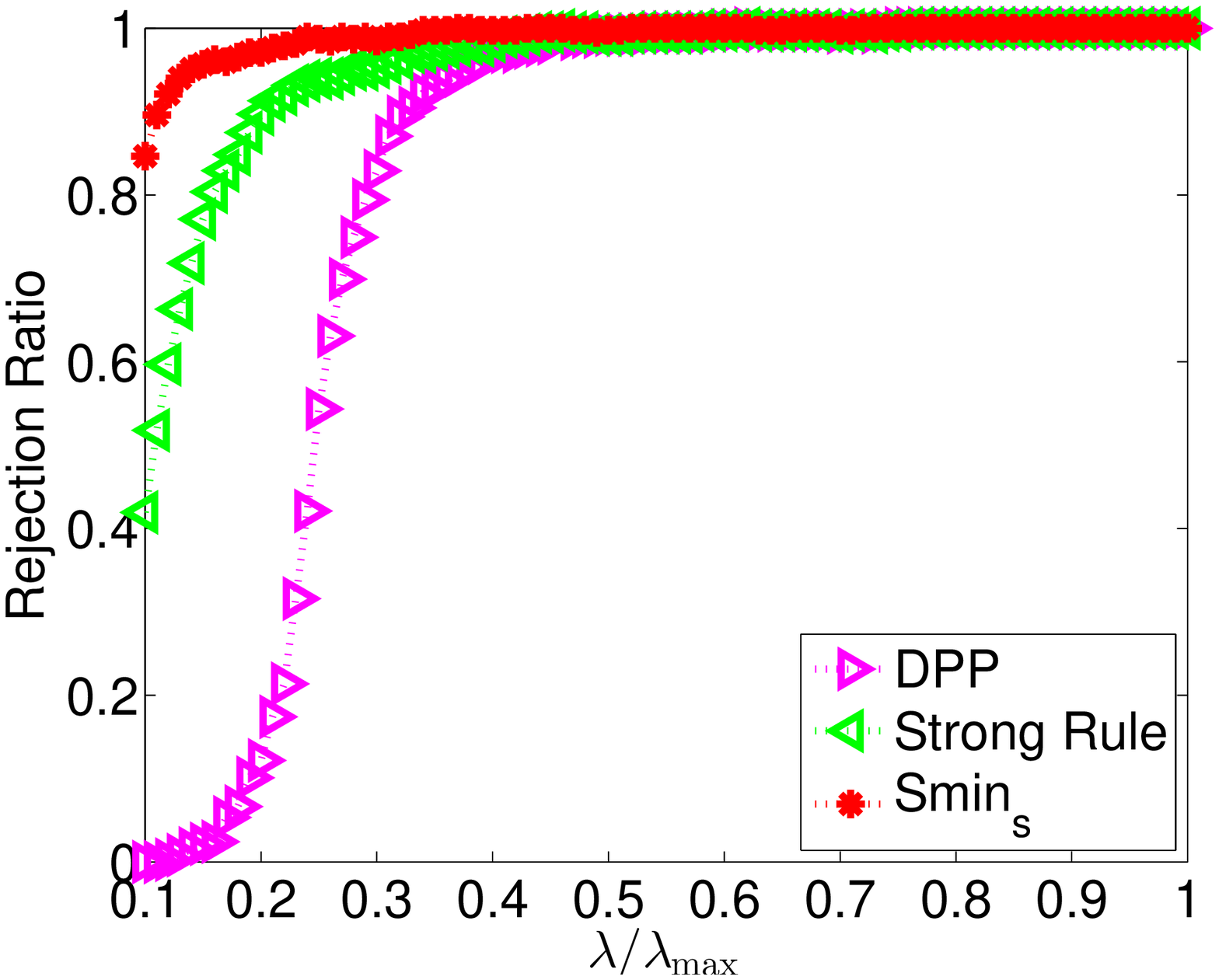}
}
\subfigure[$n_g=800$] { \label{fig:ng800e2}
\includegraphics[width=0.31\columnwidth]{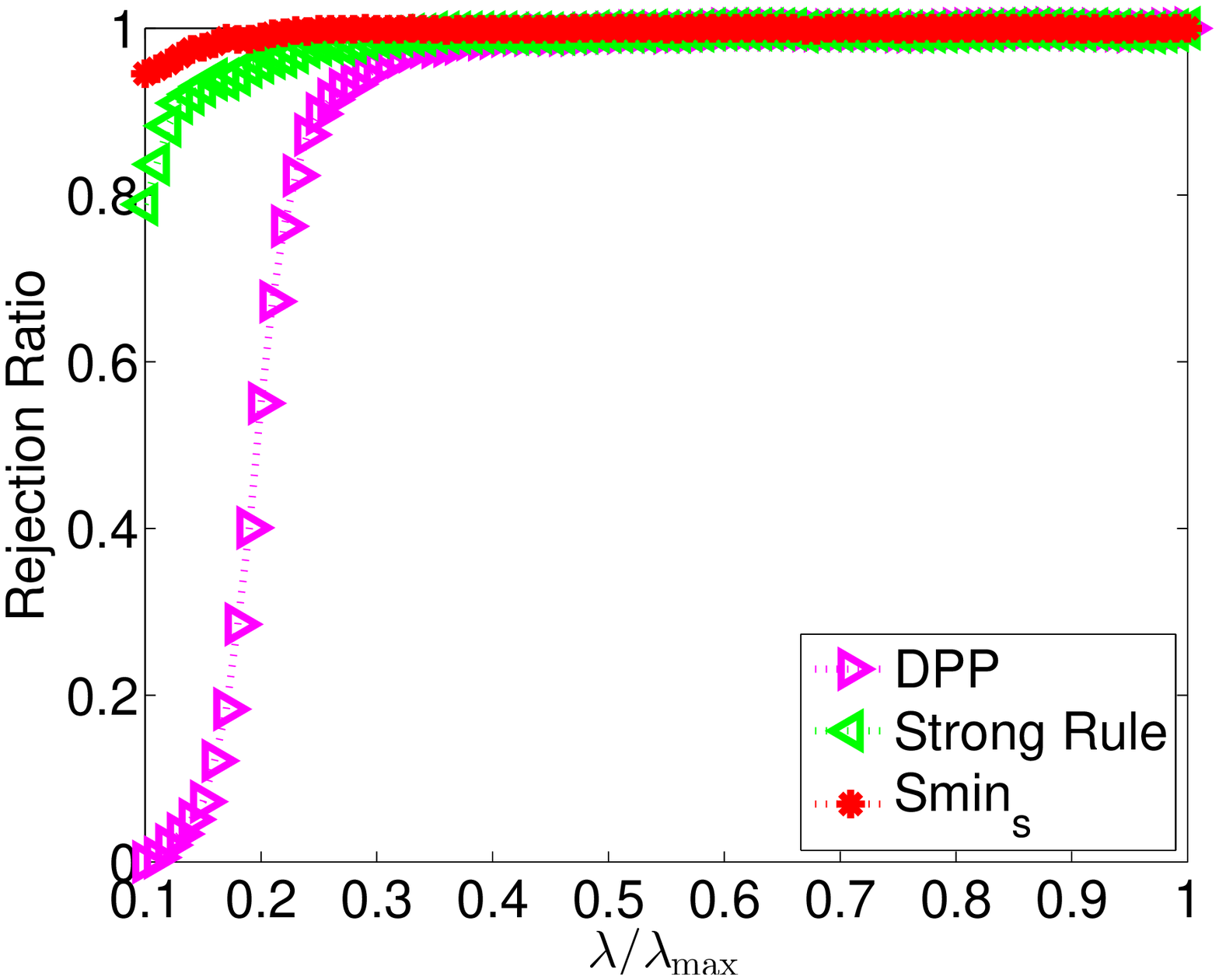}
}\\
\subfigure[$n_g=1200$] { \label{fig:ng1200e2}
\includegraphics[width=0.31\columnwidth]{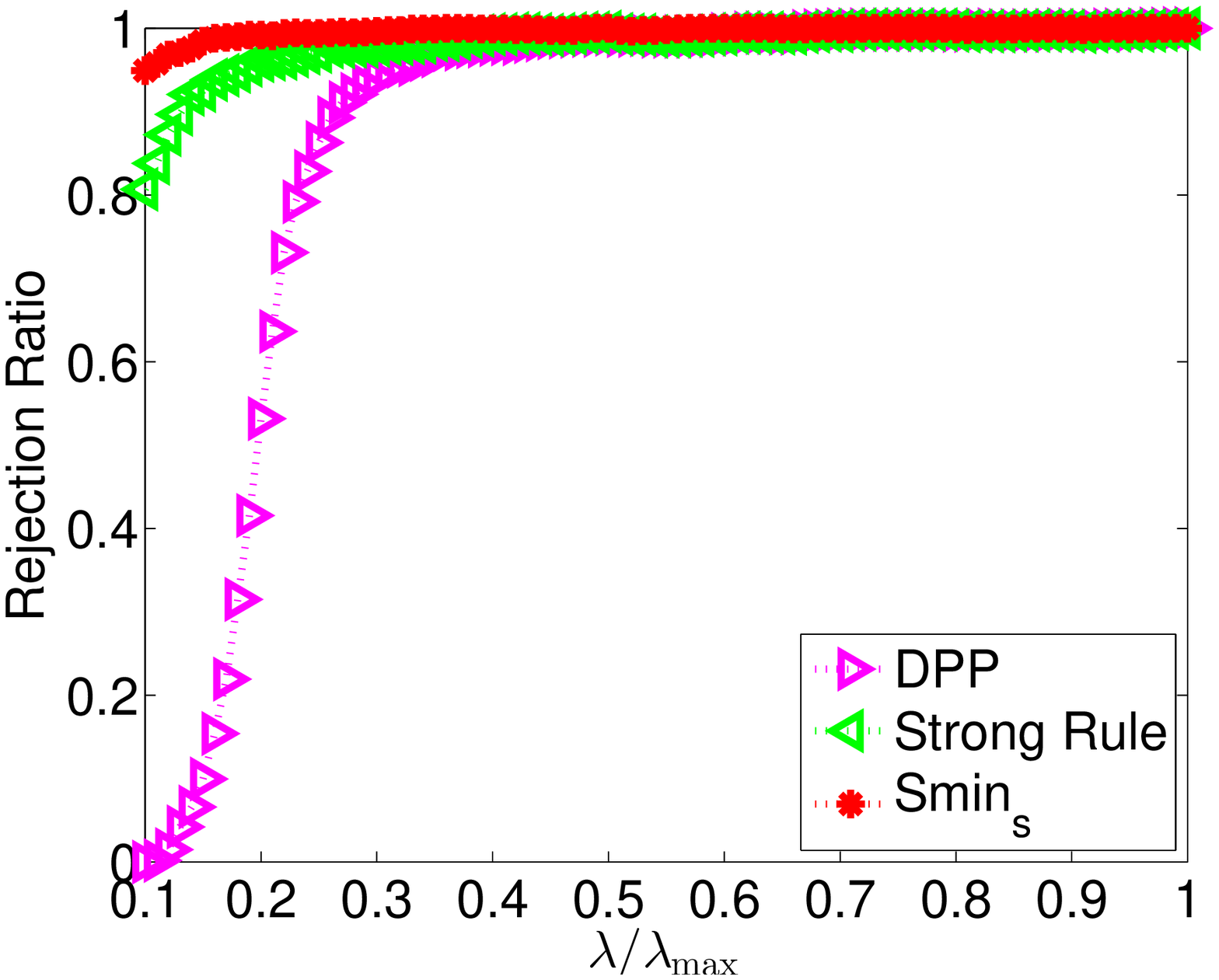}
}
\subfigure[$n_g=1600$] { \label{fig:ng1600e2}
\includegraphics[width=0.31\columnwidth]{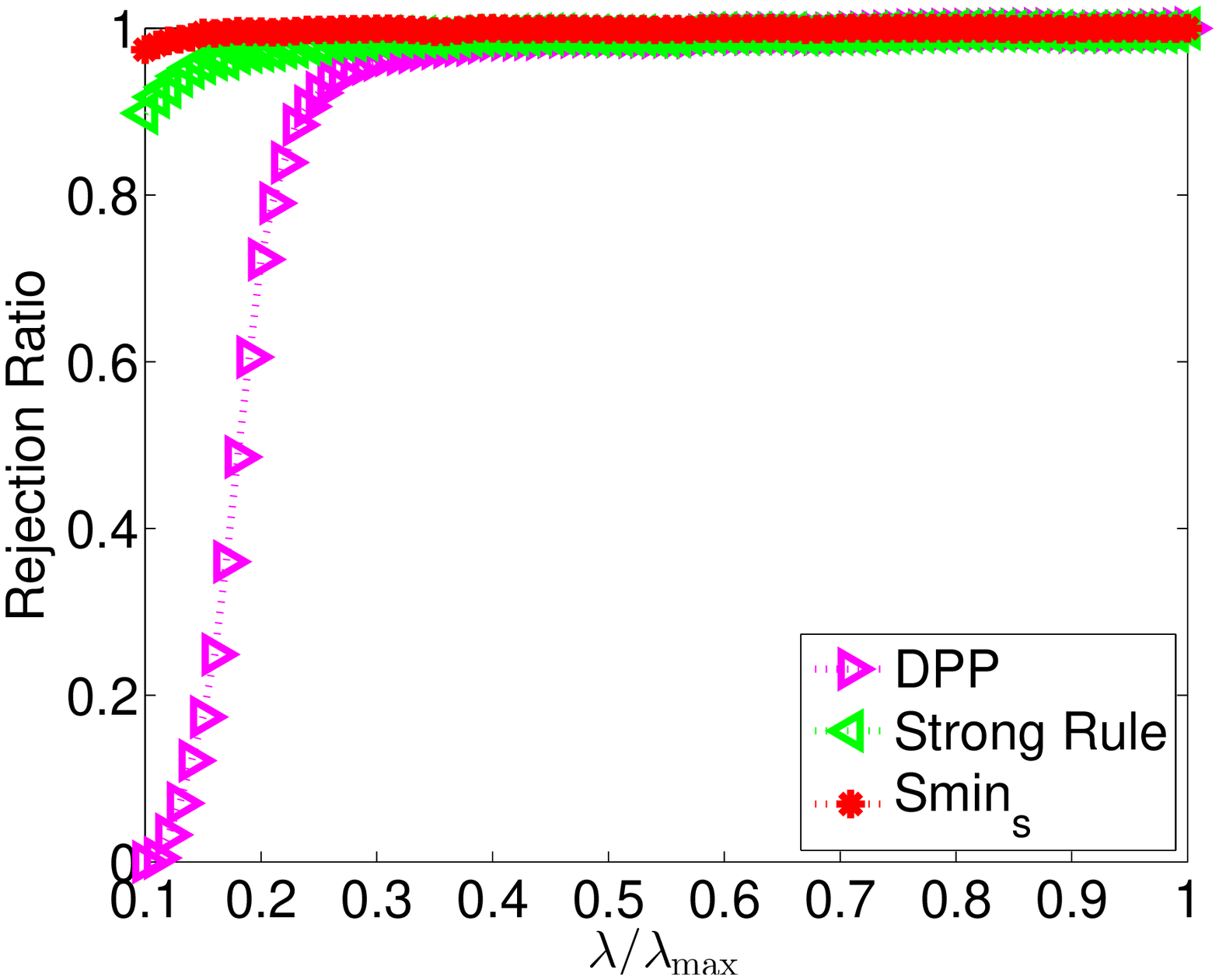}
}
\subfigure[$n_g=2000$] { \label{fig:ng2000e2}
\includegraphics[width=0.31\columnwidth]{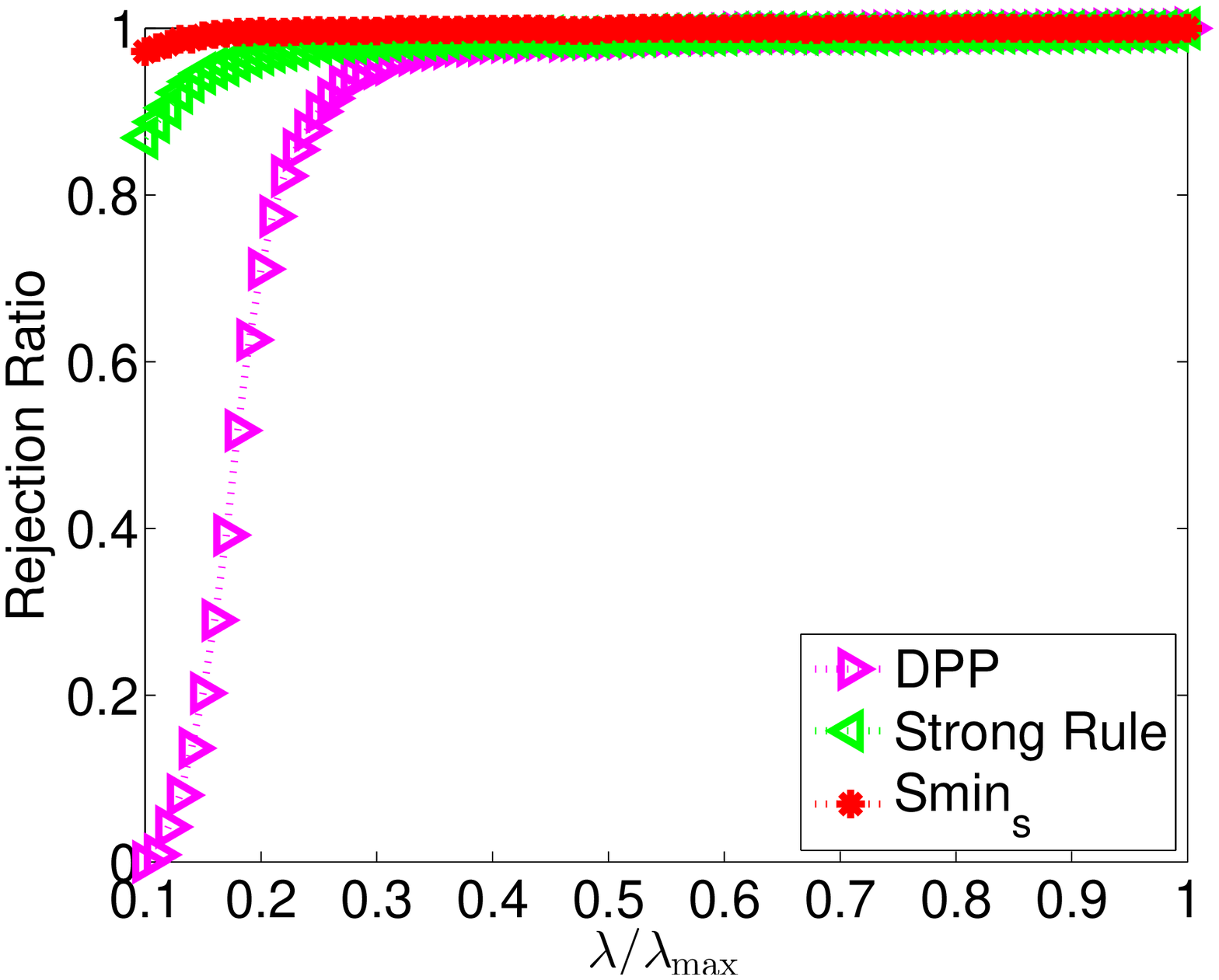}
}
}%\\[-0.4cm]
\caption{Comparison of Smin$_s$, DPP and strong rules with different numbers of groups: $B\in\mathbb{R}^{1000\times10000}$ and $q=2$.}
\label{fig:rej_ratio_e2}
\end{figure}

From Figure \ref{fig:rej_ratio_e2}, we can see that the screening methods, i.e., Smin$_s$, DPP and strong rule, are able to discard more inactive groups when the number of groups $n_g$ increases. The reason is due to the fact that the estimation of the region $\mathcal{R}$ [please refer to \eqref{def:ball} and Section \ref{subsection:region}] is more accurate with smaller group size. Notice that, a large $n_g$ implies a small average group size. Figure \ref{fig:rej_ratio_e2} implies that compared with DPP and the strong rule, Smin$_s$ is able to discard more inactive groups and more robust with respect to different values of $n_g$.

\begin{table}
\begin{center}
\begin{footnotesize}
\begin{tabular}{ |l||c||c|c|c||c|c|c| }
  \hline
  $n_g$ & GLEP$_{1q}$ & GLEP$_{1q}$+DPP & GLEP$_{1q}$+SR & GLEP$_{1q}$+Smin$_s$ & DPP & Strong Rule & Smin$_s$  \\
  \hline\hline
  $100$  & 421.62 & 275.94 & 140.52 & 41.91 & 1.07 & 2.39 & 1.10 \\\hline
  $400$  & 364.61 & 168.04 & 34.14 & 6.12 & 1.55 & 1.98 & 1.01 \\\hline
  $800$  & 410.65 & 128.33 & 9.20 & 5.03 & 1.44 & 2.02 & 1.03 \\\hline
  $1200$  & 627.85 & 136.68 & 12.33 & 6.84 & 1.56 & 1.99 & 1.13 \\\hline
 $1600$  & 741.25 & 115.77 & 10.69 & 7.10 & 1.42 & 1.91 & 1.01 \\\hline
 $2000$  & 786.89 & 130.74 & 12.54 & 7.24 & 1.57 & 2.19 & 1.09\\\hline
 %$q=3$  & 959.83 & 220.81 & 105.61 & 92.14 & 1.52 & 2.05 & 1.12 \\\hline
% $q=5$  & 1157.69 & 185.55 & 131.85 & 128.77 & 1.24 & 2.04 & 1.06 \\\hline
% $q=\infty$  & 221.83 & 9.35 & 3.62 & 2.69 & 1.01 & 1.81 & 0.99 \\\hline
\end{tabular}
\end{footnotesize}
\end{center}
\vspace{0.1in}
\caption{Running time (in seconds) for solving the mixed-norm regularized problems along a sequence of $91$ tuning parameter values equally spaced on the scale of $r=\frac{\lambda}{\lambda_{max}}$ from $1.0$ to $0.1$ by (a): GLEP$_{1q}$ (reported in the second column) without screening; (b): GLEP$_{1q}$ combined with different screening methods (reported in the third to the fifth columns).
The last three columns report the total running time (in seconds) for the screening methods. The data matrix $B$ is of size $1000\times10000$ and $q=2$.
}
\label{table:time_diff_ng}
\end{table}

Table \ref{table:time_diff_ng} further demonstrates the effectiveness of Smin$_s$ in improving the efficiency of GLEP$_{1q}$. When $n_g=100,400$, the efficiency of GLEP$_{1q}$ is improved by $10$ and $50$ times respectively. For the other cases, the efficiency of GLEP$_{1q}$ is boosted by about $100$ times with Smin$_s$.

\subsubsection{Efficiency with Different Dimensions}\label{sssection:diff_p}

In this experiment, we apply GLEP$_{1q}$ with screening methods to large dimensional data sets. We generate the data sets $B\in\mathbb{R}^{1000\times p}$, where $p=20000, 50000, 80000, 100000, 150000, 200000$. Notice that, the data matrix $B$ is a dense matrix. We set $q=2$ and the average group size $s_g=10$.

\begin{figure}[t!]
\centering{
\subfigure[$p=20000$] { \label{fig:q2_2e4_c08_e3}
\includegraphics[width=0.31\columnwidth]{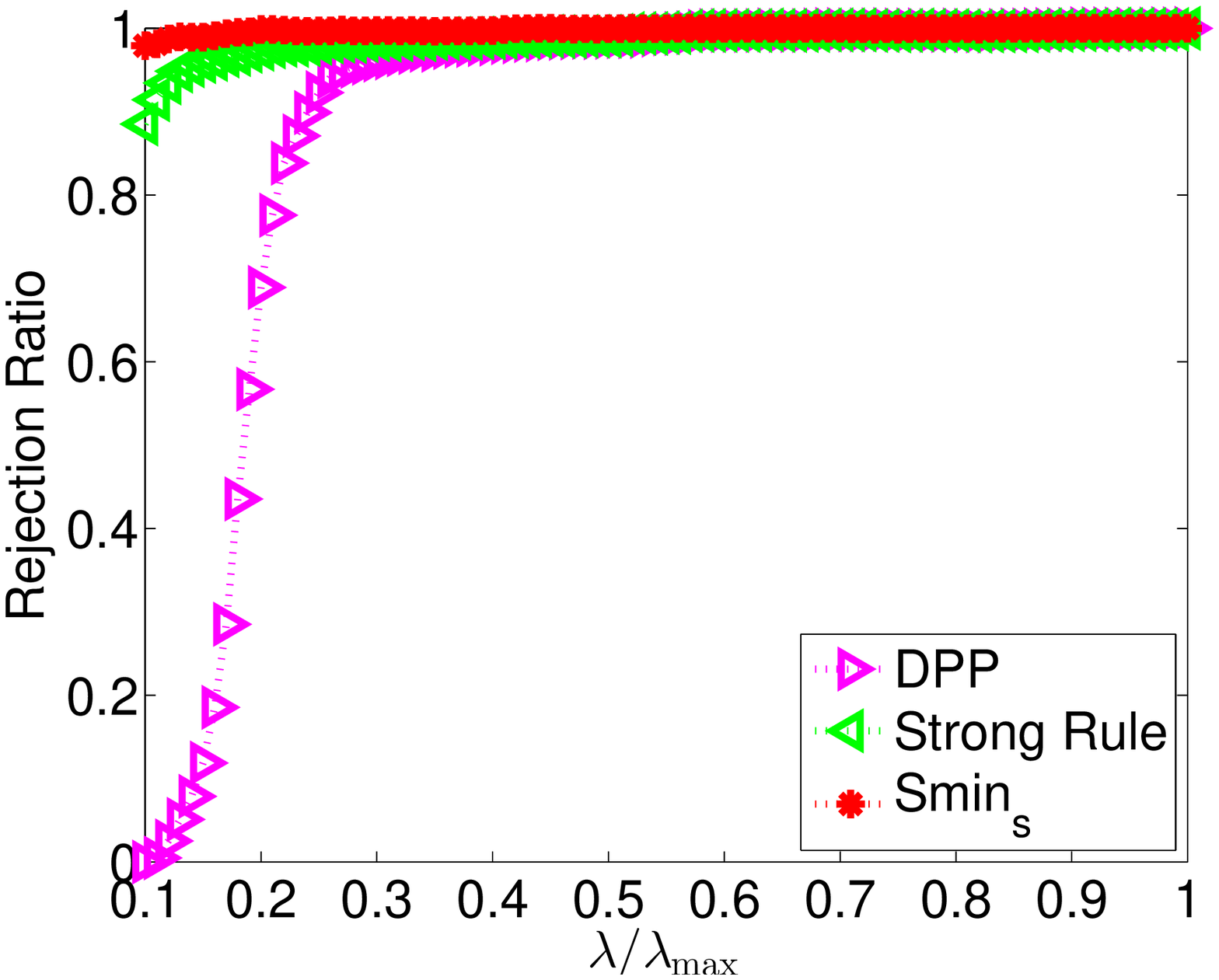}
}
\subfigure[$p=50000$] { \label{fig:q2_5e4_c08_e3}
\includegraphics[width=0.31\columnwidth]{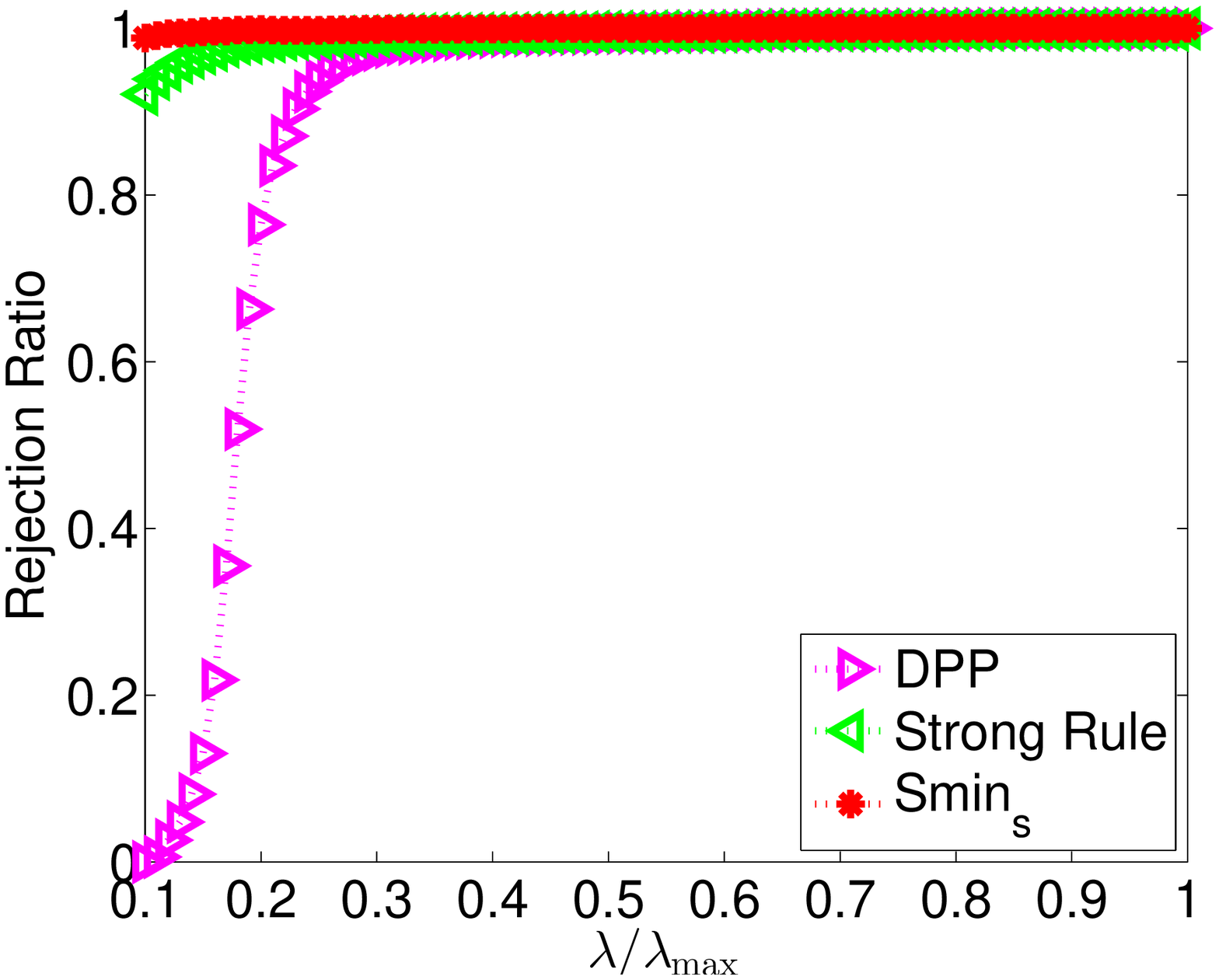}
}
\subfigure[$p=80000$] { \label{fig:q2_8e4_c08_e3}
\includegraphics[width=0.31\columnwidth]{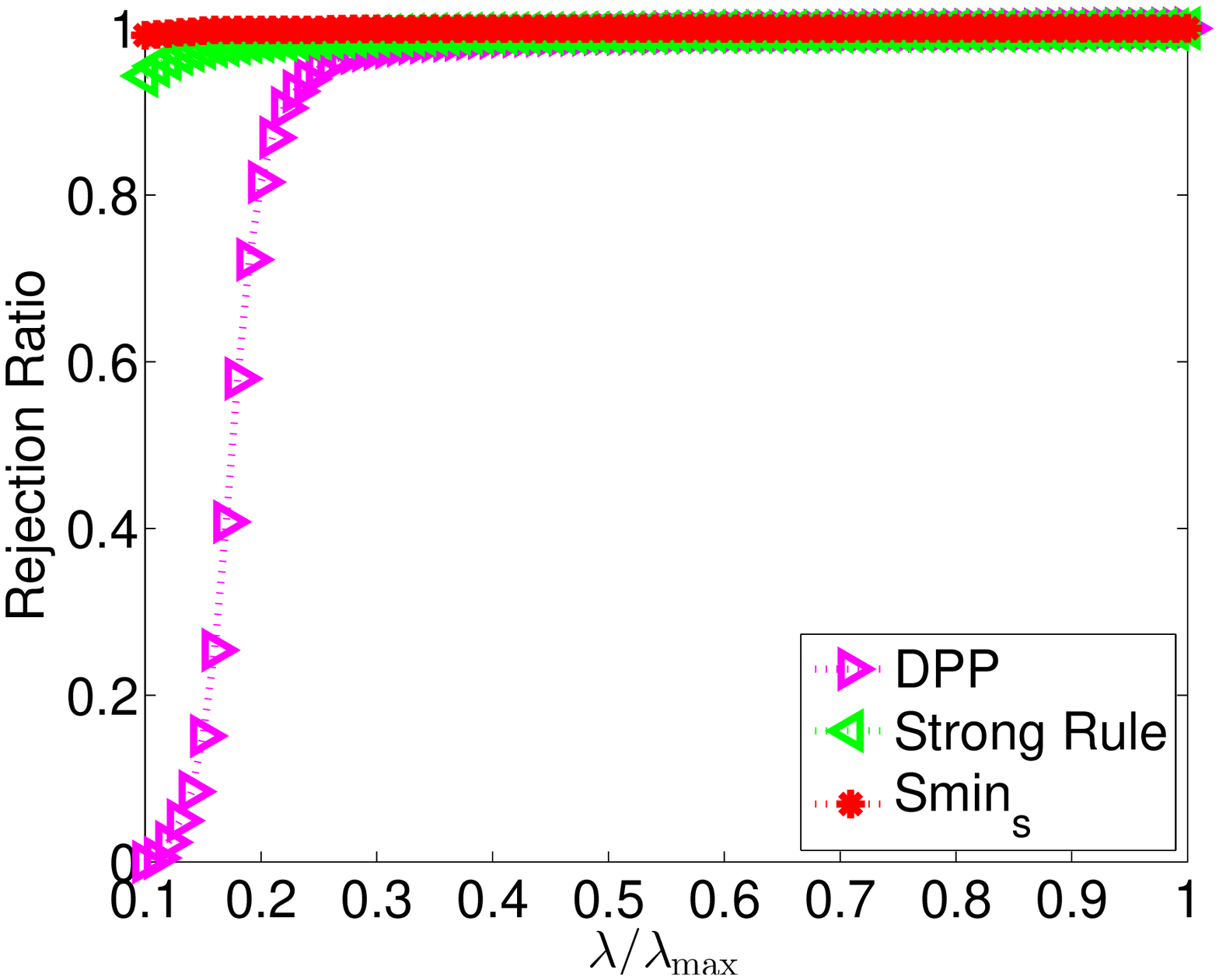}
}\\
\subfigure[$p=100000$] { \label{fig:q2_1e5_c08_e3}
\includegraphics[width=0.31\columnwidth]{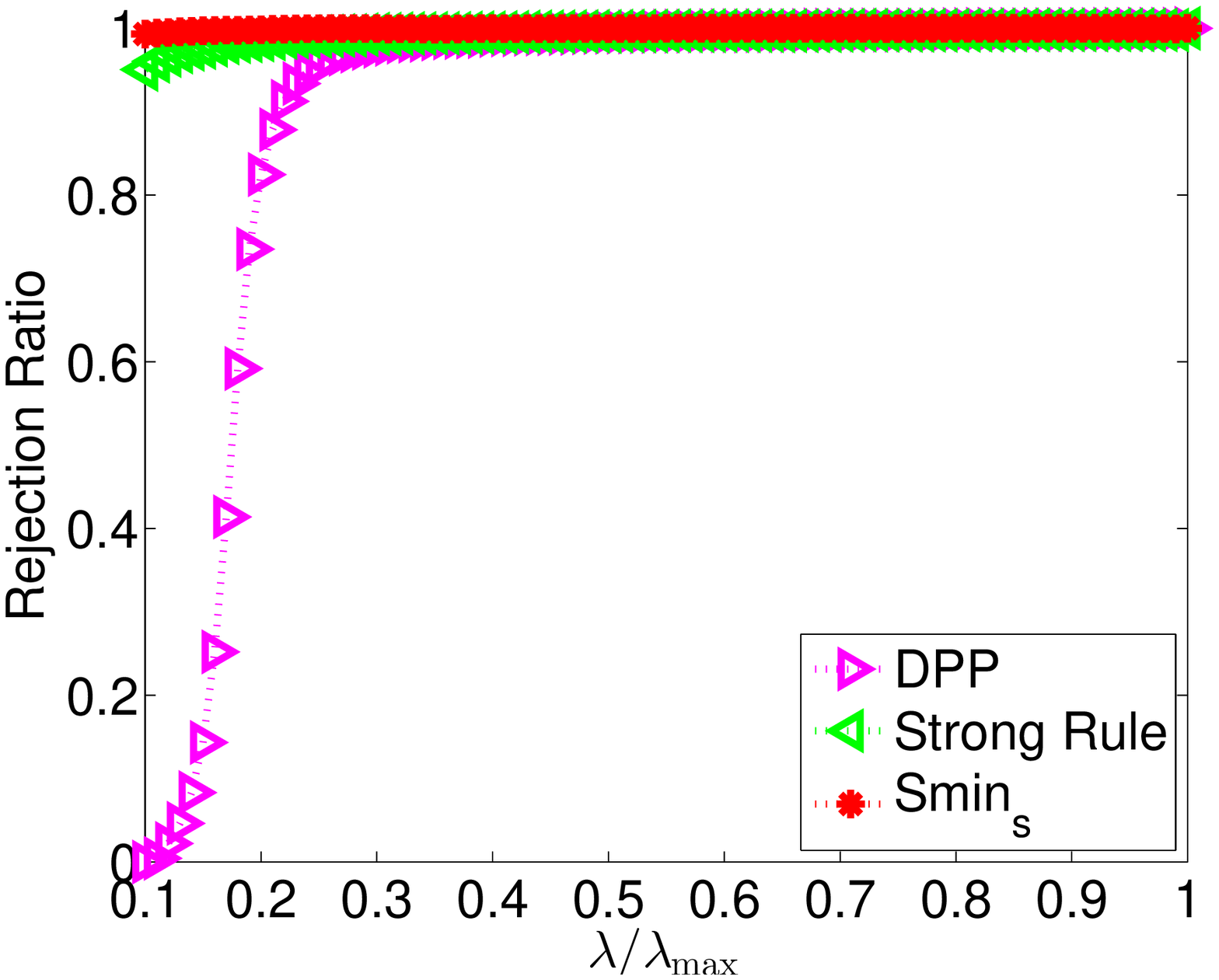}
}
\subfigure[$p=150000$] { \label{fig:q2_15e5_c08_e3}
\includegraphics[width=0.31\columnwidth]{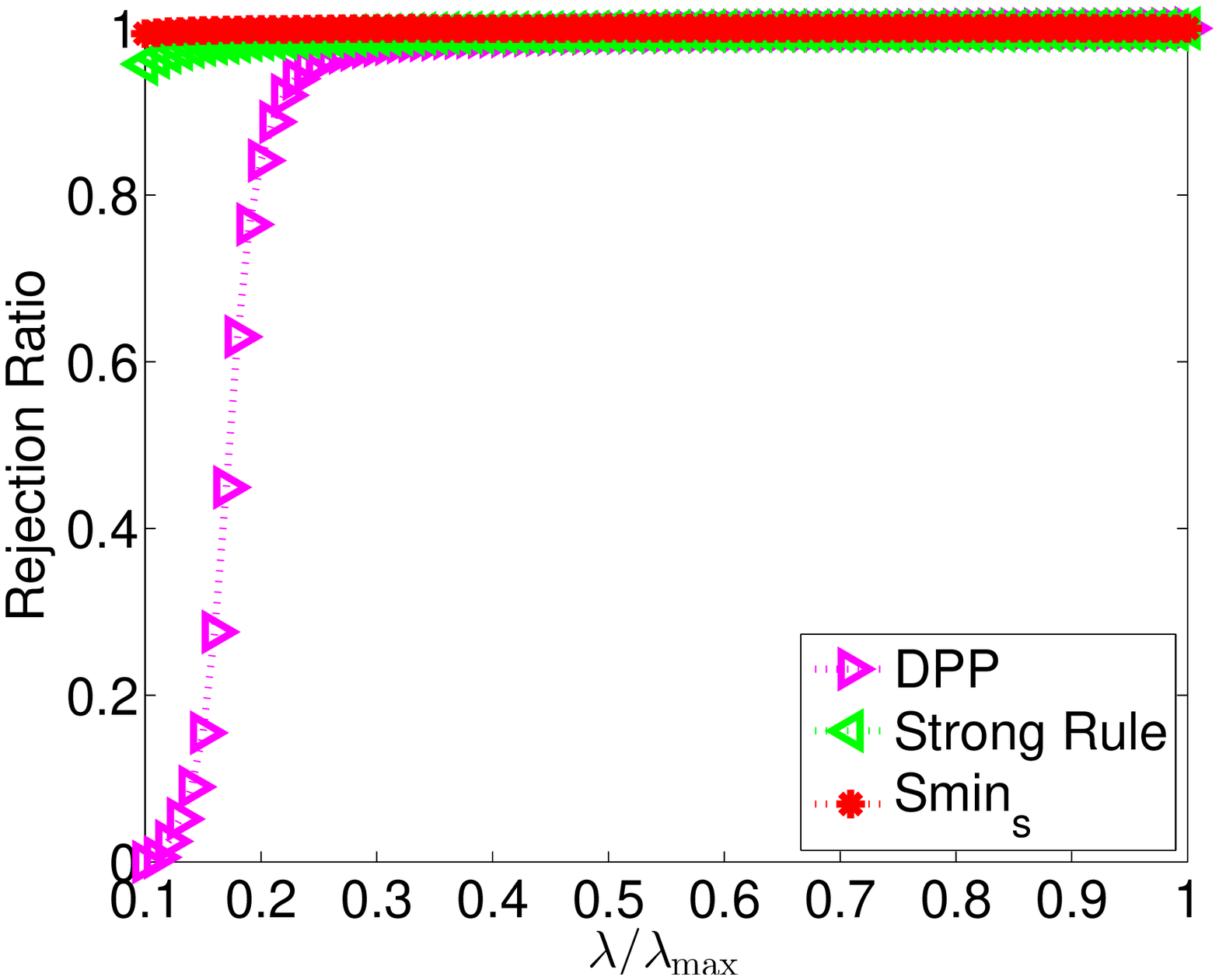}
}
\subfigure[$p=200000$] { \label{fig:q2_2e5_c08_e3}
\includegraphics[width=0.31\columnwidth]{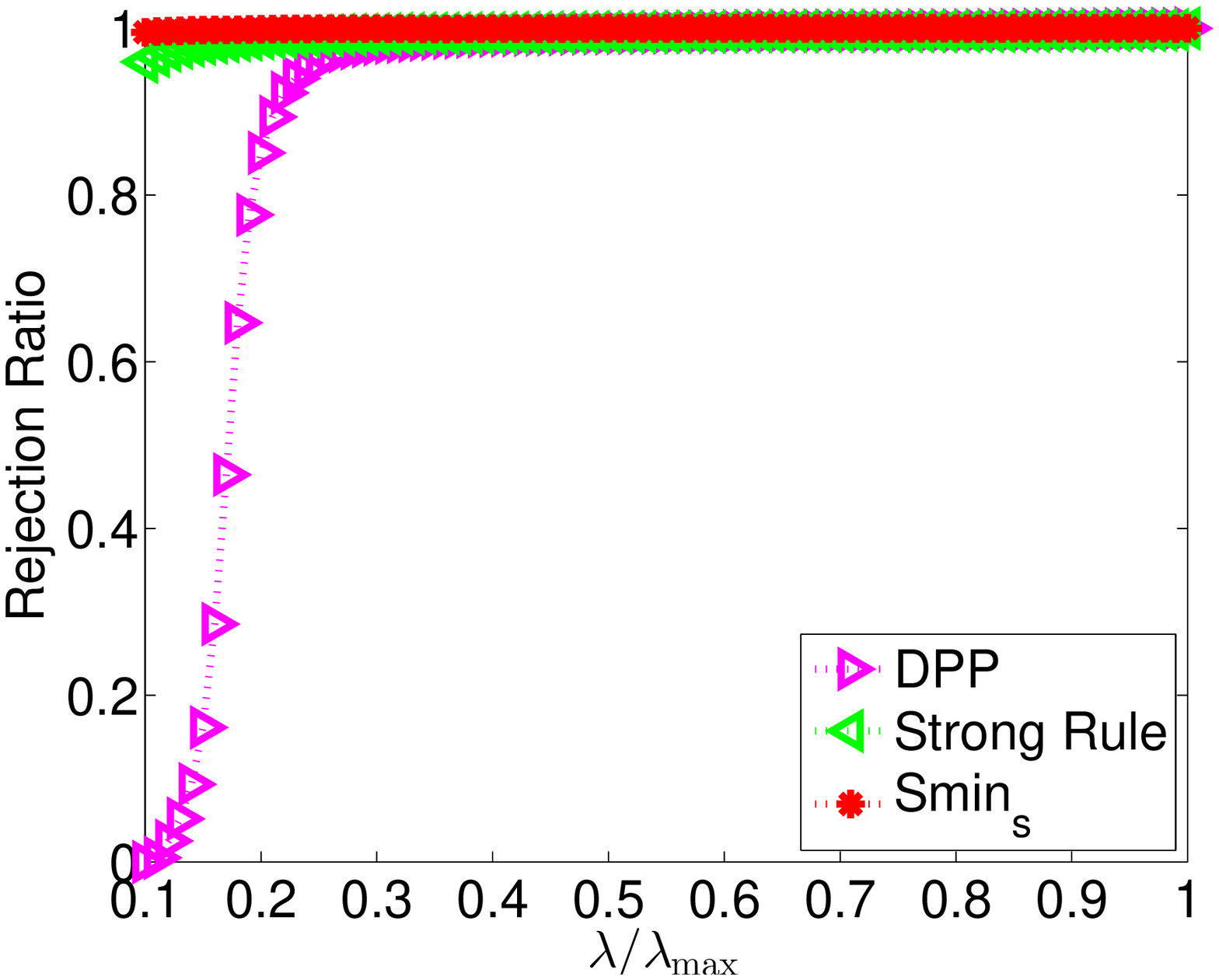}
}
%\subfigure[$q=3$] { \label{fig:q3e1}
%\includegraphics[width=0.31\columnwidth]{image/q3r_e1}
%}
%\subfigure[$q=5$] { \label{fig:q5e1}
%\includegraphics[width=0.31\columnwidth]{image/q5r_e1}
%}
%\subfigure[$q=\infty$] { \label{fig:qinfe1}
%\includegraphics[width=0.31\columnwidth]{image/q175r_e1}
%}
}%\\[-0.4cm]
\caption{Comparison of Smin$_s$, DPP and strong rules with different dimensions. We set $n_g=p/10$ and $q=2$.}
\label{fig:rej_ratio_e3}
\end{figure}

As shown in Figure \ref{fig:rej_ratio_e3}, the performance of all three screening methods, i.e., Smin$_s$, DPP and the strong rule, are robust to different dimensions. We can observe from the figure that Smin$_s$ is again the most effective screening method in discarding inactive groups.

\begin{table}
\begin{center}
\begin{footnotesize}
\begin{tabular}{ |l||c||c|c|c||c|c|c| }
  \hline
  $p$ & GLEP$_{1q}$ & GLEP$_{1q}$+DPP & GLEP$_{1q}$+SR & GLEP$_{1q}$+Smin$_s$ & DPP & Strong Rule & Smin$_s$  \\
  \hline\hline
  $20000$  & 1103.52 & 218.49 & 12.54 & 7.42 & 2.75 & 3.66 & 2.01 \\\hline
  $50000$  & 3440.17 & 508.61 & 23.07 & 11.74 & 7.21 & 9.24 & 4.93 \\\hline
  $80000$  & 6676.28 & 789.35 & 33.26 & 17.17 & 12.31 & 15.06 & 7.97 \\\hline
  $100000$  & 8525.73 & 983.03 & 37.49 & 18.88 & 16.68 & 18.82 & 9.73 \\\hline
  $150000$  & 14752.20 & 1472.79 & 55.28 & 26.88 & 26.16 & 29.00 & 15.28 \\\hline
  $200000$  & 19959.09 & 1970.02 & 69.32 & 33.51 & 38.04 & 38.88 & 21.04 \\\hline
%  $q=1.75$  & 453.60 & 97.00 & 7.65 & 5.41 & 1.89 & 1.97 & 1.06 \\\hline
% $q=2$  & 484.11 & 121.90 & 10.64 & 4.94 & 1.54 & 1.88 & 1.01 \\ \hline
% $q=2.33$  & 949.35 & 321.22 & 137.88 & 84.92 & 2.32 & 2.11 & 1.14\\\hline
% $q=3$  & 808.42 & 203.40 & 93.38 & 79.86 & 1.95 & 2.07 & 1.09 \\\hline
% $q=5$  & 652.93 & 93.88 & 53.53 & 51.53 & 1.74 & 2.04 & 1.11 \\
%  \hline
\end{tabular}
\end{footnotesize}
\end{center}
\vspace{0.1in}
\caption{Running time (in seconds) for solving the mixed-norm regularized problems along a sequence of $91$ tuning parameter values equally spaced on the scale of $r=\frac{\lambda}{\lambda_{max}}$ from $1.0$ to $0.1$ by (a): GLEP$_{1q}$ (reported in the second column) without screening; (b): GLEP$_{1q}$ combined with different screening methods (reported in the third to the fifth columns).
The last three columns report the total running time (in seconds) for the screening methods. The size of the data matrix $B$ is $1000\times p$, $n_g=p/10$ and $q=2$.}
\label{table:time_diff_p}
\end{table}

The computational savings gained by the proposed Smin$_s$ method are more significant than those reported in Sections \ref{sssection:diff_q} and \ref{sssection:diff_ng}. From Table \ref{table:time_diff_p}, when $p=200000$, the efficiency of GLEP$_{1q}$ is improved by more than $6000$ times with Smin$_{1q}$, which demonstrates that Smin$_{1q}$ is a powerful tool to accelerate the optimization of the mixed-norm regularized problems especially for large dimensional data.

\section{Conclusion}\label{s:conclusion}

In this paper, we propose the GLEP$_{1q}$ algorithm and the corresponding screening method, that is, Smin, for solving the
$\ell_1/\ell_q$-norm regularized problem with any $q \geq 1$. The
main technical contributions of this paper include two parts.

First, we develop an efficient algorithm
for the $\ell_1/\ell_q$-norm regularized Euclidean projection
(EP$_{1q}$), which is a key building block of GLEP$_{1q}$.
Specifically, we analyze the key theoretical properties of the
solution of EP$_{1q}$, based on which we develop an efficient
algorithm for EP$_{1q}$ by solving two zero finding problems. Our
analysis also reveals why EP$_{1q}$ for the general $q$ is
significantly more challenging than the special cases such as $q=2$.

Second, we develop a novel screening method (Smin) for large dimensional mixed-norm regularized problems, which is based on an accurate estimation of the possible region of the dual optimal solution. Our method is safe and can be integrated with any existing solvers. Our extensive experiments demonstrate that the proposed Smin is very powerful in discarding inactive groups, resulting in huge computational savings (up to three orders of magnitude) especially for large dimensional problems.

In this paper, we focus on developing efficient algorithms for solving the
$\ell_1/\ell_q$-regularized problem. We plan to study the
effectiveness of the $\ell_1/\ell_q$ regularization under different
values of $q$ for real-world applications in computer vision and
bioinformatics, e.g., imaging genetics \cite{Thompson2013}. We also plan to conduct the
distribution-specific~\cite{Eldar:2010} theoretical studies for
different values of $q$.

{\small
\bibliographystyle{plain}
\bibliography{referenceGallery}

\begin{thebibliography}{10}

\bibitem{Argyriou:2008}
A.~Argyriou, T.~Evgeniou, and Massimiliano Pontil.
\newblock Convex multi-task feature learning.
\newblock {\em Machine Learning}, 73(3):243--272, 2008.

\bibitem{Bach:2008}
F.~Bach.
\newblock Consistency of the group lasso and multiple kernel learning.
\newblock {\em Journal of Machine Learning Research}, 9:1179--1225, 2008.

\bibitem{Beck:2009}
A.~Beck and M.~Teboulle.
\newblock A fast iterative shrinkage-thresholding algorithm for linear inverse
  problems.
\newblock {\em SIAM Journal on Imaging Sciences}, 2(1):183--202, 2009.

\bibitem{BergFriedlander:2008}
E.~Berg, M.~Schmidt, M.~P. Friedlander, and K.~Murphy.
\newblock Group sparsity via linear-time projection.
\newblock Tech. Rep. TR-2008-09, Department of Computer Science, University of
  British Columbia, Vancouver, July 2008.

\bibitem{Bertsekas2003}
D.~Bertsekas.
\newblock {\em Convex Analysis and Optimization}.
\newblock Athena Scientific, 2003.

\bibitem{Boyd:2003}
S.~Boyd, L.~Xiao, and A.~Mutapcic.
\newblock Subgradient methods: Notes for ee392o, 2003.

\bibitem{DUCHI:2009:icml}
J.~Duchi and Y.~Singer.
\newblock Boosting with structural sparsity.
\newblock In {\em International Conference on Machine Learning}, 2009.

\bibitem{Duchi:2009:FOLO}
J.~Duchi and Y.~Singer.
\newblock Online and batch learning using forward backward splitting.
\newblock {\em Journal of Machine Learning Research}, 10:2899--2934, 2009.

\bibitem{Eldar:2010}
Y.~Eldar and H.~Rauhut.
\newblock Average case analysis of multichannel sparse recovery using convex
  relaxation.
\newblock {\em IEEE Transactions on Information Theory}, 56(1):505--519, 2010.

\bibitem{Friedman:2010}
J.~Friedman, T.~Hastie, and R.~Tibshirani.
\newblock A note on the group lasso and a sparse group lasso.
\newblock Technical report, Department of Statistics, Stanford University,
  2010.

\bibitem{FRIEDMAN:2008:ID23}
J.~Friedman, T.~Hastie, and R.~Tibshirani.
\newblock Regularized paths for generalized linear models via coordinate
  descent.
\newblock {\em Journal of Statistical Software}, 33:1--22, 2010.

\bibitem{Ghaoui2012}
L.~El Ghaoui, V.~Viallon, and T.~Rabbani.
\newblock Safe feature elimination in sparse supervised learning.
\newblock {\em Pacific Journal of Optimization}, 8:667--698, 2012.

\bibitem{Guler2010}
O.~G\"{u}ler.
\newblock {\em Foundations of Optimization}.
\newblock Springer, 2010.

\bibitem{Guyon:2004}
I.~Guyon, A.~B. Hur, S.~Gunn, and G.~Dror.
\newblock Result analysis of the nips 2003 feature selection challenge.
\newblock In {\em Neural Information Processing Systems}, pages 545--552, 2004.

\bibitem{HALE:2007:ID23}
E.T. Hale, W.~Yin, and Y.~Zhang.
\newblock Fixed-point continuation for $\ell_1$-minimization: Methodology and
  convergence.
\newblock {\em SIAM Journal on Optimization}, 19(3):1107--1130, 2008.

\bibitem{Hiriart-Urruty:1993}
J.~Hiriart-Urruty and C.~Lemar\'{e}chal.
\newblock {\em Convex Analysis and Minimization Algorithms I \& II}.
\newblock Springer Verlag, Berlin, 1993.

\bibitem{Kowalski:2009}
M.~Kowalski.
\newblock Sparse regression using mixed norms.
\newblock {\em Applied and Computational Harmonic Analysis}, 27(3):303--324,
  2009.

\bibitem{Langford:2009:online}
J.~Langford, L.~Li, and T.~Zhang.
\newblock Sparse online learning via truncated gradient.
\newblock {\em Journal of Machine Learning Research}, 10:777--801, 2009.

\bibitem{Liu:han:2009:blockwise:cd}
H.~Liu, M.~Palatucci, and J.~Zhang.
\newblock Blockwise coordinate descent procedures for the multi-task lasso,
  with applications to neural semantic basis discovery.
\newblock In {\em International Conference on Machine Learning}, 2009.

\bibitem{Liu:han:2009:report}
H.~Liu and J.~Zhang.
\newblock On the estimation and variable selection consistency of the bock
  $q$-norm regression.
\newblock Technical report, Department of Statistics, Carnegie Mellon
  University, 2009.

\bibitem{Liu:2009:uai}
J.~Liu, S.~Ji, and J.~Ye.
\newblock Multi-task feature learning via efficient $\ell_{2,1}$-norm
  minimization.
\newblock In {\em Uncertainty in Artificial Intelligence}, 2009.

\bibitem{Liu:2009:SLEP:manual}
J.~Liu, S.~Ji, and J.~Ye.
\newblock {\em SLEP: Sparse Learning with Efficient Projections}.
\newblock Arizona State University, 2009.

\bibitem{Liu:2009:icml}
J.~Liu and J.~Ye.
\newblock Efficient {E}uclidean projections in linear time.
\newblock In {\em International Conference on Machine Learning}, 2009.

\bibitem{Meier:2008}
L.~Meier, S.~Geer, and P.~B\"uhlmann.
\newblock The group lasso for logistic regression.
\newblock {\em Journal of the Royal Statistical Society: Series B}, 70:53--71,
  2008.

\bibitem{Moreau:1965}
J.-J. Moreau.
\newblock Proximit\'{e} et dualit\'{e} dans un espace hilbertien.
\newblock {\em Bull. Soc. Math. France}, 93:273--299, 1965.

\bibitem{Negahban:2009}
S.~Negahban, P.~Ravikumar, M.~Wainwright, and B.~Yu.
\newblock A unified framework for high-dimensional analysis of $m$-estimators
  with decomposable regularizers.
\newblock In {\em Advances in Neural Information Processing Systems}, pages
  1348--1356. 2009.

\bibitem{Negahban:2008:nips}
S.~Negahban and M.~Wainwright.
\newblock Joint support recovery under high-dimensional scaling: Benefits and
  perils of $\ell_{1,\infty}$-regularization.
\newblock In {\em Advances in Neural Information Processing Systems}, pages
  1161--1168. 2008.

\bibitem{NEMIROVSKI:1994:ID6}
A.~Nemirovski.
\newblock {\em Efficient methods in convex programming}.
\newblock Lecture Notes, 1994.

\bibitem{Nesterov:2004}
Y.~Nesterov.
\newblock {\em Introductory Lectures on Convex Optimization: A Basic Course}.
\newblock Kluwer Academic Publishers, 2004.

\bibitem{Nesterov:2007}
Y.~Nesterov.
\newblock Gradient methods for minimizing composite objective function.
\newblock {\em CORE Discussion Paper}, 2007.

\bibitem{Nesterov:dual:averaging:2009}
Y.~Nesterov.
\newblock Primal-dual subgradient methods for convex problems.
\newblock {\em Mathematical Programming}, 120(1):221--259, 2009.

\bibitem{Obozinski:2007}
G.~Obozinski, B.~Taskar, and M.~I. Jordan.
\newblock Joint covariate selection for grouped classification.
\newblock Technical report, Statistics Department, UC Berkeley, 2007.

\bibitem{Quattoni:2009}
A.~Quattoni, X.~Carreras, M.~Collins, and T.~Darrell.
\newblock An efficient projection for $\ell_{1,\infty}$,infinity
  regularization.
\newblock In {\em International Conference on Machine Learning}, 2009.

\bibitem{Shi:2008}
J.~Shi, W.~Yin, S.~Osher, and P.~Sajda.
\newblock A fast hybrid algorithm for large scale $\ell_1$-regularized logistic
  regression.
\newblock {\em Journal of Machine Learning Research}, 11:713--741, 2010.

\bibitem{Sra2012}
S.~Sra.
\newblock Fast projections onto mixed-norm balls with applications.
\newblock {\em Data Mining and Knowledge Discovery}, 25:358--377, 2012.

\bibitem{Thompson2013}
P.~Thompson, T.~Ge, D.~Glahn, N.~Jahanshad, and T.~Nichols.
\newblock Genetics of the connectome.
\newblock {\em NeuroImage}, To appear.

\bibitem{TIBSHIRANI:1996:ID24}
R.~Tibshirani.
\newblock Regression shrinkage and selection via the lasso.
\newblock {\em Journal of the Royal Statistical Society Series B},
  58(1):267--288, 1996.

\bibitem{tibshirani2012}
R.~Tibshirani, J.~Bien, J.~Friedman, Trevor Hastie, Noah Simon, J.~Taylor, and
  R.~Tibshirani.
\newblock Strong rules for discarding predictors in lasso-type problems.
\newblock {\em Journal of the Royal Statistical Society: Series B},
  74:245--266, 2012.

\bibitem{Tseng:2001}
P.~Tseng.
\newblock Convergence of block coordinate descent method for nondifferentiable
  minimization.
\newblock {\em Journal of Optimization Theory and Applications}, 109:474--494,
  2001.

\bibitem{Tseng:2009:cd}
P.~Tseng and S.~Yun.
\newblock A coordinate gradient descent method for nonsmooth separable
  minimization.
\newblock {\em Mathematical Programming}, 117(1):387--423, 2009.

\bibitem{Vogt2012}
J.~Vogt and V.~Roth.
\newblock A complete analysis of the l1q group-lasso.
\newblock In {\em International Conference on Machine Learning}, 2012.

\bibitem{Wang2012}
J.~Wang, B.~Lin, P.~Gong, P.~Wonka, and J.~Ye.
\newblock Lasso screening rules via dual polytope projection.
\newblock {\em arXiv:1211.3966v1 [cs.LG]}, 2012.

\bibitem{Xiao:average:2009}
L.~Xiao.
\newblock Dual averaging methods for regularized stochastic learning and online
  optimization.
\newblock In {\em Advances in Neural Information Processing Systems}, 2009.

\bibitem{Yosida:1964}
K.~Yosida.
\newblock {\em Functional Analysis}.
\newblock Springer Verlag, Berlin, 1964.

\bibitem{Yuan:2006}
M.~Yuan and Y.~Lin.
\newblock Model selection and estimation in regression with grouped variables.
\newblock {\em Journal Of The Royal Statistical Society Series B},
  68(1):49--67, 2006.

\bibitem{Zhao:2009}
P.~Zhao, G.~Rocha, and B.~Yu.
\newblock The composite absolute penalties family for grouped and hierarchical
  variable selection.
\newblock {\em Annals of Statistics}, 37(6A):3468--3497, 2009.

\bibitem{Zhao:2004}
P.~Zhao and B.~Yu.
\newblock Boosted lasso.
\newblock Technical report, Statistics Department, UC Berkeley, 2004.

\end{thebibliography}
}

\end{document}